\documentclass{article}

\usepackage[pdftex]{graphicx, color}
\usepackage{amsmath,amsthm,amssymb,amsfonts}
\usepackage{algorithm,algorithmic}
\usepackage{bm}
\usepackage[title]{appendix}
\usepackage{authblk}
\usepackage{indentfirst}
\usepackage{geometry}
\usepackage{url}

\newcommand{\argmax}{\mathop{\rm arg~max}\limits}
\newcommand{\argmin}{\mathop{\rm arg~min}\limits}
\theoremstyle{plain}
\newtheorem{theorem}{Theorem}[section]

\newtheorem{remark}{Remark}[section]

\allowdisplaybreaks[2]

\makeatletter
\long\def\@makecaption#1#2{%
  \normalsize
  \vskip\abovecaptionskip
  \sbox\@tempboxa{#1: #2}%
  \ifdim \wd\@tempboxa >\hsize
    #1: #2\par
  \else
    \global \@minipagefalse
    \hb@xt@\hsize{\hfil\box\@tempboxa\hfil}%
  \fi
  \vskip\belowcaptionskip}
\makeatother

\title{Selective Inference for Latent Block Models}
\author[1]{Chihiro Watanabe\thanks{watanabe-chihiro763@g.ecc.u-tokyo.ac.jp}}
\author[1,2]{Taiji Suzuki\thanks{taiji@mist.i.u-tokyo.ac.jp}}
\affil[1]{{\normalsize Graduate School of Information Science Technology, The University of Tokyo, Tokyo, Japan}}
\affil[2]{{\normalsize Center for Advanced Intelligence Project (AIP), RIKEN, Tokyo, Japan}}
\date{}

\begin{document}
\maketitle

\begin{abstract}
Model selection in latent block models has been a challenging but important task in the field of statistics. Specifically, a major challenge is encountered when constructing a test on a block structure obtained by applying a specific clustering algorithm to a finite size matrix. In this case, it becomes crucial to consider the selective bias in the block structure, that is, the block structure is selected from all the possible cluster memberships based on some criterion by the clustering algorithm. To cope with this problem, this study provides a selective inference method for latent block models. Specifically, we construct a statistical test on a set of row and column cluster memberships of a latent block model, which is given by a squared residue minimization algorithm. The proposed test, by its nature, includes and thus can also be used as the test on the set of row and column cluster numbers. We also propose an approximated version of the test based on simulated annealing to avoid combinatorial explosion in searching the optimal block structure. The results show that the proposed exact and approximated tests work effectively, compared to the naive test that did not take the selective bias into account. 

\end{abstract}


\section{Introduction}
\label{sec:introduction}

A latent block model or an LBM \cite{Govaert2003, Hartigan1972} has been widely used as a generative model of a relational data matrix, where the rows and columns represent different objects (e.g., customers and items), and its $(i, j)$th element shows some relationship between objects $i$ and $j$ (e.g., how many times the customer $i$ purchased item $j$). Until now, its effectiveness has been shown in various practical datasets, including customer-product transaction relationships \cite{Shan2008} and gene expression data \cite{Saber2011, Wyse2012}. In LBMs, we assume that there is an underlying block structure (i.e., a set of row and column cluster memberships) behind the observed data matrix and that each element of the matrix is generated independently from an identical distribution, given such a block structure. Particularly, a Gaussian LBM \cite{Lomet2012, Nadif2010} is useful to model a relational data matrix with real elements; this type of LBM is the focus of the current study. In a Gaussian LBM, we assume that each entry follows a Gaussian distribution, whose mean and variance are fixed constants in the same block (a formal description of Gaussian LBMs is given in Section \ref{sec:notation}). 

Besides estimating the block structure from a given observed data matrix based on an LBM, it is also important to \textit{test} the validity of a model (i.e., the number of blocks) or an estimation result. Until now, several tests \cite{Bickel2016, Hu2020, Lei2016, Watanabe2021, Yuan2018} have been proposed for determining the number of blocks in block models, such as a stochastic block model (SBM), which is a model for a square symmetric matrix (e.g., an adjacency matrix of the network structure). Among these studies, only \cite{Watanabe2021}'s test can be applied to the LBM setting; however, its target is different from ours in that it is limited to the number of blocks, not to the cluster memberships. Moreover, it is an asymptotic test, and thus its guarantee cannot be verified with a finite size observed matrix. 

In regard to an SBM, several studies have proposed a statistical test for a given set of community memberships of an observed matrix \cite{Gangrade2019, Hu2020, Karwa2016}. In \cite{Gangrade2019}, based on the numbers of edges within and across the clusters, two tests were proposed for an SBM; one of these tests included a goodness-of-fit test of community memberships. Although this study's objective is similar to ours, its problem setting is quite different from ours in various aspects, such as the setting of the alternative hypothesis and the assumptions in the network structure (e.g., there are two equal-sized communities in a given network and more intra-community edges than inter-community ones). Another study \cite{Hu2020} proposed an asymptotic test on both the number of communities and the community memberships of an SBM, whose validity is guaranteed with the infinite matrix size. This study is different from ours in that our proposed test is validated with a finite size matrix. \cite{Karwa2016} proposed a non-asymptotic test for an SBM setting; they generate finite samples of networks from the distribution of an SBM, conditioned on its sufficient statistics based on Markov chain Monte Carlo (MCMC), and, subsequently, compute the estimator of the $p$-value as the ratio of the test statistics of sampled networks being equal to or larger than that of an observed network. This study is somewhat similar to ours in that it tries to approximate the $p$-value under the condition that some function value of an observed matrix is given; however, it is fully based on a Metropolis-Hastings (MH) algorithm, and thus the resulting $p$-value is \textit{not} exact with finite samples. 

There have been many studies on statistical tests for SBMs, but none of them have enabled us to test the cluster memberships of LBMs. Particularly, in this study, we derive an \textit{exact} $p$-value in the following context, which is a typical case in practice. First, we estimate an underlying block structure or cluster memberships of the rows and columns of an observed data matrix, based on a specific criterion. For instance, as a criterion, we use the \textit{squared residue} or the sample variance within the same block \cite{Cho2004, Hartigan1972}, whose formal definition is given in Section \ref{sec:srm}. Subsequently, we perform a statistical test on the clustering result, \textit{which has been selected as an optimal block structure based on the data matrix}, in terms of the criterion described above. In regard to the construction of a valid statistical test, one concern is that it necessitates taking into account the \textit{selective bias} \cite{Berk2013, Lee2016, Loftus2015}. A test on cluster memberships tends to be inappropriately positive, that is, it tends not to reject the hypothesis that the estimated cluster memberships are correct, when the test fails to consider the fact that the hypothetical set of cluster memberships was selected by using the information of a data matrix. 

In order to perform a valid statistical inference in such a situation, \cite{Berk2013, Lee2016} introduced the methodology of post-selection inference. Particularly, \textit{selective inference} methods facilitate inference of a hypothesis selected based on some criterion, where we use the same data for the hypothesis selection as well as for its inference \cite{Lee2016}. The main idea behind the selective inference is to reveal the probability distribution of a given test statistic under the selection condition. By conditioning on the selection event, we can appropriately construct a test without the selective bias. Such selective inference methods have been developed for various problem settings, including variable selection in linear regression with L1 regularization \cite{Lee2016} and that with marginal screening \cite{Lee2014} and k-means clustering \cite{Inoue2017}. Concerning the problems related to the analysis of relational data matrices, several studies have proposed selective inference methods for biclustering \cite{Henriques2018, Lee2015}. Although they also concern a block (or multiple blocks) in a relational data matrix, their problem settings are different from ours. In our problem setting, a block structure corresponds to a set of cluster memberships of \textit{all} the rows and columns of an observed matrix. In other words, by rearranging the indices of rows and columns, a block structure is represented by a regular lattice on a matrix. However, \cite{Henriques2018, Lee2015} aimed to find a submatrix (or multiple submatrices) of the original data matrix whose mean is significantly larger than zero. 
Figure \ref{fig:As_existing} illustrates the difference between the optimal cluster memberships of the proposed and existing methods \cite{Lee2015}\footnote{To plot Figure \ref{fig:As_existing}, we randomly generated data matrices with the sizes of $(n, p) = (9, 9)$. We set the null and hypothetical sets of cluster numbers at $(2, 2)$; we defined the null cluster memberships as $g^{\mathrm{(N)}, (1)}_i = (i \bmod 2) + 1$, for all $i$, and $g^{\mathrm{(N)}, (2)}_j = (j \bmod 2) + 1$ for all $j$. In regard to the mean vector, we used the following setting: 
$\bm{\mu}_0 = \mathrm{vec} \left( \begin{bmatrix}
0.5 & 0 \\
0 & 0 \\
\end{bmatrix} \right)$. Based on the above settings, we generated a data vector by $\bm{x} \sim N(\bm{\mu}_0, 0.75^2 I)$ and applied the biclustering algorithms of the proposed and existing methods \cite{Lee2015}. The biclustering algorithm of the proposed method outputs a regular-grid bicluster structure based on the squared residue minimization, while that of \cite{Lee2015} outputs an $n_0 \times p_0$ submatrix with the largest sample mean, where we set $n_0 = p_0 = 5$.}. Since they are based on the mutually different assumptions on the latent bicluster structure, their ``optimal'' cluster memberships are not always identical, even with the same observed matrix. 
To the best of our knowledge, no study has proposed a selective inference method for the LBM setting, despite the effectiveness of LBMs in relational data analysis. 

\begin{figure}[t]
  \centering
  \includegraphics[width=0.32\hsize]{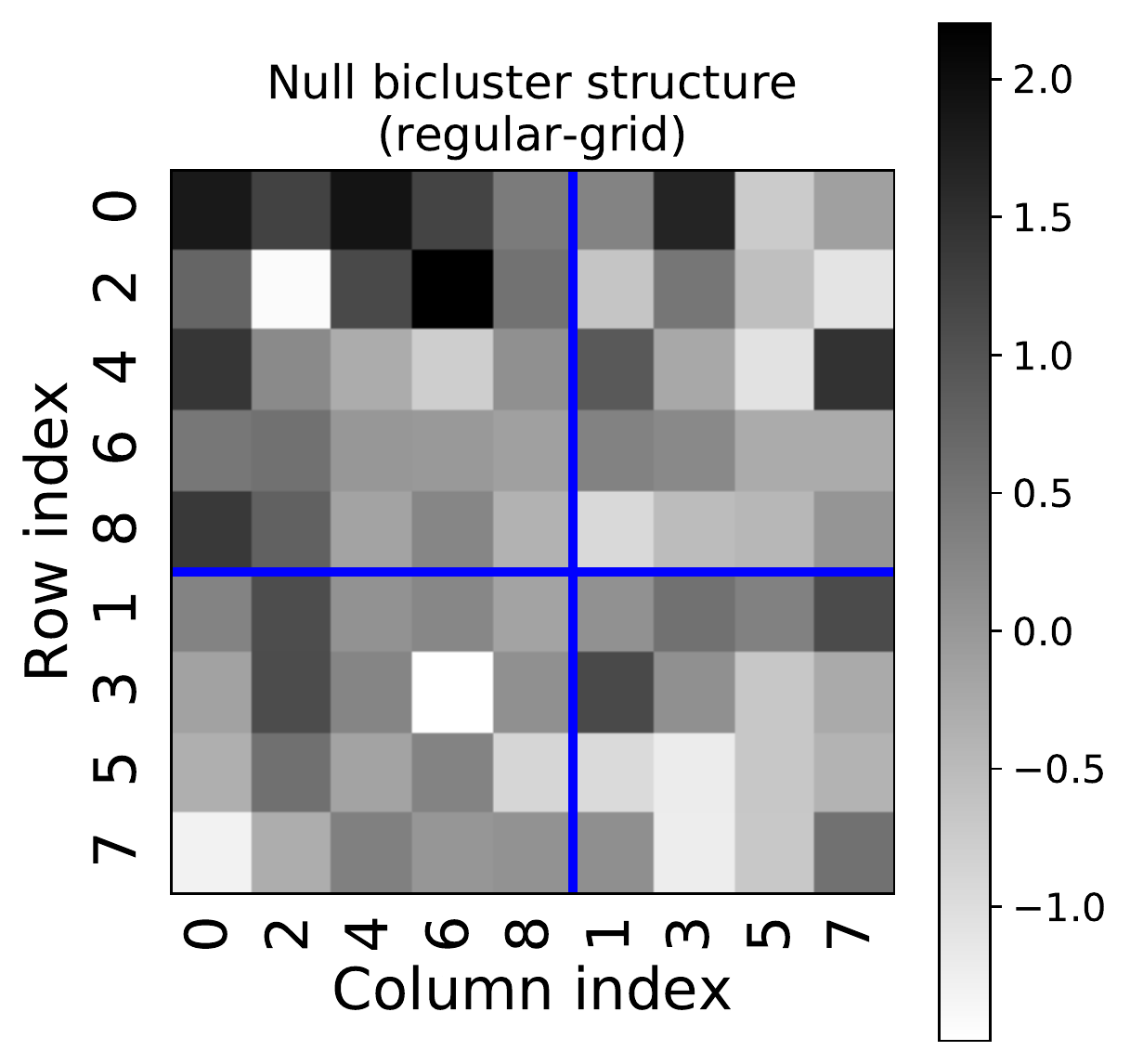}
  \includegraphics[width=0.32\hsize]{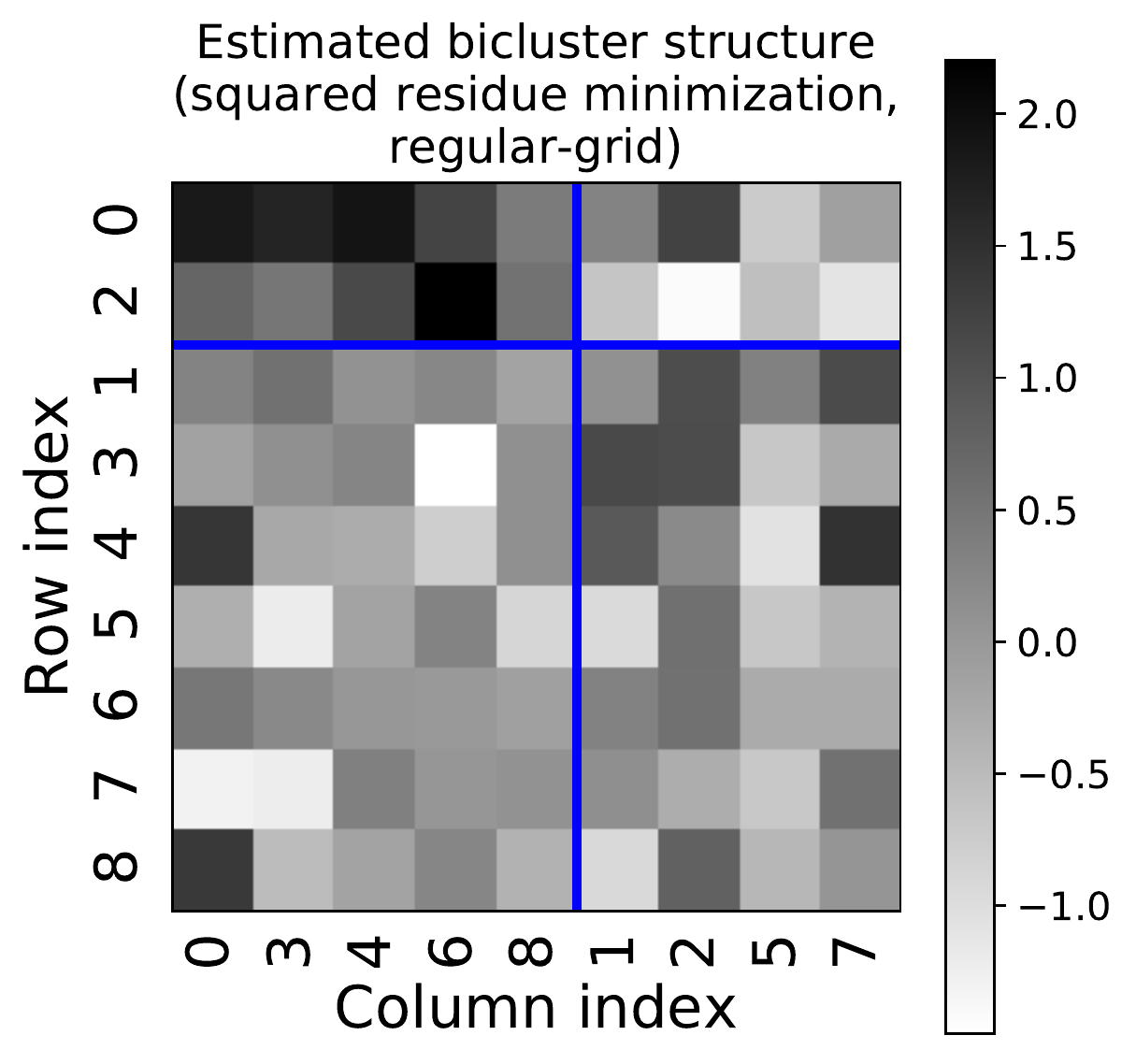}
  \includegraphics[width=0.32\hsize]{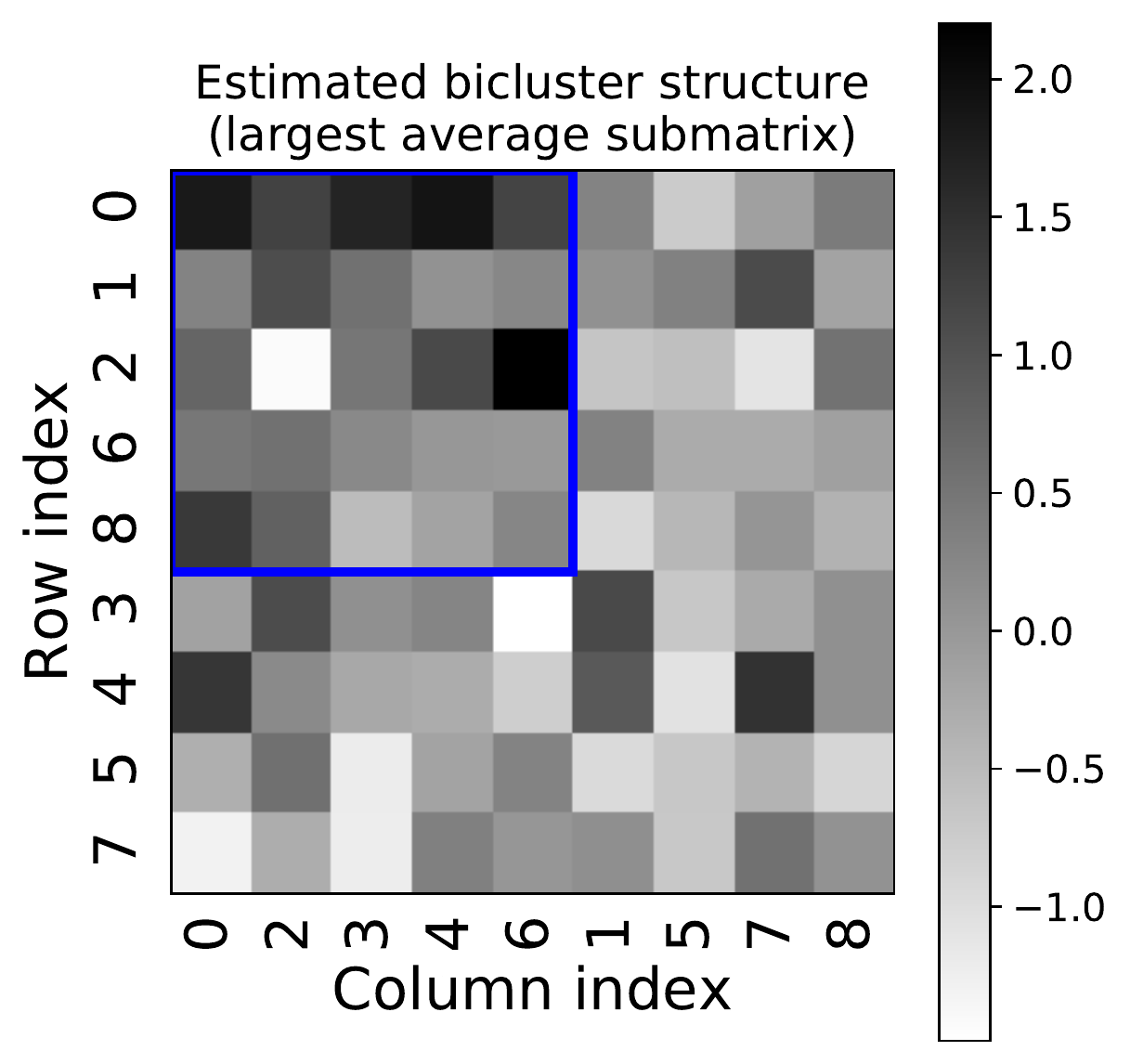}
  \caption{Examples of the null and estimated bicluster structures of the observed data matrix with the size of $(n, p) = (9, 9)$. The rows and columns of the observed matrix were sorted according to their clusters, and the blue lines indicate the cluster memberships. The biclustering algorithms of the proposed and existing methods \cite{Lee2015} do not necessarily yield identical bicluster structures with the same observed matrix.}
  \label{fig:As_existing}
\end{figure}

This study proposes a new selective inference method for LBMs. Unlike our previous study \cite{Watanabe2021}, where the validity of the test is guaranteed only in the asymptotic sense (i.e., with the infinite matrix size), we develop an \textbf{exact} test on a block structure, which is selected based on a given observed matrix with a \textbf{finite} size and the squared residue minimization algorithm. To construct such a statistical test, we considered the fact that the selection event based on the squared residue can be formulated as a set of quadratic inequalities in terms of the data vector, which is the vectorization of the observed data matrix. On this basis, we can show that the test statistic follows a truncated chi distribution, under the selection condition (a formal definition of the test statistic is given in Section \ref{sec:test}). 

Since the exact test requires solving two combinatorial optimization problems---one for selecting the block structure with the minimum squared residue, and the other for determining the truncation interval of the distribution of the test statistic---its computation will be intractable with a large size observed matrix or with a larger hypothetical number of blocks. To cope with such combinatorial explosion, we also develop an approximated version of the test based on simulated annealing (SA). 

The remaining part of this paper is as follows. In Section \ref{sec:problem}, we first define notations and describe assumptions necessary for developing the proposed statistical test. We also define the squared residue, which we use for measuring the quality of a given set of row and column cluster memberships. In Section \ref{sec:test}, we give the formal statement of the null and alternative hypotheses of the proposed test, define the test statistic, and derive its null distribution. Our main contribution lies in Theorem \ref{thm:Tchidist}; it states that, under the null hypothesis, the test statistic follows a truncated chi distribution, whose truncation interval is determined by the selection result. We also give an approximated version of the test. In Section \ref{sec:exp}, we experimentally show the effectiveness of the proposed exact and approximated tests, by checking the behavior of the $p$-values and measuring the true and false positive ratios (TPR and FPR) in both the realizable (i.e., the hypothetical cluster numbers of rows and columns $(K, H)$ are equal to the null ones $(K^{\mathrm{(N)}}, H^{\mathrm{(N)}})$) and unrealizable (i.e., at least one of $K < K^{\mathrm{(N)}}$ and $H < H^{\mathrm{(N)}}$ holds) cases. Finally, we discuss the findings and conclude the paper in Section \ref{sec:discussion} and Section \ref{sec:conclusion}, respectively. 


\section{Problem settings}
\label{sec:problem}

\subsection{Notations and assumptions on data matrix}
\label{sec:notation}

Throughout this study, we use the following definitions and notations. 

\begin{itemize}
\item Let $A = (A_{ij})_{1 \leq i \leq n, 1 \leq j \leq p} \in \mathbb{R}^{n\times p}$ be an observed data matrix with the size of $n \times p$. When constructing a statistical test, it is more convenient to use the vector representation of matrix $A$, instead of $A$ itself: 
\begin{eqnarray}
\bm{x} = \mathrm{vec} (A) \in \mathbb{R}^{np}, \ \ x_{n(j-1)+i} = A_{ij}, \ \ \mathrm{for}\ i = 1, \dots, n,\ \ j = 1, \dots, p. 
\end{eqnarray}
\item Let $g^{(1)}_i$ be the cluster index of the $i$th row, and $g^{(1)} = (g^{(1)}_i)_{1 \leq i \leq n}$. Similarly, let $g^{(2)}_j$ be the cluster index of the $j$th column, and $g^{(2)} = (g^{(2)}_j)_{1 \leq j \leq p}$. We denote a set of row and column clusters as $g = (g^{(1)}, g^{(2)}) \in \mathcal{G}$, where $\mathcal{G} = \{ (g^{(1)}, g^{(2)}) \}$ is a set of all possible cluster memberships. We also define that $\mathcal{G}_{KH}$ is a set of all possible cluster memberships with $K \times H$ or less blocks. 
\item In the null hypothesis of the proposed test, we assume that there exists a set of block memberships $g^{\mathrm{(N)}} = (g^{\mathrm{(N)}, (1)}, g^{\mathrm{(N)}, (2)})$ and that, given $g^{\mathrm{(N)}}$, each $(i, j)$th element $A_{ij}$ of an observed matrix $A$ is generated independently from a Gaussian distribution with a block-wise (\textbf{unknown}) mean $P_{ij} \equiv B_{g^{\mathrm{(N)}, (1)}_i g^{\mathrm{(N)}, (2)}_j}$ and (\textbf{known}) variance $\sigma_0^2$\footnote{We also derive the null distribution of a test statistic in case that variance $\sigma_0^2$ is unknown in Appendix \ref{sec:unknown_sigma0}.}, where $B_{kh}$ is the mean of the $(k, h)$th block: 
\begin{eqnarray}
A_{ij} \sim N(P_{ij}, \sigma_0^2), \ \ \ \mathrm{for\ all}\ i = 1, \dots, n, \ \ \ j = 1, \dots, p. 
\end{eqnarray}
In vector representation, this assumption is given by 
\begin{eqnarray}
\label{eq:x_gauss}
\bm{x} \sim N(\bm{\mu}_0, \sigma_0^2 I), 
\end{eqnarray}
where $\bm{\mu}_0$ is the \textbf{unknown} block-wise mean vector. 
\item Let $(K^{\mathrm{(N)}}, H^{\mathrm{(N)}})$ be the minimum set of row and column cluster numbers required to represent the above null set of block memberships $g^{\mathrm{(N)}}$. In the proposed test, we fix a hypothetical set of cluster numbers $(K, H)$, estimate the block structure of an observed matrix with $K \times H$ blocks, and perform a test on the estimated block memberships, which, by its nature, includes a test on cluster numbers (i.e., $(K^{\mathrm{(N)}}, H^{\mathrm{(N)}}) = (K, H)$ or at least one of $K < K^{\mathrm{(N)}}$ and $H < H^{\mathrm{(N)}}$ holds)\footnote{It must be noted, however, that the proposed test cannot be applied directly for sequential testing on cluster numbers, where the hypothetical numbers of clusters are tested in ascending order (i.e., $(K, H) = (1, 1), (1, 2), (2, 1), \dots$) until the null hypothesis is accepted. This is because the proposed test cannot distinguish the following two alternative cases: (1) $(K^{\mathrm{(N)}}, H^{\mathrm{(N)}}) = (K, H)$ holds, however, the estimated cluster memberships are incorrect, and (2) $K < K^{\mathrm{(N)}}$ or $H < H^{\mathrm{(N)}}$ holds.}.
\item We denote the set of rows in the $k$th cluster as $I_k = \{ i: g^{(1)}_i = k \}$. Similarly, we denote the set of columns in the $h$th cluster as $J_h = \{ j: g^{(2)}_j = h \}$. 
\item We denote the cluster membership vector of rows as follows: 
\begin{eqnarray}
\overline{\bm{e}}^{(k)} = (\overline{e}^{(k)}_i)_{1 \leq i \leq n} \in \mathbb{R}^n,\ \overline{e}^{(k)}_i = \begin{cases}
\frac{1}{\sqrt{|I_k|}} & \mathrm{if}\ g^{(1)}_i = k, \\
0 & \mathrm{otherwise}.
\end{cases}
\end{eqnarray}
Similarly, we denote the cluster membership vector of columns as follows: 
\begin{eqnarray}
\underline{\bm{e}}^{(h)} = (\underline{e}^{(h)}_j)_{1 \leq j \leq p} \in \mathbb{R}^p,\ \underline{e}^{(h)}_j = \begin{cases}
\frac{1}{\sqrt{|J_h|}} & \mathrm{if}\ g^{(2)}_j = h, \\
0 & \mathrm{otherwise}.
\end{cases}
\end{eqnarray}
Based on these vectors $\overline{\bm{e}}^{(k)}$ and $\underline{\bm{e}}^{(h)}$, we define a vector $\bm{e}^{(k, h)} \equiv \underline{\bm{e}}^{(h)} \otimes \overline{\bm{e}}^{(k)} \in \mathbb{R}^{np}$ and matrix $E^{(g)} \equiv I - \sum_k \sum_h \bm{e}^{(k, h)} (\bm{e}^{(k, h)})^{\top}$. It must be noted that $E^{(g)}$ is a projection matrix, that is, $(E^{(g)})^{\top} = E^{(g)}$ and $(E^{(g)})^2 = E^{(g)}$ hold. 
\end{itemize}


\subsection{Clustering algorithm based on squared residue minimization}
\label{sec:srm}

To estimate the block structure of a given observed matrix $A$, we use a clustering algorithm $\mathcal{A}: \bm{x} \mapsto \hat{\mathcal{M}} \in \mathcal{G}_{KH}$ that outputs a block structure minimizing the \textit{squared residue}, that is, the sample variance $\sigma^2$ within the same block. A squared residue has been proposed for measuring the quality of a biclustering result \cite{Cho2004, Hartigan1972}, and its definition is given by
\begin{align}
\label{eq:sr_def}
\sigma^2 &= \frac{1}{np} \sum_k \sum_h \sum_{i \in I_k} \sum_{j \in J_h} \left( A_{ij} - \frac{1}{|I_k||J_h|} \sum_{i' \in I_k} \sum_{j' \in J_h} A_{i'j'} \right)^2 \nonumber \\
&= \frac{1}{np} \left[ \sum_{i, j} A_{ij}^2 - \sum_k \sum_h \frac{1}{|I_k||J_h|} \left( \sum_{i \in I_k} \sum_{j \in J_h} A_{ij} \right)^2 \right] \nonumber \\
&= \frac{1}{np} \left[ \sum_{i, j} A_{ij}^2 - \sum_k \sum_h \left( \frac{1}{\sqrt{|I_k||J_h|}} \sum_{i \in I_k} \sum_{j \in J_h} A_{ij} \right)^2 \right] \nonumber \\
&= \frac{1}{np} \left\{ \bm{x}^{\top} \bm{x} - \sum_k \sum_h \left[ (\underline{\bm{e}}^{(h)} \otimes \overline{\bm{e}}^{(k)})^{\top} \bm{x} \right]^2 \right\} \nonumber \\
&= \frac{1}{np} \left\{ \bm{x}^{\top} \bm{x} - \sum_k \sum_h \left[ (\bm{e}^{(k, h)})^{\top} \bm{x} \right]^2 \right\} \nonumber \\
&= \frac{1}{np} \bm{x}^{\top} \left[ I - \sum_k \sum_h \bm{e}^{(k, h)} (\bm{e}^{(k, h)})^{\top} \right] \bm{x} 
= \frac{1}{np} \bm{x}^{\top} E^{(g)} \bm{x}. 
\end{align}

Therefore, the squared residue minimization clustering algorithm $\mathcal{A}$ outputs the set of cluster memberships $\hat{g} = (\hat{g}^{(1)}, \hat{g}^{(2)})$, which satisfies
\begin{eqnarray}
\label{eq:g_hat}
\hat{g} \in \hat{\mathcal{M}} (\bm{x}) = \argmin_{g \in \mathcal{G}_{KH}} \sigma^2 = \argmin_{g \in \mathcal{G}_{KH}} \bm{x}^{\top} E^{(g)} \bm{x}. 
\end{eqnarray}
It must be noted that the above solution $\hat{g}$ is the maximum likelihood estimator of the cluster memberships with a mean estimator $\hat{B} (g)$ and a known standard deviation $\sigma_0$. The log likelihood of a set of cluster memberships $g = (g^{(1)}, g^{(2)})$ and the mean parameter $B$ is given by
\begin{align}
\label{eq:log_lh}
\mathcal{L} (g, B; \bm{x}) &= -np \log \left( \sqrt{2 \pi \sigma_0^2} \right) - \frac{1}{2 \sigma_0^2} \sum_{i=1}^n \sum_{j=1}^p \left( x_{n (j - 1) + i} - B_{g^{(1)}_i g^{(2)}_j} \right)^2 \nonumber \\
&= -np \log \left( \sqrt{2 \pi \sigma_0^2} \right) - \frac{1}{2 \sigma_0^2} \sum_{k=1}^K \sum_{h=1}^H \sum_{i \in I_k} \sum_{j \in J_h} \left( x_{n (j - 1) + i} - B_{kh} \right)^2. 
\end{align}
Let $\hat{B} (g) = (\hat{B}_{kh} (g))_{1 \leq k \leq K, 1 \leq h \leq H}$ be the maximum likelihood estimator of mean $B$ for a given fixed cluster memberships $g$. From (\ref{eq:log_lh}), we can easily derive that $\hat{B}_{kh} (g) = (1/|I_k||J_h|) \sum_{i \in I_k} \sum_{j \in J_h} x_{n (j - 1) + i}$. By combining this fact with (\ref{eq:sr_def}) and (\ref{eq:log_lh}), we see that the squared residue minimization is equivalent to the likelihood maximization with a mean estimator $\hat{B} (g)$.

Equation (\ref{eq:g_hat}) is equivalent to a set of quadratic inequalities
\begin{eqnarray}
\label{eq:condtion_ghat}
\bm{x}^{\top} E^{(\hat{g})} \bm{x} \leq \bm{x}^{\top} E^{(g)} \bm{x} 
\iff \bm{x}^{\top} \left( E^{(g)} - E^{(\hat{g})} \right) \bm{x} \geq 0, 
\end{eqnarray}
for all $g \in \mathcal{G}_{KH}$. In other words, the selection rule can be represented as a set of quadratic inequalities in terms of the data vector $\bm{x}$. 
It must be noted that, under the null hypothesis, the solution $\hat{g}$ of (\ref{eq:g_hat}) is unique almost surely. 
To prove this fact, we first define a quadratic function $F^{(g, g')}: \mathbb{R}^{np} \mapsto \mathbb{R}$ for a fixed $(g, g')$ as $F^{(g, g')} (\bm{x}) \equiv \bm{x}^{\top} \left( E^{(g)} - E^{(g')} \right) \bm{x}$. We also define that $g = g'$, if the sets of cluster memberships $g$ and $g'$ are equivalent up to the permutation of cluster indices, and that $g \neq g'$ otherwise. If $g \neq g'$, $E^{(g)} - E^{(g')}$ is not a zero matrix (the proof of this is in Appendix \ref{sec:ap_nonzero_E}), and thus the Lebesgue measure of a set of points $\bm{x}$ that satisfy $F^{(g, g')} (\bm{x}) = 0$ is zero. By combining this fact and the assumption (\ref{eq:x_gauss}) of the null hypothesis, $F^{(g, g')} (\bm{x}) \neq 0$ holds for a fixed combination of $(g, g')$ almost surely. Since $\{(g, g'): g, g' \in \mathcal{G}_{KH}, g \neq g' \}$ is a finite set, we finally have 
\begin{align}
&\mathrm{Pr} \left( \exists g, g' \in \mathcal{G}_{KH},\ \mathrm{s.t.}\ g \neq g',\ g, g' \in \argmin_{g \in \mathcal{G}_{KH}} \sigma^2 \right) \nonumber \\
&= \mathrm{Pr} \left( \exists g, g' \in \mathcal{G}_{KH},\ \mathrm{s.t.}\ g \neq g',\ F^{(g, g')} (\bm{x}) = 0,\ g, g' \in \argmin_{g \in \mathcal{G}_{KH}} \sigma^2 \right) \nonumber \\
&\leq \mathrm{Pr} \left( \exists g, g' \in \mathcal{G}_{KH},\ \mathrm{s.t.}\ g \neq g',\ F^{(g, g')} (\bm{x}) = 0 \right) \nonumber \\
&\leq \sum_{g, g' \in \mathcal{G}_{KH}, g \neq g'} \mathrm{Pr} \left( F^{(g, g')} (\bm{x}) = 0 \right) = 0. 
\end{align}
In case of a tie (i.e., multiple solutions of $\hat{g}$ exist that satisfy (\ref{eq:g_hat})) that occurs with probability zero, we can choose any one of them as $\hat{g}$ independently with $\bm{x}$. 


\section{Statistical test on the solution of squared residue minimization}
\label{sec:test}

\subsection{Null distribution of test statistic $T$}

As described in Section \ref{sec:problem}, in the null hypothesis of the proposed test, we assume that there exists a set of block memberships $g^{\mathrm{(N)}}$ and that given $g^{\mathrm{(N)}}$, each element of an observed data vector $\bm{x}$ is generated independently from a Gaussian distribution, whose mean is constant within the same block. Our main purpose is to test whether an estimated block structure $\hat{g} \equiv (\hat{g}^{(1)}, \hat{g}^{(2)})$, which is selected based on the squared residue criterion in Section \ref{sec:srm}, is equal to the null one $g^{\mathrm{(N)}}$. 
Formally, the null and alternative hypotheses of the proposed test are given by 
\begin{eqnarray}
\label{eq:test}
\mathrm{(N):}\ E^{(\hat{g})} \bm{\mu}_0 = \bm{0}, \ 
\mathrm{(A):}\ E^{(\hat{g})} \bm{\mu}_0 \neq \bm{0}. 
\end{eqnarray}
It must be noted that the equation $E^{(\hat{g})} \bm{\mu}_0 = \bm{0}$ is equivalent to the statement that the elements of the vector $\bm{\mu}_0$ are constant in the same block in the set of cluster memberships $\hat{g}$. In other words, the above statement of the null hypothesis is that a given observed matrix is generated based on the latent block structure $\hat{g}$, which is selected as a solution that minimizes the squared residue. 

To perform the test of (\ref{eq:test}), we check the squared residue $\sigma^2$ of the given observed matrix $A$ under the condition that the estimated block structure $\hat{g}$ is selected. 
Under the null hypothesis, we have $E^{(\hat{g})} \bm{\mu}_0 = 0$. Here, matrix $E^{(\hat{g})}$ solely depends on the estimated set of cluster memberships $\hat{g}$. In other words, under the condition that $\hat{\mathcal{M}} (\bm{x}) = \hat{g}$ holds, matrix $E^{(\hat{g})}$ is fixed. 
Therefore, based on the result in \cite{Loftus2015}, the following theorem holds: 
\begin{theorem}
\label{thm:Tchidist}
Under the null hypothesis, we have
\begin{eqnarray}
\label{eq:defT}
T \equiv \frac{\| \bm{r} \|_2}{\sigma_0}, \ \ \ 
T | \{ \hat{g}, \bm{z}, \bm{u} \} \sim \chi_{(np - KH) | \hat{M}^{(\hat{g})}}, 
\end{eqnarray}
where $\| \cdot \|_2$ and $\chi_{c | M}$, respectively, denote the Euclid norm and the truncated chi distribution with $c$ degrees of freedom and with truncation interval of $M$ and
\begin{eqnarray}
\label{eq:defruz}
&&\bm{r} \equiv E^{(\hat{g})} \bm{x}, \ \ \ 
\bm{u} \equiv \frac{1}{\| \bm{r} \|_2} \bm{r}, \ \ \ 
\bm{z} \equiv \bm{x} - \bm{r}, \nonumber \\
&&\hat{M}^{(\hat{g})} \equiv \{ t \geq 0: \hat{g} \in \hat{\mathcal{M}} (t \sigma_0 \bm{u} + \bm{z}) \}. 
\end{eqnarray}
\end{theorem}
\begin{proof}
Let $E$ be a fixed $np \times np$ projection matrix satisfying the following conditions: 
\begin{itemize}
\item $\mathrm{rank} (E) = np - KH$. 
\item $E \bm{\mu}_0 = \bm{0}$. 
\end{itemize}

A singular value decomposition of a matrix $E$ satisfying the above two conditions is given by
\begin{eqnarray}
\label{eq:EandD}
E = V^{\top} D V, \ \ \ \ \ 
D \equiv \begin{bmatrix}
I_{(np-KH)} & O_{(np-KH), KH} \\
O_{KH, (np-KH)} & O_{KH, KH} \\
\end{bmatrix}, 
\end{eqnarray}
where we denote the $a \times a$ identity matrix and $a \times b$ zero matrix, respectively, as $I_a$ and $O_{a, b}$. 

Based on such a matrix $E$, we use the following notations: 
\begin{eqnarray}
&&\bm{r}_E \equiv E \bm{x}, \ \ \ 
T_E = \frac{\| \bm{r}_E \|_2}{\sigma_0}, \ \ \ 
\bm{u}_E \equiv \frac{1}{\| \bm{r}_E \|_2} \bm{r}_E, \ \ \ 
\bm{z}_E \equiv \bm{x} - \bm{r}_E. 
\end{eqnarray}

In the above definitions, we can transform $T_E$ by the following equations: 
\begin{align}
\label{eq:TE}
T_{E} &= \frac{\sqrt{\bm{x}^{\top} E \bm{x}}}{\sigma_0} 
= \frac{\sqrt{(\bm{x} - \bm{\mu}_0)^{\top} E (\bm{x} - \bm{\mu}_0)}}{\sigma_0}\ \ \ (\because E \bm{\mu}_0 = \bm{0}) \nonumber \\
&= \frac{\sqrt{(\bm{x} - \bm{\mu}_0)^{\top} V^{\top} D V (\bm{x} - \bm{\mu}_0)}}{\sigma_0} 
= \frac{\sqrt{(\bm{x} - \bm{\mu}_0)^{\top} V^{\top} \tilde{D}^{\top} \tilde{D} V (\bm{x} - \bm{\mu}_0)}}{\sigma_0}, \nonumber \\
\tilde{D} &\equiv \begin{bmatrix}
I_{(np-KH)} & O_{(np-KH), KH} \\
\end{bmatrix} \in \mathbb{R}^{(np-KH) \times np}. 
\end{align}
Here, we used the fact that $\tilde{D}^{\top} \tilde{D} = D$. 

By using the assumption that $\bm{x} \sim N(\bm{\mu}_0, \sigma_0^2 I)$ holds and the independence of matrix $E$ of $\bm{x}$, we have 
\begin{align}
\label{eq:DDVnorm}
&\frac{1}{\sigma_0} \tilde{D} V (\bm{x} - \bm{\mu}_0) \sim N(\bm{0}, \tilde{D} V (\tilde{D} V)^{\top}). \nonumber \\
&\iff \frac{1}{\sigma_0} \tilde{D} V (\bm{x} - \bm{\mu}_0) \sim N(\bm{0}, I_{np-KH}).
\end{align}
Here, we considered the fact that $\tilde{D} \tilde{D}^{\top} = I$. Therefore, by combining (\ref{eq:TE}) and (\ref{eq:DDVnorm}), we have 
\begin{eqnarray}
\label{eq:T_E_chi}
T_{E} \sim \chi_{(np-KH)}, 
\end{eqnarray}
where $\chi_c$ denotes the chi distribution with $c$ degrees of freedom. 

In regard to $\bm{u}_E$ and $\bm{z}_E$, we have
\begin{eqnarray}
\label{eq:uz_zero}
\bm{u}_E \cdot \bm{z}_E = \frac{1}{\| \bm{r}_E \|_2} \bm{r}_E^{\top} (\bm{x} - \bm{r}_E) 
= \frac{1}{\| \bm{r}_E \|_2} (\bm{x}^{\top} E^{\top} \bm{x} - \bm{x}^{\top} E^{\top} E \bm{x}) = 0. 
\end{eqnarray}
In the last equation, we considered the fact that $E^{\top} E = E$. 

Here, since $T_E$ and $(\bm{u}_E, \bm{z}_E)$ are mutually independent (the proof of this is in Appendix \ref{sec:ap_indTuz}), we have
\begin{eqnarray}
\label{eq:tildeT_uz}
T_{E} | \bm{u}_{E}, \bm{z}_{E} \sim \chi_{(np-KH)}. 
\end{eqnarray}

Next, we consider adding a condition of selection event of $\hat{g}$ to the distribution of $T_{E} | \bm{u}_{E}, \bm{z}_{E}$ in (\ref{eq:tildeT_uz}). Given $\bm{u}_{E}$ and $\bm{z}_{E}$, the result of selection depends solely on the value of $T_{E}$. Therefore, adding the selection condition $\hat{\mathcal{M}} (\bm{u}_{E} T_{E} \sigma_0 + \bm{z}_{E}) = \hat{g}$ to (\ref{eq:tildeT_uz}) corresponds to truncation of $T_{E}$ to the region where $\hat{\mathcal{M}} (\bm{u}_{E} T_{E} \sigma_0 + \bm{z}_{E}) = \hat{g}$ holds: 
\begin{eqnarray}
\label{eq:tildeT_uzg}
T_{E} | \bm{u}_{E}, \bm{z}_{E}, \hat{g} \sim \chi_{(np-KH) | \hat{M}^{(\hat{g})} (E)}. 
\end{eqnarray}

Third, we consider replacing $E$ in (\ref{eq:tildeT_uzg}) with $E^{(\hat{g})}$, which is the output by clustering algorithm $\mathcal{A}$ based on the data vector $\bm{x}$. It must be noted that the matrix $E^{(\hat{g})}$ is also a projection matrix with the rank of $(np - KH)$ (the proof of this is in Appendix \ref{sec:ap_rankE}), from its definition, and $E^{(\hat{g})} \bm{\mu}_0 = \bm{0}$ holds. 

Since matrix $E^{(\hat{g})}$ depends on the data vector $\bm{x}$ only through the choice of $\hat{g}$ (i.e., $E^{(\hat{g})}$ is fixed, given $\hat{g}$), under the condition that the selection result $\hat{g}$ is given, (\ref{eq:tildeT_uzg}) still holds with matrix $E^{(\hat{g})}$, which concludes the proof. 
\end{proof}

\begin{remark}[Generalization of Theorem \ref{thm:Tchidist}]
Theorem \ref{thm:Tchidist} holds if the selection event of the estimated block structure $\hat{g}$ can be formulated as a set of quadratic inequalities in terms of the data vector $\bm{x}$, by modifying the definition of the function $\mathcal{M}$. In other words, for a selected block structure $\hat{g}$, there exists some $\mathcal{I}_{\hat{g}} \in \mathbb{N}$ and $\{ Q^{(\hat{g}, i)}, \bm{\alpha}^{(\hat{g}, i)}, \beta^{(\hat{g}, i)} \}$, $i = 1, \dots, \mathcal{I}_{\hat{g}}$, and the selection event of $\hat{g}$ is represented by 
\begin{eqnarray}
\label{eq:condtion_ghat_general}
\hat{g} \in \mathcal{M} (\bm{x}) \iff \cap_{i \in \{ 1, \dots, \mathcal{I}_{\hat{g}} \}} \left\{ \bm{x}^{\top} Q^{(\hat{g}, i)} \bm{x} + (\bm{\alpha}^{(\hat{g}, i)})^{\top} \bm{x} + \beta^{(\hat{g}, i)} \geq 0 \right\}. 
\end{eqnarray}
Let $g(i)$ be the $i$th pattern of all the block structures with $K \times H$ blocks or less, where $i = 1, \dots, |\mathcal{G}_{KH}|$. Then, if we set $\mathcal{I}_{\hat{g}} \equiv |\mathcal{G}_{KH}|$, $Q^{(\hat{g}, i)} = E^{(g(i))} - E^{(\hat{g})}$, $\bm{\alpha}^{(\hat{g}, i)} = \bm{0}$, and $\beta^{(\hat{g}, i)} = 0$, the selection event in (\ref{eq:condtion_ghat_general}) will lead to the use of a squared residue solution. 
\end{remark}

It must be noted that if there exists multiple sets of cluster memberships that minimize the squared residue $\sigma^2$, which occurs with probability zero, from the discussion in Section \ref{sec:srm}, then Theorem \ref{thm:Tchidist} will hold for any one of them. Moreover, we define that a set of block memberships $g'$ is a \textit{refinement} of $g$ iff any block in $g'$ is a submatrix of some block in $g$. If $\hat{g}'$ is a refinement of $\hat{g}$, then Theorem \ref{thm:Tchidist} will also hold when $\hat{g}$ is replaced by $\hat{g}'$. In other words, we cannot detect that a given block structure represents a ``finer division than necessary'' with the proposed test; solving this problem is beyond the scope of this paper. 


\subsection{Statistical test based on truncated chi distribution}
\label{sec:test_tchi}

To perform a statistical test based on Theorem \ref{thm:Tchidist}, we have to specify the truncation interval of $\hat{M}^{(\hat{g})} \equiv \{ t \geq 0: \hat{\mathcal{M}} (t \sigma_0 \bm{u} + \bm{z}) = \hat{g} \}$. As shown in (\ref{eq:condtion_ghat}), this is equivalent to an interval satisfying the following condition for all $g$: 
\begin{eqnarray}
\label{eq:slct_cond}
(t \sigma_0 \bm{u} + \bm{z})^{\top} \left( E^{(g)} - E^{(\hat{g})} \right) (t \sigma_0 \bm{u} + \bm{z}) \geq 0. 
\iff f^{(g, \hat{g})} (t) \equiv a^{(g, \hat{g})} t^2 + b^{(g, \hat{g})} t + c^{(g, \hat{g})} \geq 0, 
\end{eqnarray}
where
\begin{align}
a^{(g, \hat{g})} &\equiv \sigma_0^2 \bm{u}^{\top} \left( E^{(g)} - E^{(\hat{g})} \right) \bm{u}, \nonumber \\
b^{(g, \hat{g})} &\equiv \sigma_0 \left[ \bm{u}^{\top} \left( E^{(g)} - E^{(\hat{g})} \right) \bm{z} + \bm{z}^{\top} \left( E^{(g)} - E^{(\hat{g})} \right) \bm{u} \right], \nonumber \\
c^{(g, \hat{g})} &\equiv \bm{z}^{\top} \left( E^{(g)} - E^{(\hat{g})} \right) \bm{z}. 
\end{align}

From the definition of $\bm{u}$ and $\bm{z}$ in (\ref{eq:defruz}), we have $E^{(\hat{g})} \bm{u} = \bm{u}$ and $E^{(\hat{g})} \bm{z} = \bm{0}$, which simplifies the above coefficients $a^{(g, \hat{g})}$, $b^{(g, \hat{g})}$, and $c^{(g, \hat{g})}$ as follows: 
\begin{align}
a^{(g, \hat{g})} &= - \sigma_0^2 \bm{u}^{\top} \left( I - E^{(g)} \right) \bm{u} = - \sigma_0^2 \left\| \left( I - E^{(g)} \right) \bm{u} \right\|_2^2 \leq 0, \nonumber \\
b^{(g, \hat{g})} &= 2 \sigma_0 \bm{u}^{\top} E^{(g)} \bm{z}, \nonumber \\
c^{(g, \hat{g})} &= \bm{z}^{\top} E^{(g)} \bm{z} = \| E^{(g)} \bm{z} \|_2^2 \geq 0. 
\end{align}
Here, in the transformation of $b^{(g, \hat{g})}$, we used the fact that matrices $E^{(g)}$ and $E^{(\hat{g})}$ are symmetric. 

We consider the condition under which (\ref{eq:slct_cond}) holds in the two cases, $a^{(g, \hat{g})} = 0$ and $a^{(g, \hat{g})} < 0$. 
\begin{itemize}
\item If $a^{(g, \hat{g})} = 0$, we have $E^{(g)} \bm{u} = \bm{u}$, which results in that $b^{(g, \hat{g})} = 2 \sigma_0 \bm{u}^{\top} \bm{z} = 0$ (since $\bm{u}^{\top} \bm{z} = 0$ holds from (\ref{eq:uz_zero})).  Therefore, in this case, the selection condition (\ref{eq:slct_cond}) always holds. 
\item If $a^{(g, \hat{g})} < 0$, $\max_t f^{(g, \hat{g})} (t) \geq f^{(g, \hat{g})} (0) = c^{(g, \hat{g})} \geq 0$. Therefore, for $t \geq 0$, the interval that satisfies $f^{(g, \hat{g})} (t) \geq 0$ is $\left[ 0, \frac{- b^{(g, \hat{g})} - \sqrt{(b^{(g, \hat{g})})^2 - 4 a^{(g, \hat{g})} c^{(g, \hat{g})}}}{2 a^{(g, \hat{g})}} \right]$. 
\end{itemize}

Overall, the interval of $t$ where (\ref{eq:slct_cond}) holds is given by
\begin{eqnarray}
\label{eq:t_interval}
\hat{M}^{(\hat{g})} = \left[0, \beta^{(\hat{g})} \right], \ \ \ 
\beta^{(\hat{g})} \equiv \min_{g: a^{(g, \hat{g})} \neq 0} \left( \frac{- b^{(g, \hat{g})} - \sqrt{(b^{(g, \hat{g})})^2 - 4 a^{(g, \hat{g})} c^{(g, \hat{g})}}}{2 a^{(g, \hat{g})}} \right). 
\end{eqnarray}
It must be noted that $\cap_{g \in \mathcal{G}_{KH}, g \neq \hat{g}} \left( a^{(g, \hat{g})} < 0 \right)$ holds almost surely, based on a similar discussion as that in Section \ref{sec:srm}. Formally, for a fixed $g, g' \in \mathcal{G}_{KH}$, $\bm{y} \equiv E^{(g')} \bm{x}$ follows a Gaussian distribution. If $g \neq g'$, $E^{(g)} - E^{(g')}$ is not a zero matrix, and thus the Lebesgue measure of a set of points $\bm{y}$ satisfying $\bm{y}^{\top} \left( E^{(g)} - E^{(g')} \right) \bm{y} = 0$ is zero. Similarly, $\| \bm{y} \|_2^2 > 0$ holds with probability one. By combining these facts, $a^{(g, g')} \equiv \frac{1}{\| \bm{y} \|_2^2} \bm{y}^{\top} \left( E^{(g)} - E^{(g')} \right) \bm{y} \neq 0$ holds for a fixed combination of $(g, g')$ satisfying $g \neq g'$ almost surely. Therefore, we have 
\begin{align}
&\mathrm{Pr} \left( \exists g \in \mathcal{G}_{KH},\ \mathrm{s.t.}\ g \neq \hat{g},\ a^{(g, \hat{g})} = 0 \right) \nonumber \\
&\leq \mathrm{Pr} \left( \exists g, g' \in \mathcal{G}_{KH},\ \mathrm{s.t.}\ g \neq g',\ a^{(g, g')} = 0 \right) \nonumber \\
&\leq \sum_{g, g' \in \mathcal{G}_{KH}, g \neq g'} \mathrm{Pr} \left( a^{(g, g')} = 0 \right) = 0. 
\end{align}
To derive the last equation, we used the fact that $\{(g, g'): g, g' \in \mathcal{G}_{KH}, g \neq g' \}$ is a finite set. 

We denote a set of cluster memberships attaining the boundary of this interval as $\tilde{g}$, that is, 
\begin{eqnarray}
\label{eq:tilde_g}
\tilde{g} \equiv \argmin_{g: a^{(g, \hat{g})} \neq 0} \left( \frac{- b^{(g, \hat{g})} - \sqrt{(b^{(g, \hat{g})})^2 - 4 a^{(g, \hat{g})} c^{(g, \hat{g})}}}{2 a^{(g, \hat{g})}} \right). 
\end{eqnarray}

From Theorem \ref{thm:Tchidist}, given $\{ \hat{g}, \bm{z}, \bm{u} \}$, a $p$-value $p_T$ of the test statistic $T$ in (\ref{eq:defT}) is given by
\begin{eqnarray}
\label{eq:pval}
p_T = \begin{cases}
1 - \frac{\gamma \left( \frac{np-KH}{2}, \frac{T^2}{2} \right)}{\gamma \left( \frac{np-KH}{2}, \frac{ \left( \beta^{(\hat{g})} \right)^2}{2} \right)}
\sim U[0, 1] & \mathrm{if}\ 0 \leq T \leq \beta^{(\hat{g})}, \\
0 & \mathrm{otherwise}, 
\end{cases}
\end{eqnarray}
where $\gamma (\cdot, \cdot)$ is the lower incomplete gamma function. This holds from the fact that, for any random variable $X$ with a probability density function $f(x)$, $F(X) \equiv \int_{-\infty}^X f(x) \mathrm{d}x \sim U[0, 1]$. To derive the $p$-value in (\ref{eq:pval}), we used the fact that the cumulative distribution function of chi-square distribution with $c$ degrees of freedom and with truncation interval of $[0, a]$ is given by 
\begin{eqnarray}
\begin{cases}
F(x) = 0 & \mathrm{if}\ x < 0, \\
F(x) = \frac{\gamma(c/2, x/2)}{\gamma(c/2, a/2)} & \mathrm{if}\ 0 \leq x \leq a, \\
F(x) = 1 & \mathrm{if}\ x > 1. 
\end{cases}
\end{eqnarray}


\subsection{Approximated test based on simulated annealing}
\label{sec:test_approx}

The exact statistical test in Section \ref{sec:test_tchi} requires us to search (i) the optimal set of cluster memberships $\hat{g}$, which minimizes the squared residue in (\ref{eq:sr_def}), and (ii) the set of cluster memberships $\tilde{g}$ in (\ref{eq:tilde_g}), which determines the truncation interval. We can see that the number of mutually different patterns of block structures with \textbf{exactly} $K \times H$ blocks is lower bounded by $K^{n - K} H^{p - H}$ (see Appendix \ref{sec:ap_patterns} for more detailed discussions). 

To cope with such combinatorial explosion, we propose an approximated statistical test based on SA, besides the exact test described in Section \ref{sec:test_tchi}. SA is an iterative algorithm that can be used for obtaining approximated solutions of combinatorial optimization problems \cite{Cerny1985, Kirkpatrick1983}; its basic procedure is given as follows: 
\begin{enumerate}
\item Define a cooling schedule or the sequence of temperatures $\{T_t\}_{t = 0}^{\infty}$, a threshold $\epsilon$, a finite set of states $\mathcal{S}$, and an objective function $f$ on $\mathcal{S}$. For all the experiments, we set the threshold at $\epsilon = 10^{-6}$. Our purpose is to find a state $x \in \mathcal{S}$ that minimizes $f(x)$. For each state $x \in \mathcal{S}$, we also define a set of neighbors $N(x) \subseteq \mathcal{S}$ and a transition probability $R(x, x')$ from state $x$ to $x'$, for all $x' \in \mathcal{S}$, where $R(x, x') > 0$ if $x' \in N(x)$ and $R(x, x') = 0$ otherwise. Finally, define an initial step $t \gets 0$ and initial state $x_0 \in \mathcal{S}$, and let $f^{(0)} \equiv f(x_0)$. 
\item If $T_t < \epsilon$, stop the algorithm and output the current state $x_t$. Otherwise, randomly choose a neighbor $x'$ of the current state $x_t$ (i.e., $x' \in N(x_t)$) with probability $R(x_t, x')$. Let $f' \equiv f(x')$ and $\Delta f \equiv f' - f^{(t)}$. 
\begin{itemize}
  \item If $\Delta f < 0$, then move to state $x'$ and set $x_{t + 1} = x'$ and $f^{(t + 1)} = f'$. 
  \item Otherwise, with probability $\exp \left( -\frac{\Delta f}{T_t} \right)$, move to state $x'$ and set $x_{t + 1} = x'$ and $f^{(t + 1)} = f'$. Otherwise, stay at the current state $x_t$ and set $x_{t + 1} = x_t$ and $f^{(t + 1)} = f^{(t)}$. 
\end{itemize}
\label{enm:repeat_sa}
\item Let $t \gets t + 1$ and go to \ref{enm:repeat_sa}. 
\end{enumerate}
It has been proven that the solution given by the above SA algorithm converges in probability to the global optimal solution of a given problem, under the following conditions \cite{Hajek1988}: 
\begin{enumerate}
\renewcommand{\theenumi}{(\alph{enumi})}
\renewcommand{\labelenumi}{(\alph{enumi})}
\item \textbf{Irreducibility}: we call that the state $y$ is \textit{reachable} at \textit{height} $E$ from state $x$ if $x = y$ or a sequence of states $x=x_1, x_2, \dots, x_p=y$ exists such that (1) $R(x_t, x_{t+1}) > 0$, for all $t \in \{1, \dots, p-1 \}$, and (2) $f(x_t) \leq E$, for all $t \in \{1, \dots, p \}$. We simply call that $y$ is reachable from $x$ if $y$ is reachable from $x$ at some height $E$. The first condition is that for any pair of states $(x, y)$, $y$ is reachable from $x$. 
\label{enm:irreducibility}
\item \textbf{Weak reversibility}: The second condition is that, for any $E \in \mathbb{R}$ and for any pair of states $(x, y)$, $y$ is reachable at height $E$ from $x$ iff $x$ is reachable at height $E$ from $y$. 
\label{enm:wr}
\item We call that state $x$ is a \textit{local minimum} if no state $y \in \mathcal{S}$ satisfying $f(y) < f(x)$ (i.e., a better solution) is reachable at height $f(x)$ from $x$. In other words, to find a better solution from a local minimum $x$, we need to pass through some ``worse'' states, where the value of the objective function is larger than that of $x$. We define that the \textit{depth} of a local minimum $x$ is $+\infty$ if $x$ is a global optimal state; otherwise, it is the minimum $E > 0$ such that some state $y$ (i.e., better solution) with $f(y) < f(x)$ exists and $y$ is reachable at height $f(x) + E$ from $x$. 
The third condition is that the cooling schedule of temperature satisfies the following conditions: (1) $T_t \geq T_{t + 1}$, for all $t \geq 0$, (2) $\lim_{t \to \infty} T_t = 0$, and (3) $\sum_{t=0}^{\infty} \exp \left( -\frac{d^*}{T_t} \right) = +\infty$, where $d^*$ is the maximum depth of all the states that are locally, but not globally, optimal solutions. 
\label{enm:temp_cond}
\end{enumerate}

\begin{algorithm}[t]
\caption{SA algorithm for finding the minimum squared residue solution $\hat{g}$. }         
\label{algo:min_sr}
\begin{algorithmic}[1]
\REQUIRE A cooling schedule of temperature $\{ T_t \}_{t=0}^{\infty}$ and a threshold $\epsilon$. 
\ENSURE Approximated optimal set of cluster memberships $\hat{g}$ in terms of the squared residue. 
\STATE $t \gets 0$. 
\STATE Randomly generate initial cluster memberships: $\hat{g} = (\hat{g}^{(1)}, \hat{g}^{(2)})$. 
\STATE Compute the initial value of the objective function: $f \gets \bm{x}^{\top} E^{(\hat{g})} \bm{x}$. 
\WHILE{$T_t \geq \epsilon$}
\STATE Randomly choose a row or column index $m$ from the uniform distribution on $\{1, \dots, n + p\}$. 
  \IF{$m \leq n$}
  \STATE $i \gets m$. 
  \STATE Randomly generate a new cluster index $k'$ of the $i$th \textbf{row} from the uniform distribution on $\{1, \dots, K\} \setminus \hat{g}^{(1)}_i$. Let $\hat{g}'$ be the set of cluster memberships given by $\hat{g}' = ((\hat{g}')^{(1)}, (\hat{g}')^{(2)})$, $(\hat{g}')^{(1)}_i = k'$, $(\hat{g}')^{(1)}_{i'} = \hat{g}^{(1)}_{i'}$, for $i' \neq i$, and $(\hat{g}')^{(2)} = \hat{g}^{(2)}$. 
  \ELSE
  \STATE $j \gets m - n$. 
  \STATE Randomly generate a new cluster index $h'$ of the $j$th \textbf{column} from the uniform distribution on $\{1, \dots, H\} \setminus \hat{g}^{(2)}_j$. Let $\hat{g}'$ be the set of cluster memberships given by $\hat{g}' = ((\hat{g}')^{(1)}, (\hat{g}')^{(2)})$, $(\hat{g}')^{(1)} = \hat{g}^{(1)}$, $(\hat{g}')^{(2)}_j = h'$, and $(\hat{g}')^{(2)}_{j'} = \hat{g}^{(2)}_{j'}$, for $j' \neq j$. 
  \ENDIF
\STATE Compute the value of the objective function: $f' \gets \bm{x}^{\top} E^{(\hat{g}')} \bm{x}$. 
\STATE $\Delta f \gets f' - f$. 
  \IF{$\Delta f < 0$}
  \STATE $\hat{g} \gets \hat{g}'$, $f \gets f'$. 
  \ELSE
  \STATE With probability $\exp \left( - \frac{\Delta f}{T_t} \right)$, $\hat{g} \gets \hat{g}'$, $f \gets f'$. 
  \ENDIF
\STATE $t \gets t + 1$. 
\ENDWHILE
\end{algorithmic}
\end{algorithm}

\begin{algorithm}[p]
\caption{SA algorithm for finding the solution $\tilde{g}$ of the truncation interval. }         
\label{algo:interval}
\begin{algorithmic}[1]
\REQUIRE Optimal set of cluster memberships $\hat{g}$ in terms of the squared residue, a cooling schedule of temperature $\{ T_t \}_{t=0}^{\infty}$, and a threshold $\epsilon$. 
\ENSURE Approximated optimal set of cluster memberships $\tilde{g}$ for determining the truncation interval. 
\STATE $t \gets 0$. 
\STATE Randomly generate initial cluster memberships: $\tilde{g} = (\tilde{g}^{(1)}, \tilde{g}^{(2)})$. 
\STATE Compute the initial value of the objective function: if $a^{(\tilde{g}, \hat{g})} = 0$, then $f \gets +\infty$; otherwise, $f \gets (- b^{(\tilde{g}, \hat{g})} - \sqrt{(b^{(\tilde{g}, \hat{g})})^2 - 4 a^{(\tilde{g}, \hat{g})} c^{(\tilde{g}, \hat{g})}})/(2 a^{(\tilde{g}, \hat{g})})$. 
\WHILE{$T_t \geq \epsilon$}
\STATE Randomly choose the size $s$ of a subset of row or column indices from $\{1, \dots, n + p\}$: 
\begin{eqnarray}
\label{eq:algo2_trans_prob}
\begin{cases}
s \gets 1 & \mathrm{with\ probability}\ \frac{1}{2} + \frac{1}{2^{n + p}}, \\
s \gets s' & \mathrm{with\ probability}\ \frac{1}{2^{s'}},\ \mathrm{for}\ s' \in \{2, \dots, n + p\}. 
\end{cases}
\end{eqnarray}
\STATE Randomly choose a set of $s$ row or column indices $\mathcal{S}$ without duplication from the uniform distribution. 
\STATE $\tilde{g}' \gets \tilde{g}$. 
\FOR{each row or column index in $\mathcal{S}$}
  \IF{the $i$th \textbf{row }is selected}
    \STATE Randomly generate a new cluster index $k'$ of the $i$th \textbf{row} from the uniform distribution on $\{1, \dots, K\} \setminus \tilde{g}^{(1)}_i$. $(\tilde{g}')^{(1)}_i \gets k'$. 
  \ELSIF{the $j$th \textbf{column} is selected}
    \STATE Randomly generate a new cluster index $h'$ of the $j$th \textbf{column} from the uniform distribution on $\{1, \dots, H\} \setminus \tilde{g}^{(2)}_j$. $(\tilde{g}')^{(2)}_j \gets h'$. 
  \ENDIF
\ENDFOR
\STATE Compute the value of the objective function: if $a^{(\tilde{g}', \hat{g})} = 0$, then $f \gets +\infty$; otherwise, $f' \gets (- b^{(\tilde{g}', \hat{g})} - \sqrt{(b^{(\tilde{g}', \hat{g})})^2 - 4 a^{(\tilde{g}', \hat{g})} c^{(\tilde{g}', \hat{g})}})/(2 a^{(\tilde{g}', \hat{g})})$. 
\STATE $\Delta f \gets f' - f$. 
\IF{$\Delta f < 0$}
  \STATE $\tilde{g} \gets \tilde{g}'$, $f \gets f'$. 
\ELSE
  \STATE With probability $\exp \left( - \frac{\Delta f}{T_t} \right)$, $\tilde{g} \gets \tilde{g}'$, $f \gets f'$. 
\ENDIF
\STATE $t \gets t + 1$. 
\ENDWHILE
\end{algorithmic}
\end{algorithm}

Algorithm \ref{algo:min_sr} is the SA algorithm for obtaining an approximated solution for the optimal set of cluster memberships $\hat{g}$ in terms of the squared residue. In this algorithm, from (\ref{eq:g_hat}), we define that the set of states $\mathcal{S}$ and the objective function $f$ are given by $\mathcal{S} \equiv \mathcal{G}_{KH}$ and $f(g) \equiv \bm{x}^{\top} E^{(g)} \bm{x}$, respectively. In each step of the algorithm, neighbors $N(g)$ of the current state $g$ are defined as a set of all the cluster memberships that differ from $g$ in exactly one row or column. It must be noted that the size of such neighbors is $|N(g)| = n (K-1) + p (H-1)$. We choose a neighbor $g'$ from the uniform distribution on $N(g)$ (i.e., with probability $R(g, g') = 1/|N(g)|$). By these definitions, Algorithm \ref{algo:min_sr} satisfies the conditions of \ref{enm:irreducibility} irreducibility and \ref{enm:wr} weak reversibility. 

Algorithm \ref{algo:interval} is the SA algorithm used for finding an approximated solution of the cluster memberships $\tilde{g}$, which determines the truncation interval. In this algorithm, we define that the set of states $\mathcal{S}$ and the objective function $f$ are given by $\mathcal{S} \equiv \mathcal{G}_{KH}$ and $f(g) \equiv \frac{- b^{(g, \hat{g})} - \sqrt{(b^{(g, \hat{g})})^2 - 4 a^{(g, \hat{g})} c^{(g, \hat{g})}}}{2 a^{(g, \hat{g})}}$, respectively. Unlike Algorithm \ref{algo:min_sr}, we have to consider the \textit{feasibility} of a solution $\tilde{g}$, that is,  it should satisfy $a^{(\tilde{g}, \hat{g})} < 0$. To guarantee the condition of \ref{enm:irreducibility} irreducibility while avoiding infeasible solutions, we defined the neighbors $N(g)$ of the current state $g$ as $N(g) \equiv \mathcal{G}_{KH}$. By this definition, for any pair of states $(g, g')$, transition from $g$ to $g'$ is possible with non-zero probability: $R(g, g') > 0$. Accordingly, we restrict the significant change in the state by controlling the transition probability $R(g, g')$, as in (\ref{eq:algo2_trans_prob}). By setting the objective function values for infeasible solutions at $+\infty$, we can avoid moving to them throughout the algorithm while satisfying the conditions of \ref{enm:irreducibility} irreducibility and \ref{enm:wr} weak reversibility. 

Regarding the cooling schedule of temperature, we can use the following definition \cite{Hajek1988}, which satisfies the conditions (1), (2), and (3) in \ref{enm:temp_cond}: 
\begin{eqnarray}
\label{eq:algo1_t}
T_t = c / \log (t + 2) \ \ \ \ \ \mathrm{for\ all}\ t \geq 0, 
\end{eqnarray}
where $c$ is a constant satisfying $c \geq d^*$. In our cases, for instance, we can define the constant $c$ as follows: 
\begin{eqnarray}
c \equiv \|\bm{x} \|_2^2 - \frac{1}{np} \left( \begin{bmatrix}1 \cdots 1\end{bmatrix} \bm{x} \right)^2, 
\end{eqnarray}
for Algorithm \ref{algo:min_sr}, since $d^* \leq \max_{g \in \mathcal{G}_{KH}} \bm{x}^{\top} E^{(g)} \bm{x} - \min_{g \in \mathcal{G}_{KH}} \bm{x}^{\top} E^{(g)} \bm{x} \leq c$. Here, we take into account the fact that the cluster memberships $\underline{g} \equiv \argmax_{g \in \mathcal{G}_{KH}} \bm{x}^{\top} E^{(g)} \bm{x}$ are attained by assigning all the elements of an observed matrix into a single block, where the objective function value is given by $\bm{x}^{\top} E^{(\underline{g})} \bm{x} = \|\bm{x} \|_2^2 - \frac{1}{np} \left( \begin{bmatrix}1 \cdots 1\end{bmatrix} \bm{x} \right)^2$, and that $\min_{g \in \mathcal{G}_{KH}} \bm{x}^{\top} E^{(g)} \bm{x} \geq 0$. 
It must be noted that, in Algorithm \ref{algo:interval}, there is no state that is local but not global minimum (i.e., all local minima are also global minima); this is because, for any pair of states $(g, g')$, $g'$ is reachable at height $f(g)$ from $g$. Therefore, from the result in \cite{Hajek1988}, the convergence in probability to a global minimum state is guaranteed without the condition (3) in \ref{enm:temp_cond}. 

Practically, an algorithm based on the cooling schedule (\ref{eq:algo1_t}) is too slow, that is, it requires much computation time before convergence. Therefore, in the experiments in Section \ref{sec:exp}, we used the cooling schedule of $T_t = T_0 \times r^{t}$, for all $t \geq 0$, though this definition satisfies only the conditions (1) and (2), not (3). 


\section{Experiment}
\label{sec:exp}

To show the validity of our proposed test, we compared its behavior with that of a \textit{naive} statistical test, which does not consider the selection event. By ignoring the fact that the set of cluster memberships $\hat{g}$ was selected based on the data vector $\bm{x}$, we construct a naive test (\textbf{which is invalid in fact}) with test statistic $T$ in (\ref{eq:defT}) by assuming
\begin{eqnarray}
\label{eq:test_naive}
T | \{ \bm{z}, \bm{u} \} \sim \chi_{(np - KH)}, 
\end{eqnarray}
from (\ref{eq:tildeT_uz}). The $p$-value of such a naive test is given by
\begin{eqnarray}
\label{eq:pval_naive}
p_T = \begin{cases}
1 - \frac{\gamma \left( \frac{np-KH}{2}, \frac{T^2}{2} \right)}{\Gamma \left( \frac{np-KH}{2} \right)} & \mathrm{if}\ 0 \leq T, \\
0 & \mathrm{otherwise}, 
\end{cases}
\end{eqnarray}
where $\Gamma (\cdot)$ is the Gamma function. 

In the following sections \ref{sec: exp_real} and \ref{sec: exp_unreal}, we check the behavior of the $p$-values and the TPR and FPR when using the proposed and naive tests, in order to show the validity of our proposed method. 
In these sections, we use the term ``realizable case'' to indicate that the hypothetical cluster numbers of rows and columns $(K, H)$ are equal to the null ones $(K^{\mathrm{(N)}}, H^{\mathrm{(N)}})$, and the term ``unrealizable case'' to indicate that at least one of $K < K^{\mathrm{(N)}}$ and $H < H^{\mathrm{(N)}}$ holds, as described in Section \ref{sec:introduction}.

Aside from the experiments in this section, we conducted sensitivity analysis of the approximated version of the proposed test with respect to the cooling schedule of SA in the realizable case in Appendix \ref{sec:sensitivity_cooling}. Moreover, we conducted an additional experiment to employ an existing fast biclustering method \cite{Tan2014} instead of the proposed SA-based algorithm for estimating the cluster memberships in Appendix \ref{sec:kmeans_biclustering}.


\subsection{Exact test in realizable case: $(K, H) = (K^{\mathrm{(N)}}, H^{\mathrm{(N)}})$}
\label{sec: exp_real}

First, we check the behavior of the $p$-values calculated by using the proposed (\ref{eq:pval}) and naive (\ref{eq:pval_naive}) tests, under the condition that the given set of cluster numbers $(K, H)$ are equal to that of the null one $(K^{\mathrm{(N)}}, H^{\mathrm{(N)}})$. As shown in Section \ref{sec:test_tchi}, the $p$-value of the proposed test follows the uniform distribution on $[0, 1]$, while there is no such guarantee for that of the naive test. 

For experiment, we randomly generated data matrices with the sizes of $(n, p) = (5, 5), (6, 6), \dots, (9, 9)$. We set the null and hypothetical sets of cluster numbers at $(2, 2)$; we defined the null cluster memberships as $g^{\mathrm{(N)}, (1)}_i = (i \bmod K^{\mathrm{(N)}}) + 1$, for all $i$, and $g^{\mathrm{(N)}, (2)}_j = (j \bmod H^{\mathrm{(N)}}) + 1$, for all $j$. In regard to the mean vector $\bm{\mu}_0$, we tried the following five settings: 
\begin{eqnarray}
\bm{\mu}_0^{(l)} = \left( 1 - \frac{l - 1}{5} \right) \left[ \mathrm{vec} \left( \begin{bmatrix}
0.7 & 0.55 \\
0.5 & 0.6 \\
\end{bmatrix} \right) - 0.5 \right] + 0.5, \ \ \ \ \ l = 1, \dots, 5. 
\end{eqnarray}

Based on the above settings, we generated $1000$ data vectors by $\bm{x} \sim N(\bm{\mu}_0^{(l)}, 0.05^2 I)$, for each setting of matrix size $(n, p)$ and mean vector $\bm{\mu}_0$. Figure \ref{fig:As} shows the examples of the generated data matrices. For each generated data vector $\bm{x}$, we computed the squared residues of all the patterns of cluster memberships $g$. Subsequently, we chose the optimal set of cluster memberships $\hat{g}$ (i.e., solution with the minimum squared residue) and checked if it is equivalent to the null set of cluster memberships $g^{\mathrm{(N)}}$. For both cases of $\hat{g} = g^{\mathrm{(N)}}$ and $\hat{g} \neq g^{\mathrm{(N)}}$, we computed the test statistic $T$ in (\ref{eq:defT}), the truncated interval in (\ref{eq:t_interval}), and the $p$-values in (\ref{eq:pval}) and (\ref{eq:pval_naive}). Subsequently, we plotted the results as follows: 
\begin{itemize}
\item For the trials where $\hat{g} = g^{\mathrm{(N)}}$ holds (i.e., under null hypothesis), we plotted the $p$-values given by (\ref{eq:pval}) and (\ref{eq:pval_naive}), in Figures \ref{fig:pvalues_p} and \ref{fig:pvalues_n}, respectively. We also plotted (i) the test statistics $D\sqrt{r}$ of the Kolmogorov-Smirnov test \cite{Conover1999} for the $p$-values of the proposed and naive tests and (ii) the accuracy of the clustering algorithm $\mathcal{A}$, that is, the ratio of the number of such null cases (i.e., $\hat{g} = g^{\mathrm{(N)}}$) to the $1000$ trials, for each setting, in Figures \ref{fig:ks_test} and \ref{fig:accuracy}, respectively. 
\item For null (i.e., $\hat{g} = g^{\mathrm{(N)}}$) and alternative (i.e., $\hat{g} \neq g^{\mathrm{(N)}}$) cases, we plotted the FPR and TPR, in Figure \ref{fig:ratios_realizable}, respectively. 
\end{itemize}

From Figures \ref{fig:pvalues_p}, \ref{fig:pvalues_n}, and \ref{fig:ks_test}, we see that the distribution of the $p$-values of the proposed test was closer to the uniform distribution on $[0, 1]$ than that of the naive test, particularly when the difference in block-wise mean between the blocks was small. This result shows that the proposed test can successfully consider the selective bias of using the squared residue minimization solution, by using the truncated chi distribution in (\ref{eq:defT}). However, the $p$-values of the naive test based on (\ref{eq:test_naive}) did not follow the uniform distribution on $[0, 1]$, since we did not consider the selective bias and conducted tests based on the (not truncated) chi distribution in the naive test. It must be noted that, in our problem setting, unlike the common statistical tests, the assertion of the null hypothesis ($(g^{\mathrm{(N)}, (1)}, g^{\mathrm{(N)}, (2)}) = (\hat{g}^{(1)}, \hat{g}^{(2)})$) is stronger than that of alternative hypothesis ($(g^{\mathrm{(N)}, (1)}, g^{\mathrm{(N)}, (2)}) \neq (\hat{g}^{(1)}, \hat{g}^{(2)})$). This results in that the $p$-values of the naive test are biased toward \textbf{larger} values than the correct ones. 

In regard to the test performance, from the results in the top of Figure \ref{fig:ratios_realizable}, we see that the FPR was low in all the settings (i.e., proposed and naive; significance rate $\alpha = 0.1, 0.05$, and $0.01$; and block-wise mean $\bm{\mu}_0$). The results in the bottom of Figure \ref{fig:ratios_realizable} shows that the TPR of the proposed test was higher than that of the naive test in the same setting, which is consistent with the discussion in the previous paragraph. However, the TPR of the proposed test was not sufficiently close to one, in all the settings. This can be attributed to the few ``true positive'' cases in the realizable setting. Figure \ref{fig:accuracy} shows that the estimated block structure $\hat{g}$ that attained the minimum squared residue was equivalent to the null one $g^{\mathrm{(N)}}$ in most cases. In other words, almost all trials were ``null cases,'' where the small number of false negative cases significantly affect the TPR. Particularly, when there is an increase in the matrix size or in the difference in the block-wise mean between the blocks, it becomes easier to estimate the null block structure, and the clustering algorithm almost always outputs the correct cluster memberships. 

\begin{figure}[t]
  \centering
  \includegraphics[width=0.99\hsize]{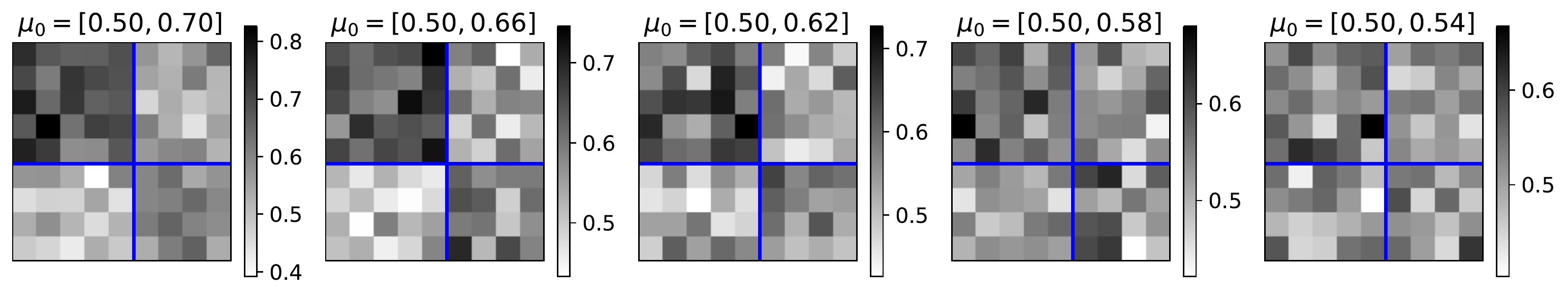}
  \caption{Examples of the observed data matrices with the size of $(n, p) = (9, 9)$, which are generated based on the different block-wise means. The title of each figure shows the range of the block-wise mean vector $\bm{\mu}_0$. The blue lines show the \textbf{null} cluster memberships. For visibility, we plotted the matrices whose rows and columns were sorted according to their null clusters.}
  \label{fig:As}
\end{figure}
\begin{figure}[p]
  \centering
  \includegraphics[width=0.9\hsize]{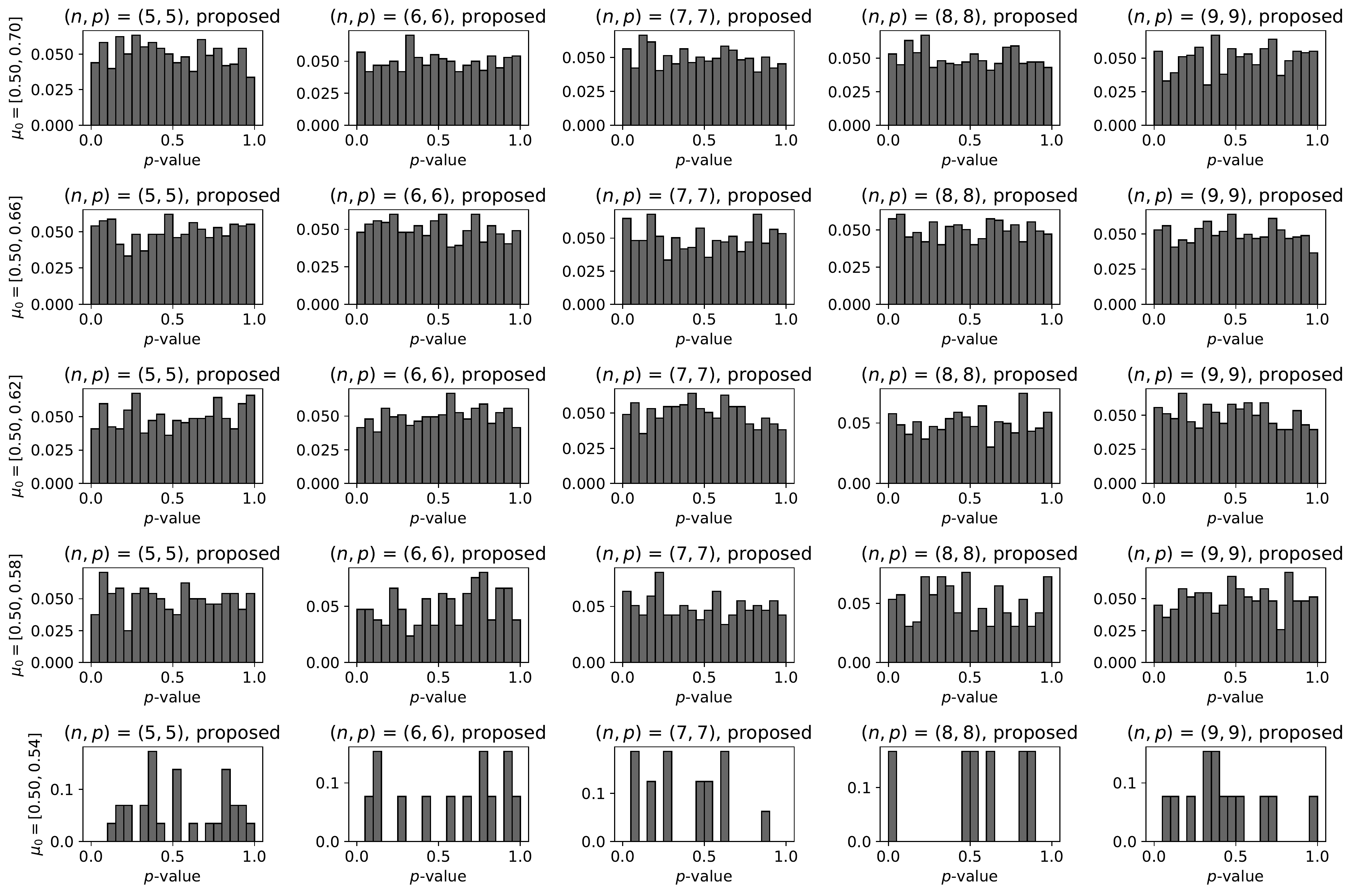}
  \caption{Histograms of $p$-values in the null case (i.e., $\hat{g} = g^{\mathrm{(N)}}$) for different matrix sizes, which was computed by the \textbf{proposed} test (\ref{eq:pval}) on the set of cluster memberships $\hat{g}$ with the minimum squared residue.}\vspace{3mm}
  \label{fig:pvalues_p}
  \includegraphics[width=0.9\hsize]{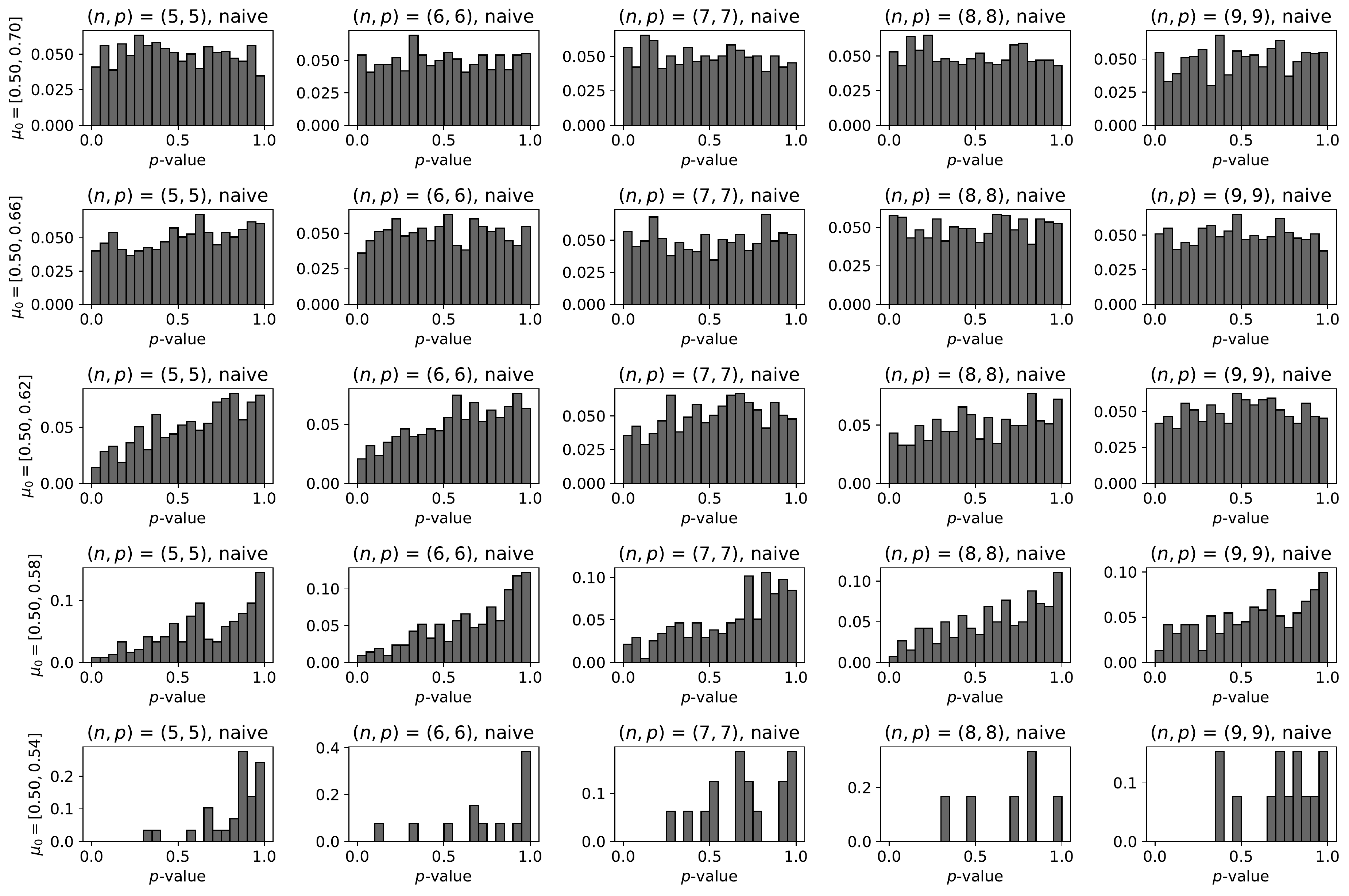}
  \caption{Histograms of $p$-values in the null case (i.e., $\hat{g} = g^{\mathrm{(N)}}$) for different matrix sizes, which was computed by the \textbf{naive} test (\ref{eq:pval_naive}).}
  \label{fig:pvalues_n}
\end{figure}
\begin{figure}[t]
  \centering
  \includegraphics[width=0.95\hsize]{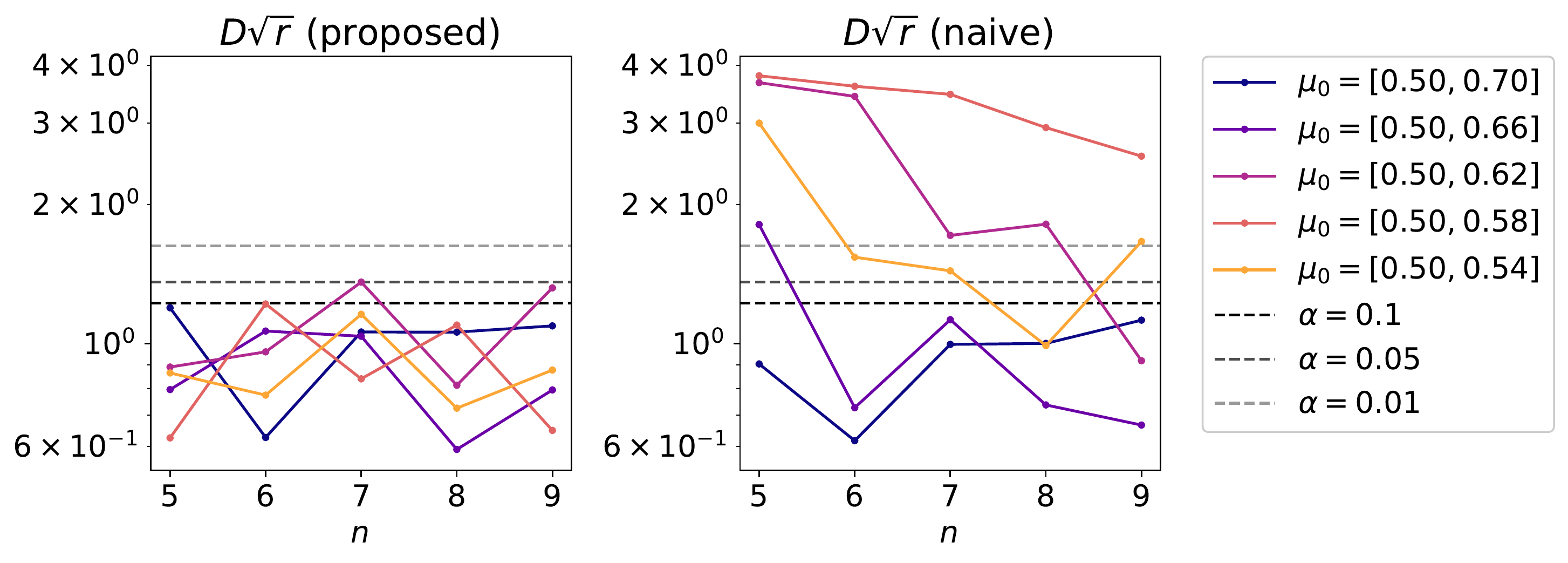}\vspace{-3mm}
  \caption{Test statistics $D\sqrt{r}$ of the Kolmogorov-Smirnov test \cite{Conover1999} for the $p$-values of the proposed (left) and naive (right) tests. The null hypothesis that $p$-value follows the uniform distribution on $[0, 1]$ is rejected if $D\sqrt{r} > \alpha$, where $\alpha$ is a given significance level.}\vspace{3mm}
  \label{fig:ks_test}
  \includegraphics[width=0.7\hsize]{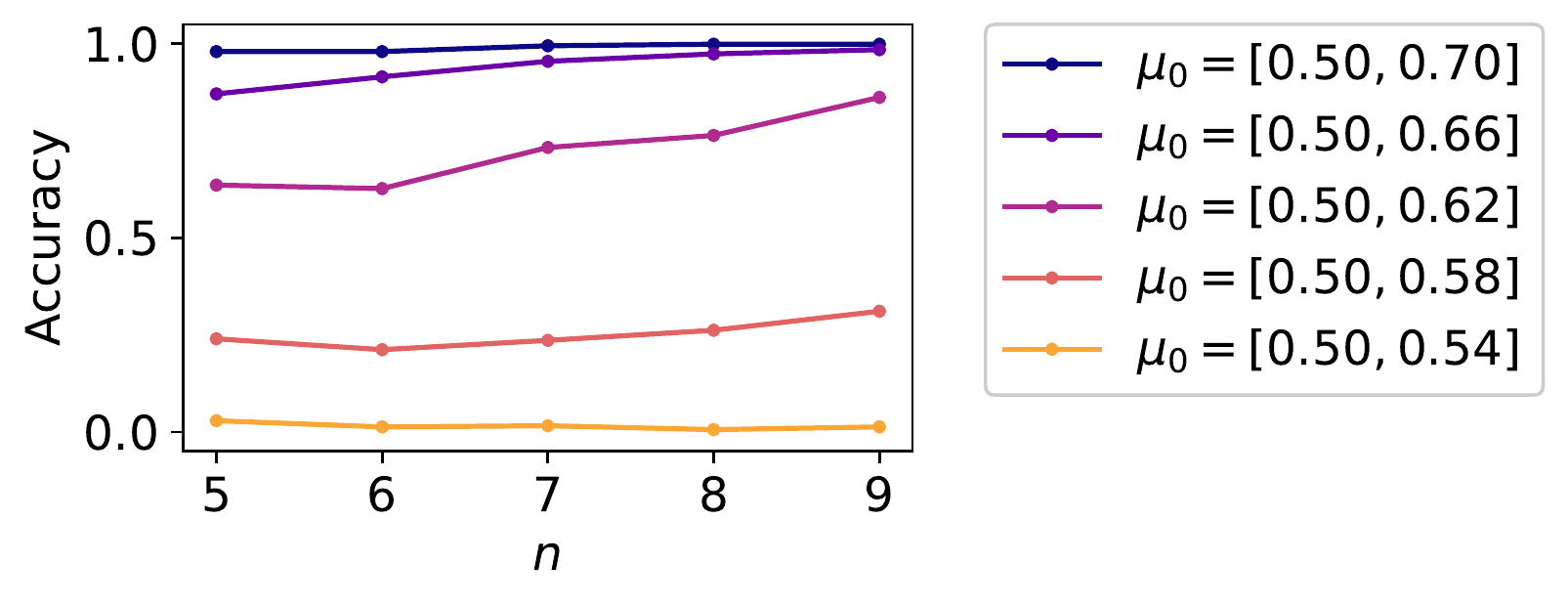}\vspace{-3mm}
  \caption{The ratio of the number of the null cases (i.e., $\hat{g} = g^{\mathrm{(N)}}$) for each setting of matrix size $(n, p)$ and mean vector $\bm{\mu}_0$. For this experiment, we used the setting of $n = p$.}
  \label{fig:accuracy}
\end{figure}
\begin{figure}[t]
  \centering
  \includegraphics[width=0.99\hsize]{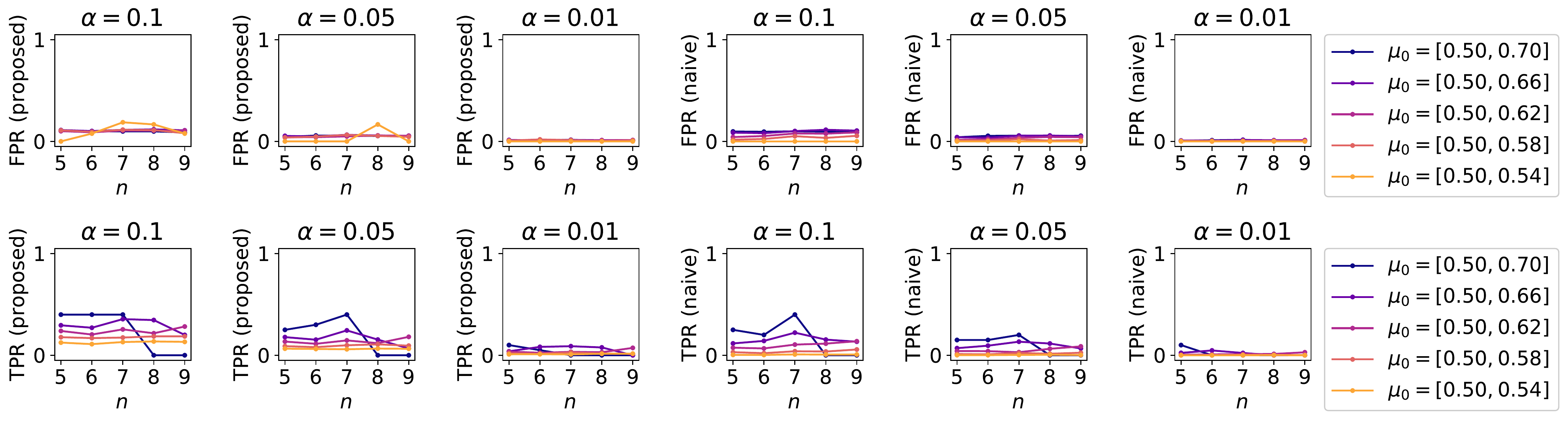}
  \caption{FPR and TPR in the realizable case with different significance rates (e.g., $\alpha = 0.1, 0.05$, and $0.01$), for the proposed (left) and naive (right) statistical tests. If there were no null (i.e., $\hat{g} = g^{\mathrm{(N)}}$) or alternative (i.e., $\hat{g} \neq g^{\mathrm{(N)}}$) cases, respectively, then the corresponding points of FPR or TPR would not have been plotted.}\vspace{3mm}
  \label{fig:ratios_realizable}
  \includegraphics[width=0.99\hsize]{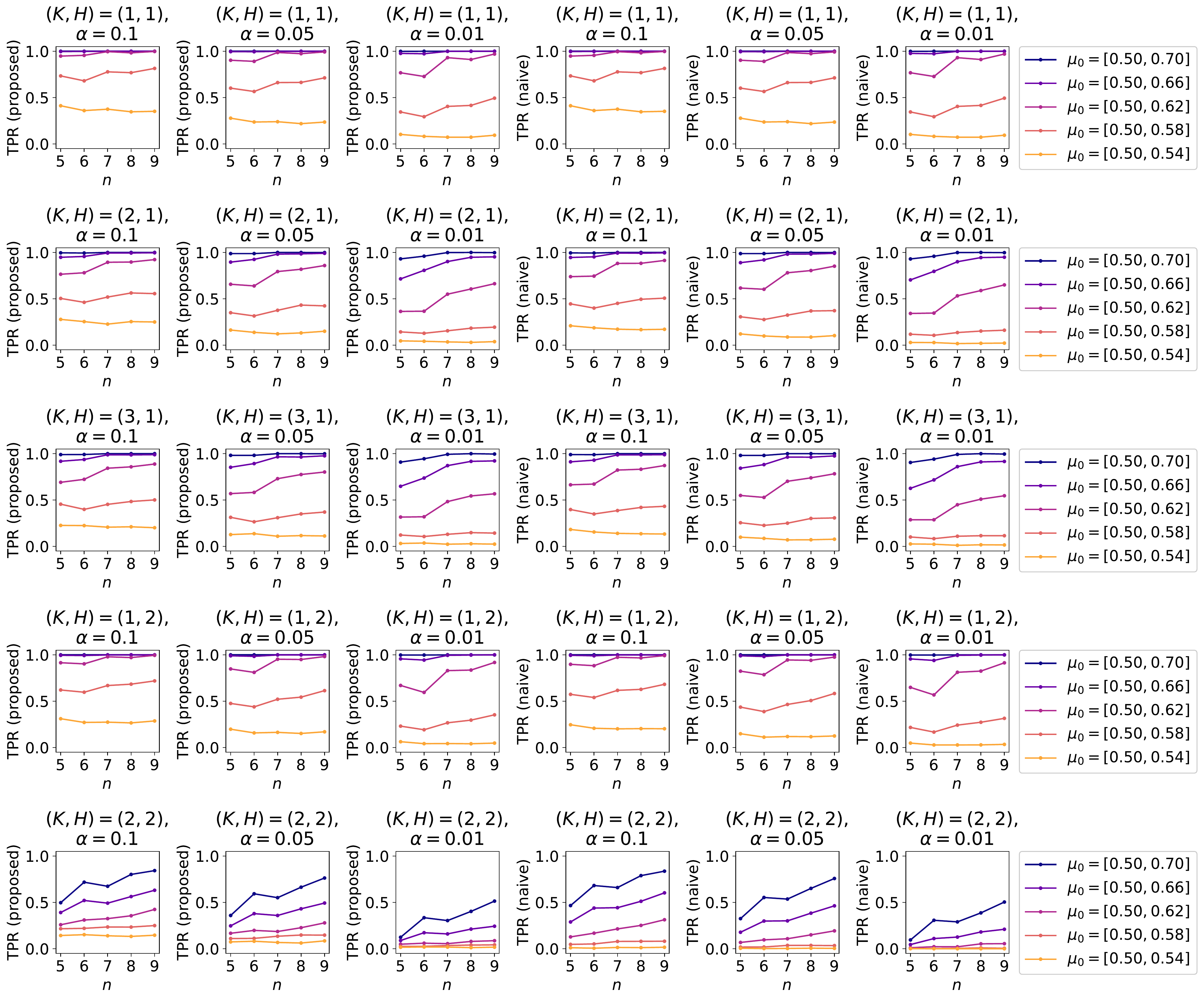}
  \caption{TPR in the unrealizable case (i.e., $K < K^{\mathrm{(N)}}$ or $H < H^{\mathrm{(N)}}$) with different significance rates (e.g., $\alpha = 0.1, 0.05$, and $0.01$), for the proposed (left) and naive (right) statistical tests.}
  \label{fig:ratios_unrealizable}
\end{figure}


\subsection{Exact test in unrealizable cases: $K < K^{\mathrm{(N)}}$ or $H < H^{\mathrm{(N)}}$}
\label{sec: exp_unreal}

Next, we compared the behavior of the proposed and naive tests in the unrealizable cases, that is, either $K < K^{\mathrm{(N)}}$ or $H < H^{\mathrm{(N)}}$ holds. 

For the experiment, we randomly generated data matrices with the sizes of $(n, p) = (5, 5), (6, 6), \dots, (9, 9)$. We set the null set of cluster numbers at $(3, 2)$ and defined the null cluster memberships as in Section \ref{sec: exp_real}. In regard to the mean vector $\bm{\mu}_0$, we tried the following five settings: 
\begin{eqnarray}
\bm{\mu}_0^{(l)} = \left( 1 - \frac{l - 1}{5} \right) \left[ \mathrm{vec} \left( \begin{bmatrix}
0.7 & 0.55 \\
0.5 & 0.6 \\
0.55 & 0.5 \\
\end{bmatrix} \right) - 0.5 \right] + 0.5, \ \ \ \ \ l = 1, \dots, 5. 
\end{eqnarray}

Based on the above settings, we generated $1000$ data vectors by $\bm{x} \sim N(\bm{\mu}_0^{(l)}, 0.05^2 I)$, for each setting of matrix size $(n, p)$ and mean vector $\bm{\mu}_0$. For each generated data vector $\bm{x}$, we computed the squared residues of all the patterns of cluster memberships $g$. Subsequently, we chose the optimal set of cluster memberships $\hat{g}$ with a given set of cluster numbers $(K, H)$. In regard to the hypothetical cluster numbers, we tried the following five settings: $(K, H) = (1, 1), (2, 1), (3, 1), (1, 2)$, and $(2, 2)$. For each setting, based on the selected result $\hat{g}$, we computed the test statistic $T$ in (\ref{eq:defT}), the truncated interval in (\ref{eq:t_interval}), and the $p$-values in (\ref{eq:pval}) and (\ref{eq:pval_naive}). Finally, we plotted the TPR of the proposed and naive tests in Figure \ref{fig:ratios_unrealizable}. 

Figure \ref{fig:ratios_unrealizable} shows that the TPR of the proposed test was higher than that of the naive test in the same setting; however, in most cases, there was a small difference between them. This may be attributed to the fact that we set the matrix size $(n, p)$ and the hypothetical block size $(K, H)$ at small numbers in order to perform the exact test, which is computationally expensive, and thus there is a marginal effect of selecting the optimal block structure $\hat{g}$ from all the patterns $\mathcal{G}_{KH}$. 
It must be noted that, unlike the realizable case in Section \ref{sec: exp_real}, the block structures output by the clustering algorithm were \textbf{always} different from the null ones because the hypothetical set of cluster numbers were insufficient to represent the null block structure in the unrealizable cases. In other words, all the $1000$ trials in each setting correspond to the alternative cases. 
The TPRs of the proposed and naive tests were comparable particularly under the following two settings: (1) the case where we set the hypothetical number of blocks at $(K, H) = (1, 1)$ and (2) the case where the difference in the null block-wise mean $\bm{\mu}$ between the blocks was relatively big. These results were caused by the nature of the biclustering problem itself, as well as by the limitation in the power of the proposed selective test. In the case of (1), we assume that the entire data matrix $A$ consists of a single block, and thus there is only one possible estimated bicluster structure $\hat{g}$, regardless of the selection event. Therefore, in this case, the proposed and naive tests are equivalent in the first place. As for the case of (2), since the difference in the null block-wise mean between the blocks was big, the test statistics of both the proposed and naive tests got large enough for the null hypothesis to be rejected. 


\subsection{Approximated test in both realizable and unrealizable cases}
\label{sec: exp_approx}

Finally, we checked the behavior of the approximated test introduced in Section \ref{sec:test_approx}. In both realizable and unrealizable cases, we generated data matrices with the sizes of $(n, p) = (10 + 2 \times m, 10 + 2 \times m)$, for $m = 0, 1, \dots, 4$, in the same way as that in Sections \ref{sec: exp_real} and \ref{sec: exp_unreal}. Concerning the following conditions, we used the same setting as in Sections \ref{sec: exp_real} and \ref{sec: exp_unreal}, respectively, for the realizable and unrealizable cases: the null and hypothetical sets of cluster numbers, the definition of null cluster memberships $g^{\mathrm{(N)}}$, mean vectors and the standard deviation $\sigma_0$, and the number of data vectors for each setting. Concerning the SA algorithms, both in the Algorithms \ref{algo:min_sr} and \ref{algo:interval}, we defined the cooling schedule as follows: $T_0 = 10$, $T_t = T_0 \times 0.99^t$ for all $t$. 

As in the cases of the exact tests, Figures \ref{fig:pvalues_p_approx} and \ref{fig:pvalues_n_approx}, respectively, show the histograms of the $p$-values of the proposed and naive approximated tests in the realizable case. For the realizable case, we also plotted (i) the test statistics $D\sqrt{r}$ of the Kolmogorov-Smirnov test \cite{Conover1999}, for the $p$-values of the proposed and naive tests, and (ii) the accuracy of the approximated clustering algorithm in Figures \ref{fig:ks_test_approx} and \ref{fig:accuracy_approx}, respectively. Figure \ref{fig:ratios_realizable_approx} shows the FPR and TPR in the realizable case, and Figure \ref{fig:ratios_unrealizable_approx} shows the TPR in the unrealizable cases. 

Figures \ref{fig:pvalues_p_approx}, \ref{fig:pvalues_n_approx}, and \ref{fig:ks_test_approx} show that the distribution of the $p$-values of the proposed test was closer to the uniform distribution on $[0, 1]$ than that of the naive test, as in the result of the exact test in Section \ref{sec: exp_real}. Concerning the test performance in the realizable case, Figure \ref{fig:ratios_realizable_approx} shows that the FPR was low in all the settings, and the TPR of the proposed test was higher than that of the naive test in the same setting. However, as in the exact case, the TPR of the proposed test was not sufficiently close to one in all the setting; this can be attributed to the few ``true positive'' cases. In the next Section \ref{sec: exp_approx_KH}, we checked more difficult cases for the approximated clustering algorithm, where the null cluster numbers were more than $(2, 2)$. 

\begin{figure}[t]
  \centering
  \includegraphics[width=0.9\hsize]{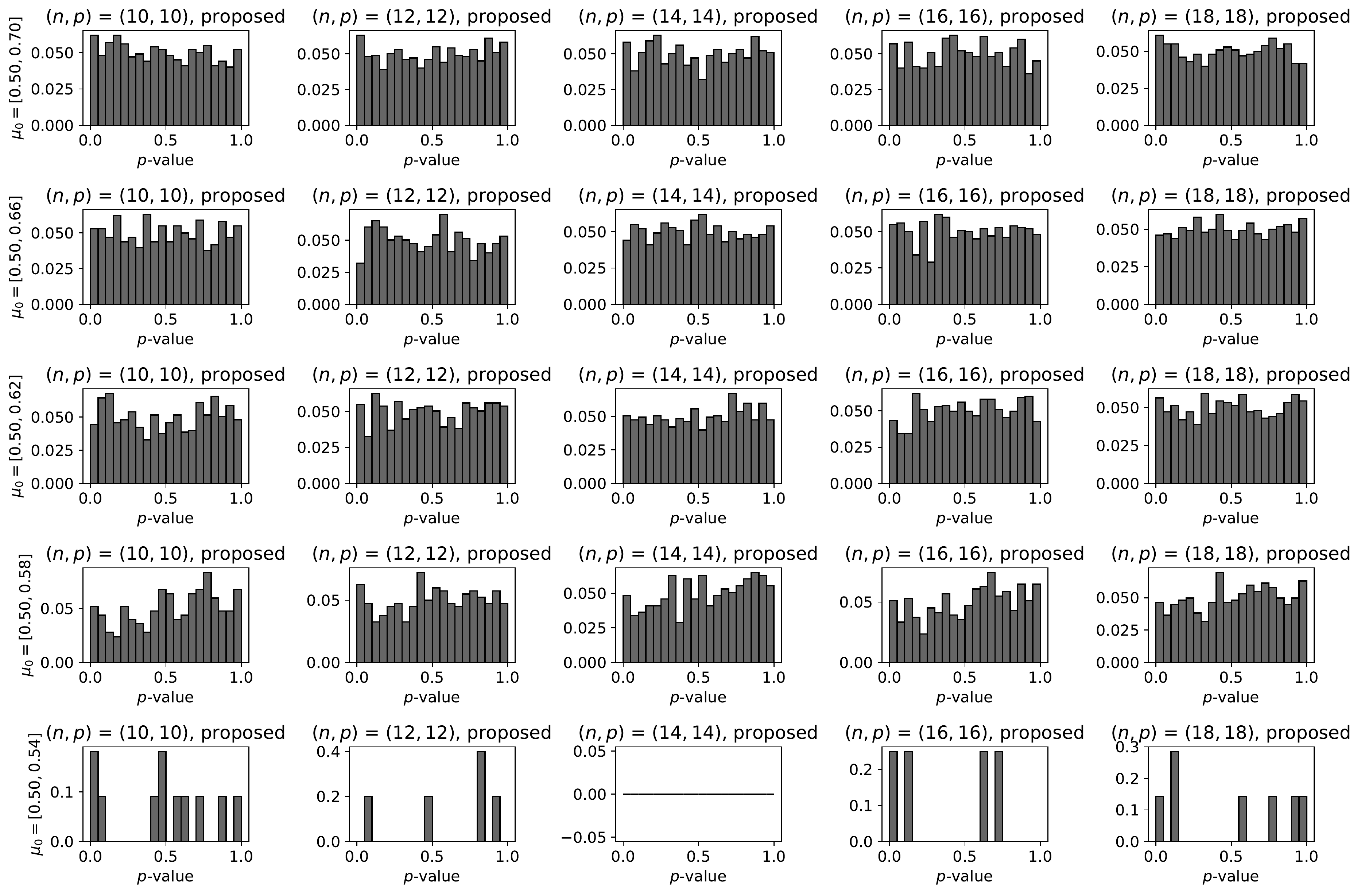}
  \caption{Histograms of $p$-values in the null case (i.e., $\hat{g} = g^{\mathrm{(N)}}$) for different matrix sizes, which was computed by the \textbf{approximated} version of the \textbf{proposed} test.}\vspace{3mm}
  \label{fig:pvalues_p_approx}
  \includegraphics[width=0.9\hsize]{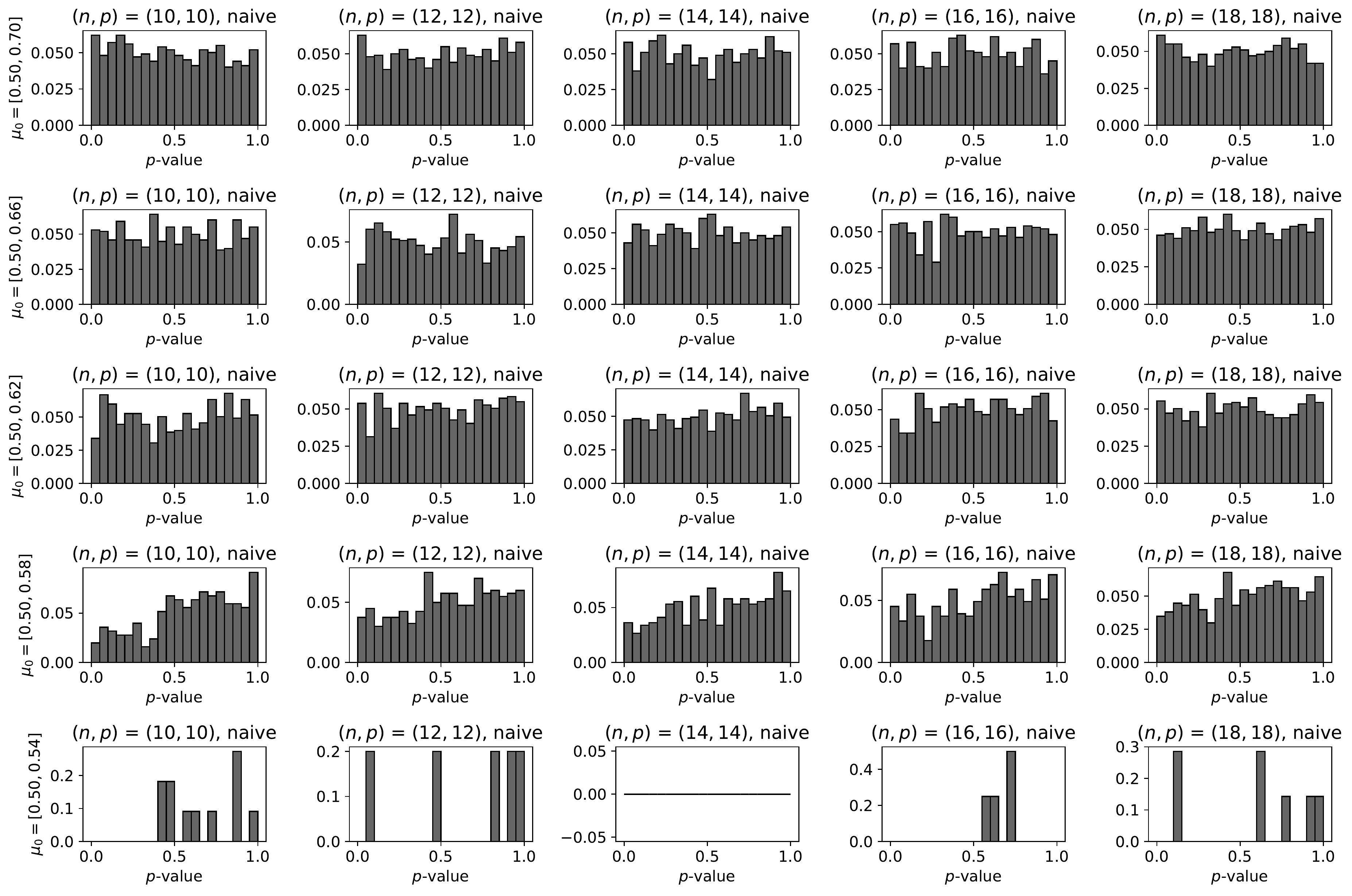}
  \caption{Histograms of $p$-values in the null case (i.e., $\hat{g} = g^{\mathrm{(N)}}$) for different matrix sizes, which was computed by the \textbf{approximated} version of the \textbf{naive} test (\ref{eq:pval_naive}).}
  \label{fig:pvalues_n_approx}
\end{figure}
\begin{figure}[t]
  \centering
  \includegraphics[width=0.95\hsize]{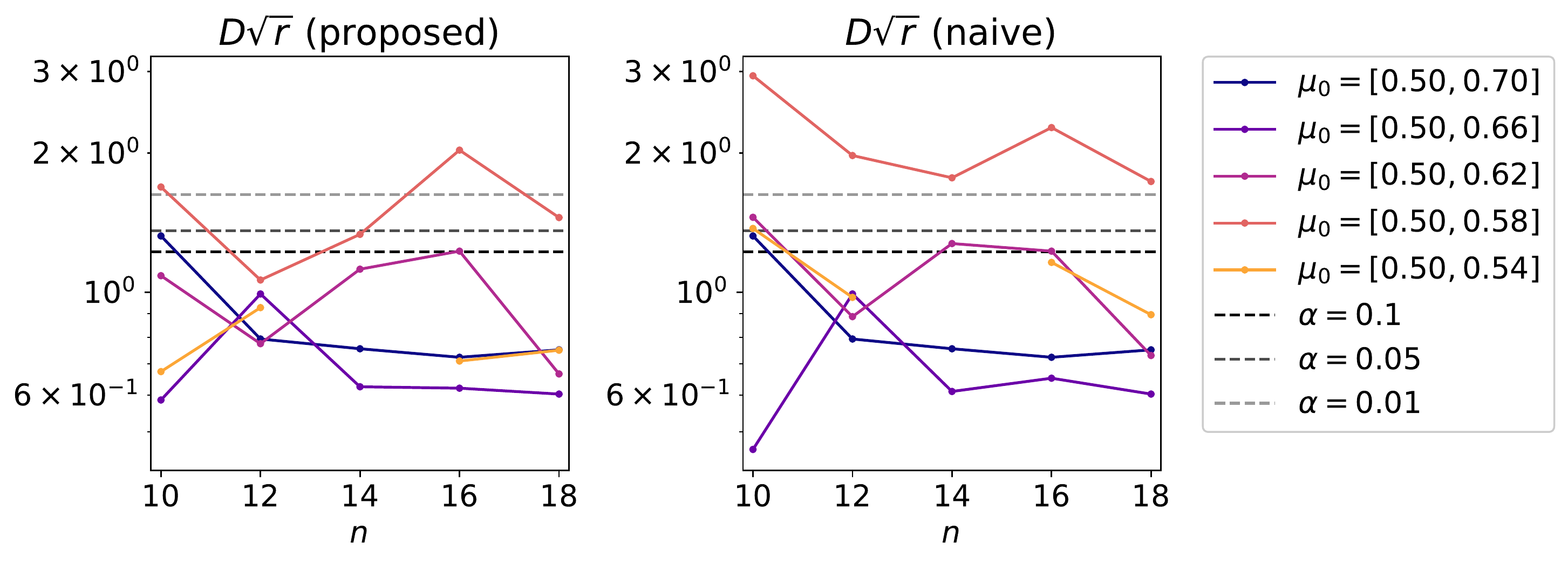}\vspace{-3mm}
  \caption{Test statistics $D\sqrt{r}$ of the Kolmogorov-Smirnov test \cite{Conover1999} for the $p$-values of the proposed (left) and naive (right) \textbf{approximated} tests.}\vspace{3mm}
  \label{fig:ks_test_approx}
  \includegraphics[width=0.7\hsize]{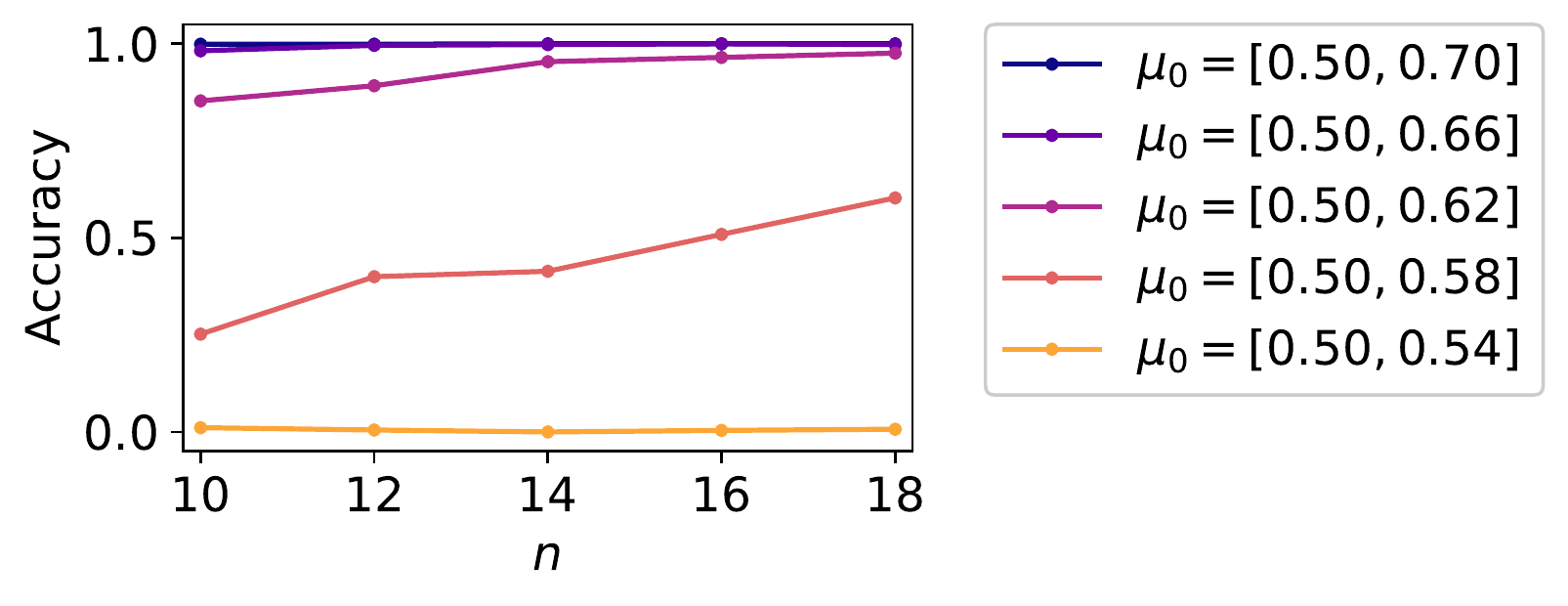}\vspace{-3mm}
  \caption{The ratio of the number of the null cases (i.e., $\hat{g} = g^{\mathrm{(N)}}$) for each setting of matrix size $(n, p)$ and mean vector $\bm{\mu}_0$, where $\hat{g}$ is output by the \textbf{approximated} clustering algorithm in Section \ref{sec:test_approx}. For the experiment, we used the setting of $n = p$.}
  \label{fig:accuracy_approx}
\end{figure}
\begin{figure}[t]
  \centering
  \includegraphics[width=0.99\hsize]{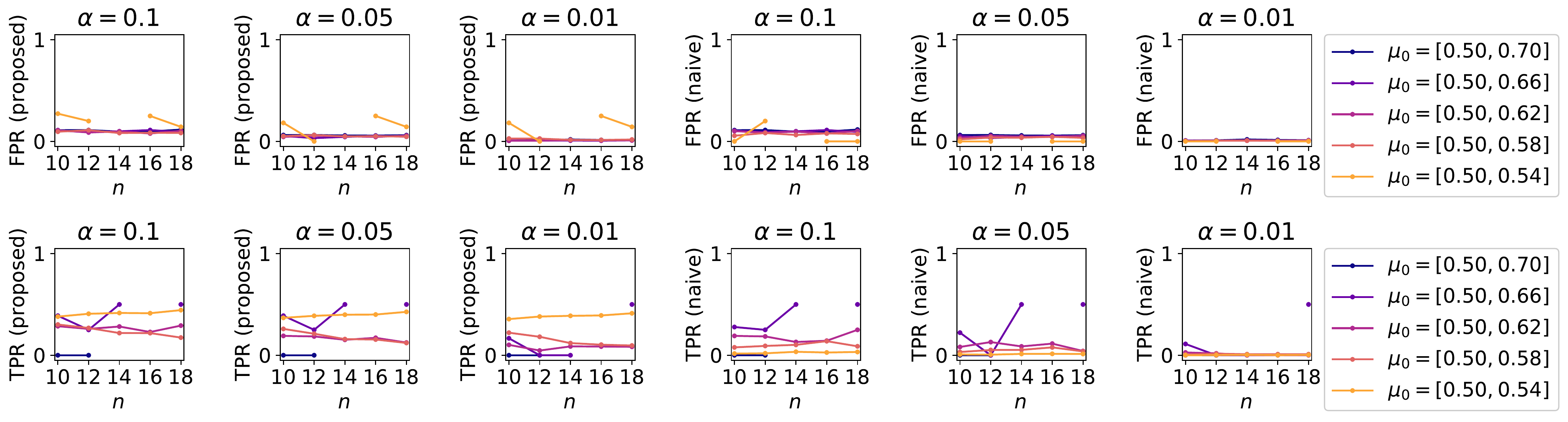}
  \caption{FPR and TPR in the realizable case with different significance rates (e.g., $\alpha = 0.1, 0.05$, and $0.01$), for the \textbf{approximated} version of the proposed (left) and naive (right) statistical tests.}\vspace{3mm}
  \label{fig:ratios_realizable_approx}
  \includegraphics[width=0.99\hsize]{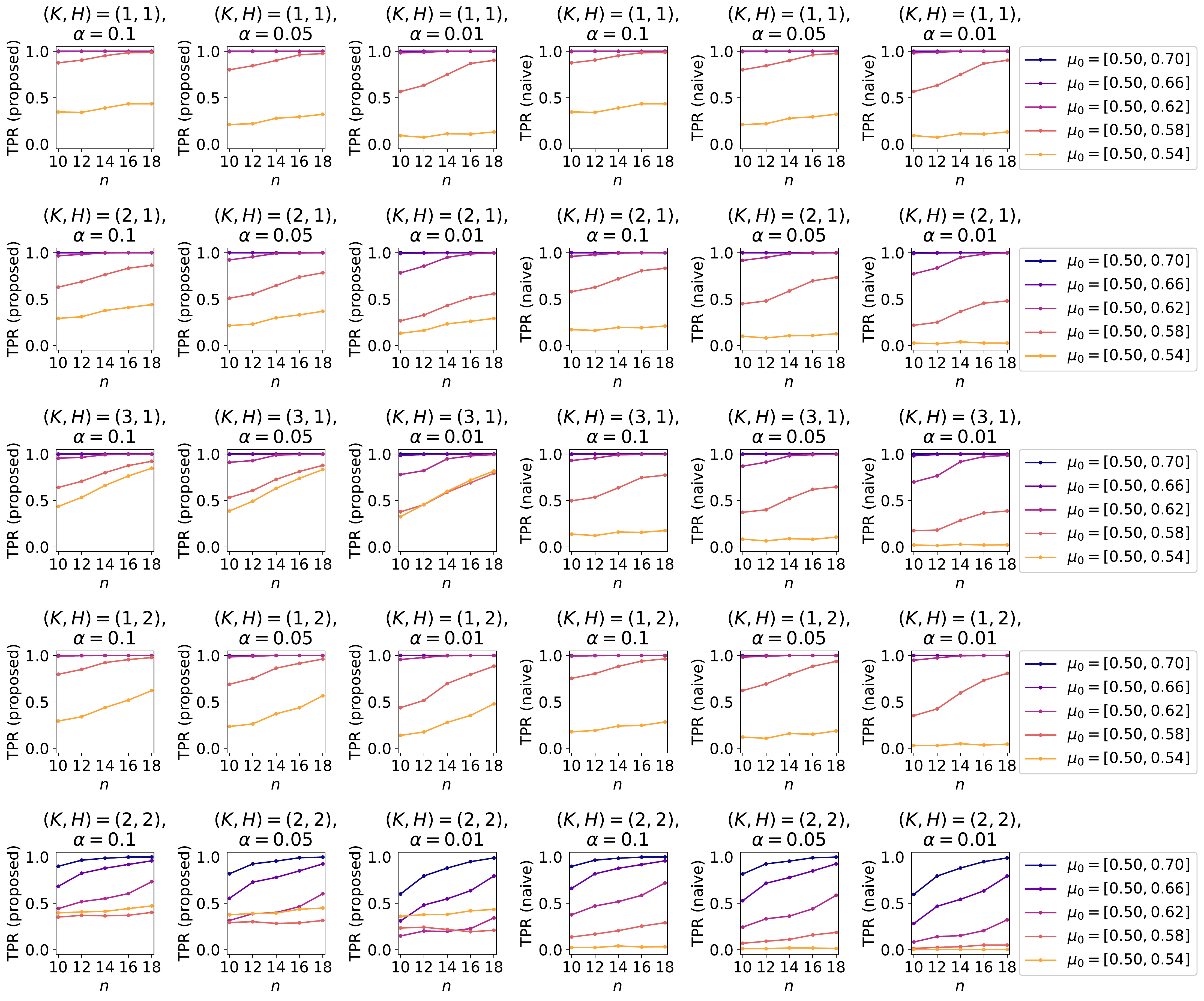}
  \caption{TPR in the unrealizable case (i.e., $K < K^{\mathrm{(N)}}$ or $H < H^{\mathrm{(N)}}$) with different significance rates (e.g., $\alpha = 0.1, 0.05$, and $0.01$), for the \textbf{approximated} version of the proposed (left) and naive (right) statistical tests.}
  \label{fig:ratios_unrealizable_approx}
\end{figure}


\subsection{Approximated test in the realizable case, $(K^{\mathrm{(N)}}, H^{\mathrm{(N)}}) = (3, 3), (4, 4), (5, 5)$}
\label{sec: exp_approx_KH}

To check the behavior of the $p$-values, FPR, and TPR of the proposed test in more difficult settings, where the clustering algorithm cannot successfully estimate the cluster memberships in most cases, we tried the following three settings of null cluster numbers: $(K^{\mathrm{(N)}}, H^{\mathrm{(N)}}) = (3, 3), (4, 4)$, and $(5, 5)$. These settings have more patterns of the possible block structures than those in the case of $(K^{\mathrm{(N)}}, H^{\mathrm{(N)}}) = (2, 2)$ in Section \ref{sec: exp_approx}. Hence, it becomes difficult for the approximated clustering algorithm (which stops at a fixed finite number of steps in the experiment) to output the null set of cluster memberships. 

We generated data matrices in the same way as that in Section \ref{sec: exp_approx}. Concerning the following conditions, we used the same setting as that of the realizable case in Section \ref{sec: exp_approx}: the set of matrix sizes $(n, p)$, the definition of the null cluster memberships $g^{\mathrm{(N)}}$, the standard deviation $\sigma_0$, and the cooling schedule of the SA algorithm. We tried the following three settings of the null number of blocks: $(3, 3)$, $(4, 4)$, and $(5, 5)$; subsequently, for each setting, we defined the mean vector $\bm{\mu}_0$ as follows: 
\begin{align}
\bm{\mu}_0^{(l)} &= \left( 1 - \frac{l - 1}{5} \right) \left[ \mathrm{vec} \left( \begin{bmatrix}
0.6 & 0.55 & 0.7 \\
0.4 & 0.6 & 0.5 \\
0.65 & 0.5 & 0.6 \\
\end{bmatrix} \right) - 0.5 \right] + 0.5, \nonumber \\
l &= 1, \dots, 5, 
\end{align}
for $(K, H) = (3, 3)$, 
\begin{align}
\bm{\mu}_0^{(l)} &= \left( 1 - \frac{l - 1}{5} \right) \left[ \mathrm{vec} \left( \begin{bmatrix}
0.6 & 0.55 & 0.7 & 0.5 \\
0.4 & 0.6 & 0.5 & 0.7 \\
0.65 & 0.5 & 0.6 & 0.4 \\
0.5 & 0.4 & 0.45 & 0.6 \\
\end{bmatrix} \right) - 0.5 \right] + 0.5,  \nonumber \\
l &= 1, \dots, 5, 
\end{align}
for $(K, H) = (4, 4)$, and 
\begin{align}
\bm{\mu}_0^{(l)} &= \left( 1 - \frac{l - 1}{5} \right) \left[ \mathrm{vec} \left( \begin{bmatrix}
0.6 & 0.55 & 0.7 & 0.5 & 0.65 \\
0.4 & 0.6 & 0.5 & 0.7 & 0.55 \\
0.65 & 0.5 & 0.6 & 0.4 & 0.45 \\
0.5 & 0.4 & 0.45 & 0.6 & 0.7 \\
0.7 & 0.65 & 0.55 & 0.45 & 0.6 \\
\end{bmatrix} \right) - 0.5 \right] + 0.5,  \nonumber \\
l &= 1, \dots, 5, 
\end{align}
for $(K, H) = (5, 5)$. We set the number of data vectors for each setting at $500$. 

We plotted  the test statistics $D\sqrt{r}$ of the Kolmogorov-Smirnov test \cite{Conover1999} for the $p$-values of the proposed and naive tests and the accuracy of the approximated clustering algorithm in Figures \ref{fig:ks_test_approx_KH} and \ref{fig:accuracy_approx_KH}, respectively. Figures \ref{fig:ratios_realizable_approx_K3H3}, \ref{fig:ratios_realizable_approx_K4H4}, and \ref{fig:ratios_realizable_approx_K5H5} show the FPR and TPR of the proposed and naive tests. These figures show that the TPRs of both the proposed and naive tests were higher than the case of $(K^{\mathrm{(N)}}, H^{\mathrm{(N)}}) = (2, 2)$; the TPR of the proposed test was higher than that of the naive one in these settings. 

\begin{figure}[t]
  \centering
  \includegraphics[width=0.95\hsize]{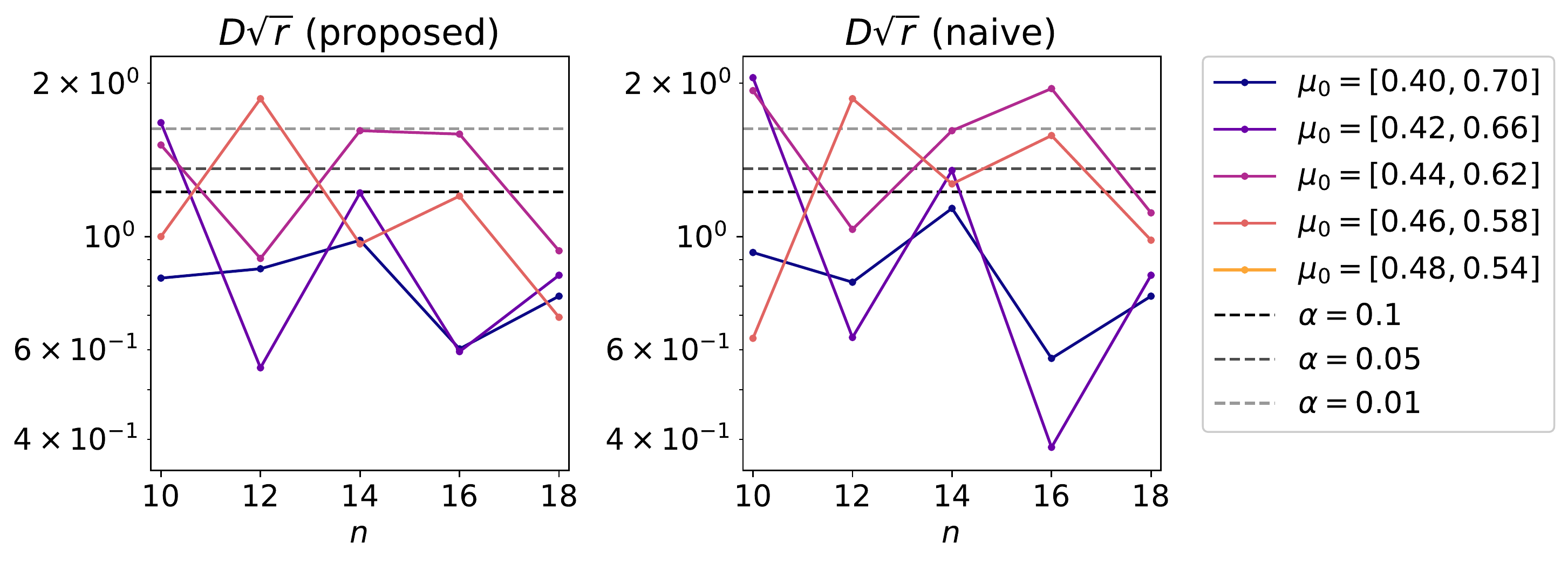}\\
  \includegraphics[width=0.95\hsize]{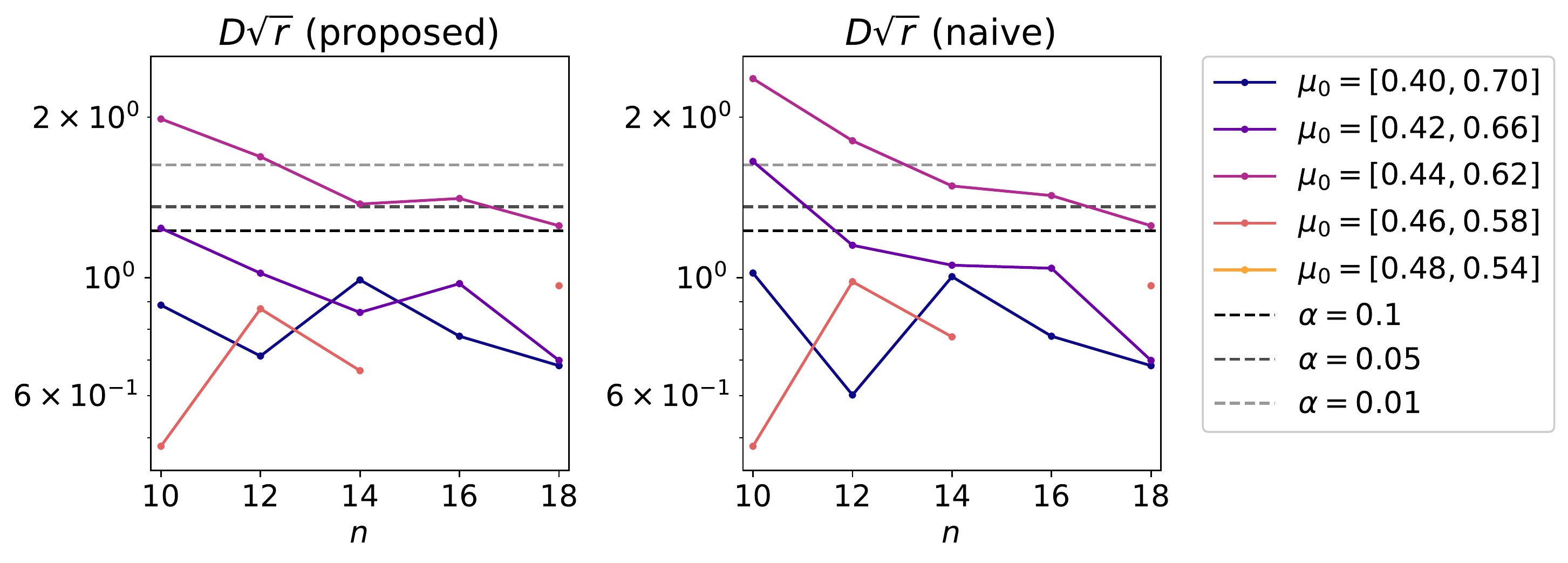}\\
  \includegraphics[width=0.95\hsize]{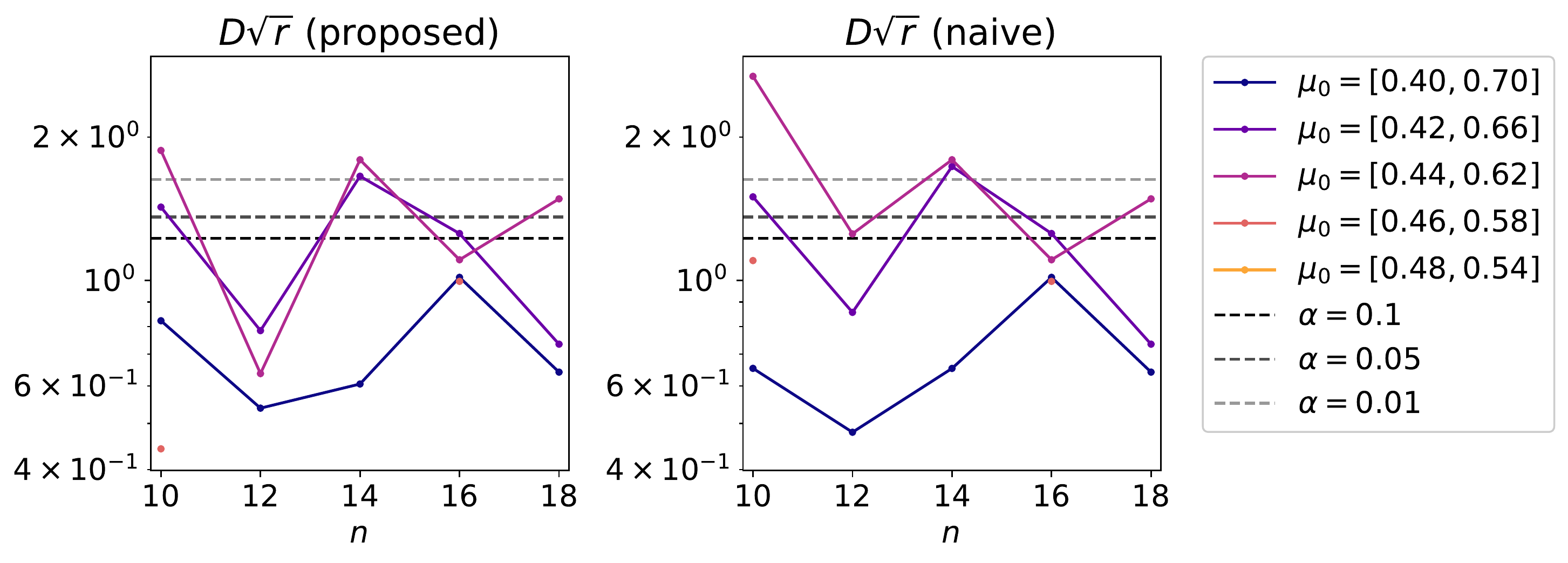}
  \caption{Test statistics $D\sqrt{r}$ of the Kolmogorov-Smirnov test \cite{Conover1999} for the $p$-values of the proposed (left) and naive (right) \textbf{approximated} tests, where $(K^{\mathrm{(N)}}, H^{\mathrm{(N)}}) = (3, 3)$ (top), $(4, 4)$ (middle), and $(5, 5)$ (bottom).}\vspace{3mm}
  \label{fig:ks_test_approx_KH}
\end{figure}
\begin{figure}[t]
  \includegraphics[width=0.7\hsize]{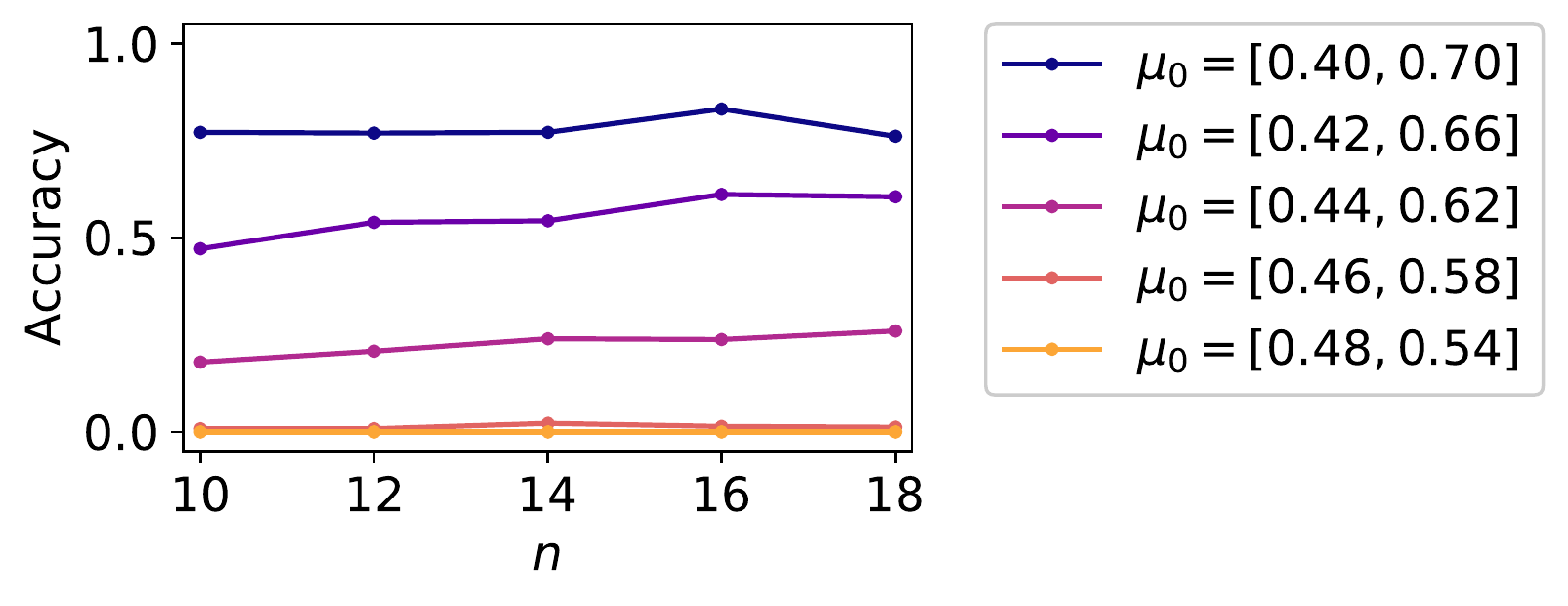}\\
  \includegraphics[width=0.7\hsize]{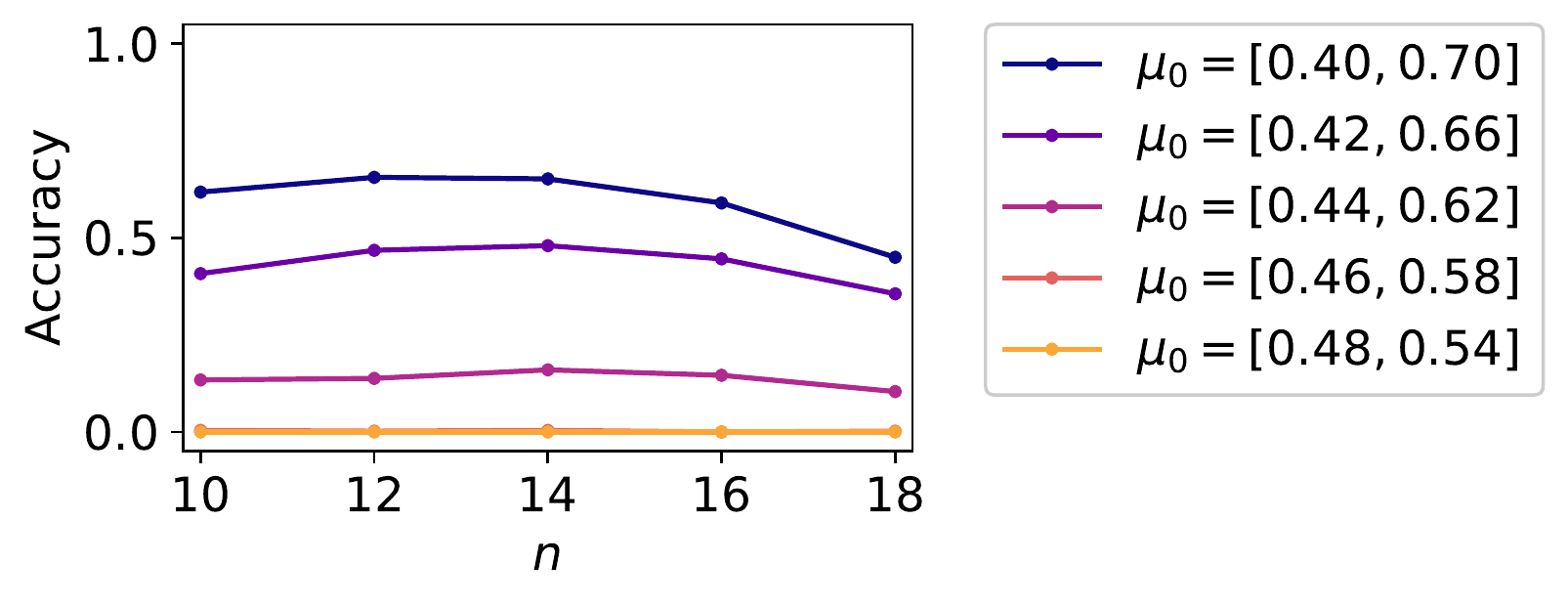}\\
  \includegraphics[width=0.7\hsize]{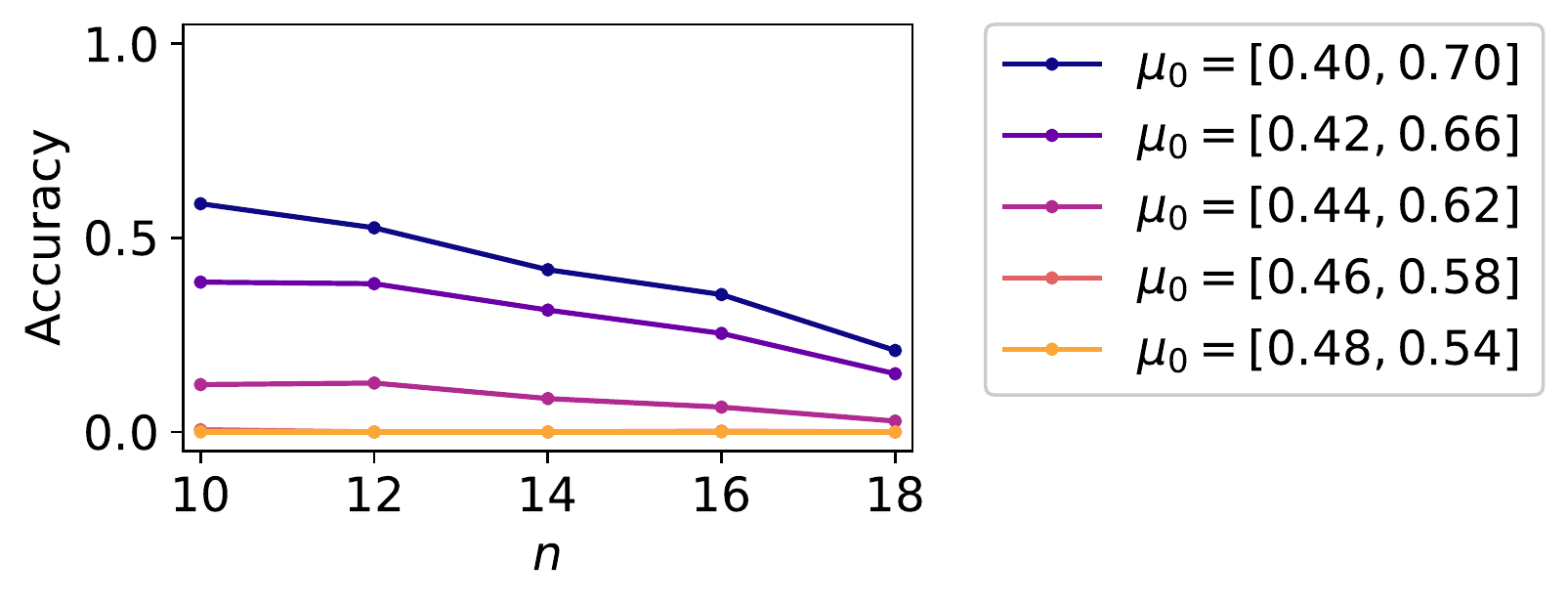}
  \caption{The ratio of the number of the null cases (i.e., $\hat{g} = g^{\mathrm{(N)}}$), for each setting of matrix size $(n, p)$ and mean vector $\bm{\mu}_0$, where $\hat{g}$ is output by the \textbf{approximated} clustering algorithm in Section \ref{sec:test_approx}; $(K^{\mathrm{(N)}}, H^{\mathrm{(N)}}) = (3, 3)$ (top), $(4, 4)$ (middle), and $(5, 5)$ (bottom). For experiment, we used the setting of $n = p$.}
  \label{fig:accuracy_approx_KH}
\end{figure}
\begin{figure}[t]
  \centering
  \includegraphics[width=0.99\hsize]{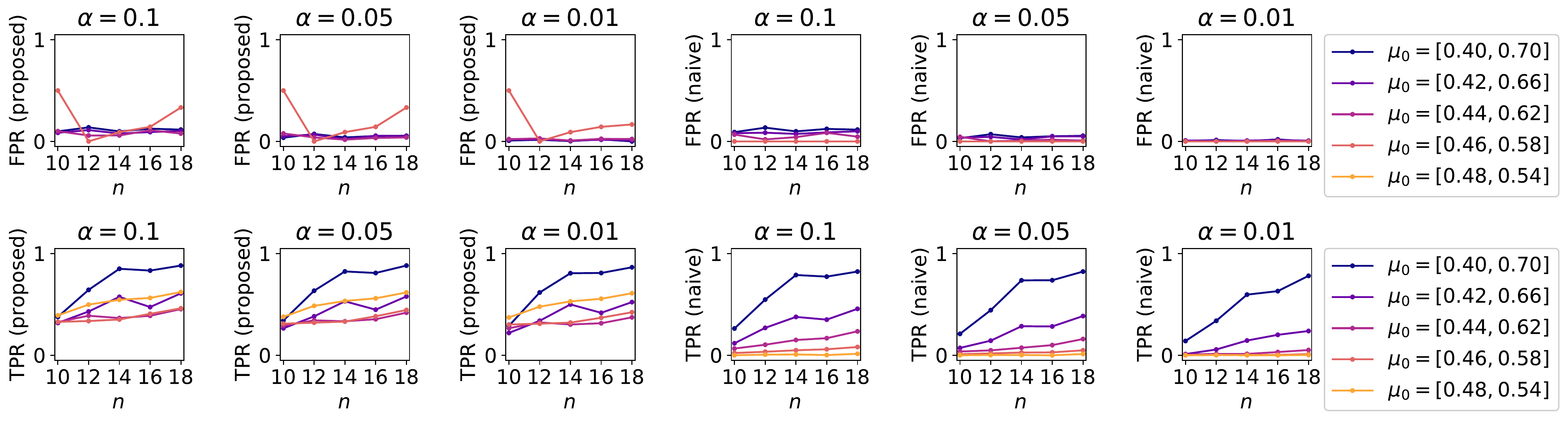}
  \caption{FPR and TPR in the realizable case with different significance rates (e.g., $\alpha = 0.1, 0.05$, and $0.01$), for the \textbf{approximated} version of the proposed (left) and naive (right) statistical tests, where $(K^{\mathrm{(N)}}, H^{\mathrm{(N)}}) = \bm{(3, 3)}$.}\vspace{3mm}
  \label{fig:ratios_realizable_approx_K3H3}
  \includegraphics[width=0.99\hsize]{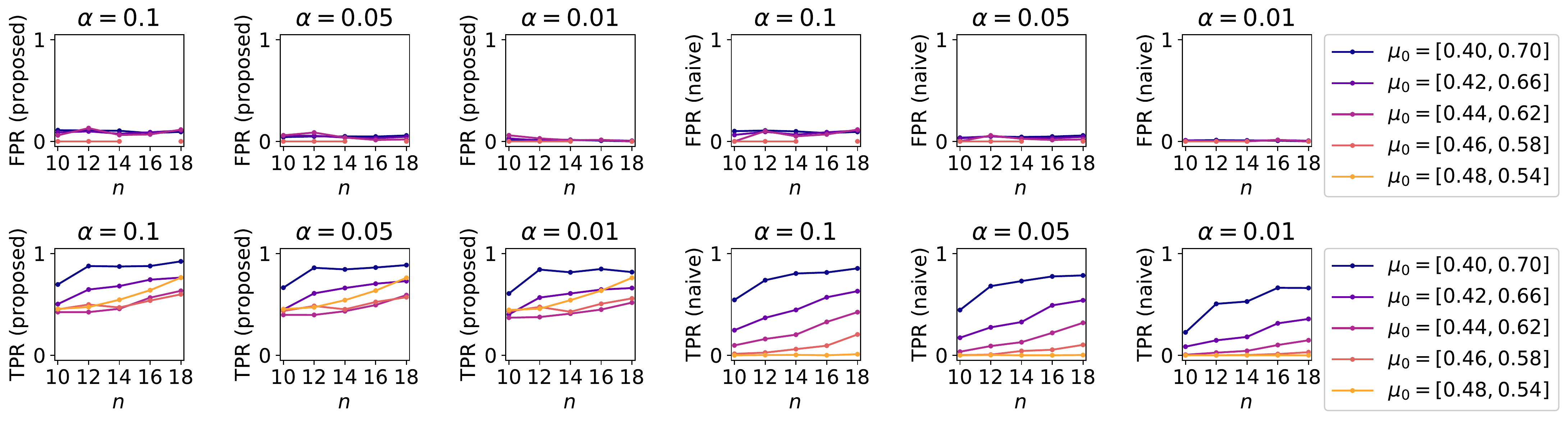}
  \caption{FPR and TPR in the realizable case with different significance rates (e.g., $\alpha = 0.1, 0.05$, and $0.01$), for the \textbf{approximated} version of the proposed (left) and naive (right) statistical tests, where $(K^{\mathrm{(N)}}, H^{\mathrm{(N)}}) = \bm{(4, 4)}$.}\vspace{3mm}
  \label{fig:ratios_realizable_approx_K4H4}
  \includegraphics[width=0.99\hsize]{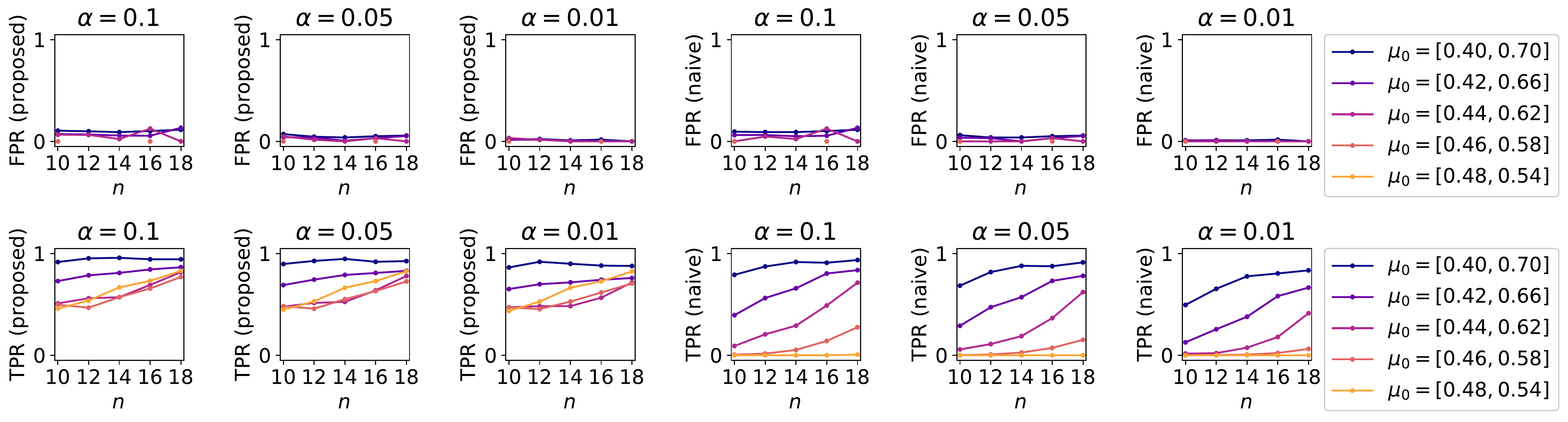}
  \caption{FPR and TPR in the realizable case with different significance rates (e.g., $\alpha = 0.1, 0.05$, and $0.01$), for the \textbf{approximated} version of the proposed (left) and naive (right) statistical tests, where $(K^{\mathrm{(N)}}, H^{\mathrm{(N)}}) = \bm{(5, 5)}$.}
  \label{fig:ratios_realizable_approx_K5H5}
\end{figure}


\section{Discussion}
\label{sec:discussion}

In this section, we discuss the following three points about the proposed statistical test: its power, the trade-off between computational efficiency and accuracy, and the extension of the finding to more generalized cases. 

First, as also pointed out in a study \cite{Loftus2015}, the null distribution of the test statistic of the proposed test is given by the conditioning on $\bm{z}$ and $\bm{u}$, besides the selected set of cluster memberships $\hat{g}$, which leads to a reduction in the test power \cite{Fithian2014}. For now, we do not have any way of removing these unnecessary parameters, owing to the problem setting of an LBM. In a one-way clustering problem, where there are $n$ data vectors with $p$ dimensions, we can at least approximate the distributions of $\bm{z}$ and $\bm{u}$ based on their histograms; however, in the LBM setting, there is only a single observed matrix with the size of $n \times p$. Solving this problem is beyond the scope of this paper; future studies should focus on constructing a more powerful selective test on a bicluster structure by using an additional technique such as a bootstrap method \cite{Terada2017, Tibshirani2018}. 

Next, we have proposed both exact and approximated tests to cope with the combinatorial explosion of the possible block memberships. 
The null distribution (\ref{eq:defT}) of the proposed test statistic is based on the assumption that the estimated cluster memberships $\hat{g}$ is the global minimum solution of the squared residue, which is difficult to obtain in the first place. 
Although it is guaranteed that the solutions of the two SA algorithms \ref{algo:min_sr} and \ref{algo:interval} converge in probability to the globally optimal solutions of their corresponding problems, we cannot validate the outputs of these algorithms with a finite number of steps, which are used in practice. It would be more desirable to derive an \textit{exact} $p$-value of some other \textit{approximated} test. 
Stopping the SA algorithms in a constant number of steps would also affect the accuracy of the test; to find the optimal solutions, we should have checked all the patterns of possible block memberships, which increase with the observed matrix size and the number of blocks. However, if we increase the number of steps according to such a problem size, then computation of the SA algorithms will get intractable. Therefore, it would be another important direction to seek a more computationally efficient test, which mitigates this trade-off. 

Finally, the proposed test has enabled us to perform a valid statistical inference for a Gaussian LBM, where we assume that each element of an observed matrix independently follows a Gaussian distribution, given a block structure. This Gaussian assumption is crucial for deriving the exact $p$-value in the selective inference framework, as in \cite{Loftus2015}. However, in many practical datasets, including the ``MovieLens'' dataset of movie ratings \cite{Harper2015} and the dataset of document-word relationships in NeurIPS conference papers \cite{Perrone2017}, the elements of the observed matrix take discrete values, where the proposed test cannot be employed. 
So far, there has been no selective test that can be directly applied to binary data vectors from a Bernoulli distribution. To address this problem, a randomized model selection method \cite{Tian2018} has been proposed to construct an asymptotically valid selective test on binary data by adding a random noise to the statistic used for a selection event. By using such a technique, future studies should generalize the proposed test for non-Gaussian cases.


\section{Conclusion}
\label{sec:conclusion}

We developed a new selective inference method on the row and column cluster memberships of a latent block model given by a clustering algorithm based on the squared residue minimization. By considering the selective bias, which is caused by the fact that the hypothetical block structure is estimated based on a given data matrix, we constructed a valid test based on a truncated chi distribution. Since such an exact test required us to obtain the global optimal solutions of two combinatorial optimization problems, we also constructed an approximated test based on simulated annealing algorithms. Experimental results showed that the proposed exact and approximated tests worked successfully, compared to the naive tests that did not take the selective bias into account. 


\section*{Acknowledgments}

TS was partially supported by JSPS KAKENHI (18K19793, 18H03201, and 20H00576), Japan Digital Design, Fujitsu Laboratories Ltd., and JST CREST. We would like to thank Editage (\url{www.editage.com}) for English language editing. 


\clearpage
\begin{appendices}
\section{Proof that $E^{(g)} - E^{(g')} \neq O$ for $g, g' \in \mathcal{G}_{KH}$, $g \neq g'$}
\label{sec:ap_nonzero_E}

\begin{proof}
We prove that $E^{(g)} - E^{(g')} \neq O$ by contradiction. Let $g, g' \in \mathcal{G}_{KH}$ be two sets of cluster memberships, both of which have $K \times H$ blocks or less and which satisfy $g \neq g'$. Specifically, we denote the exact number of blocks of $g$ as $(K^{(g)}, H^{(g)})$. Assume that $E^{(g)} - E^{(g')} = O$ holds. Then, for all $\bm{x} \in \mathbb{R}^{np}$, we have $\bm{x}^{\top} E^{(g)} \bm{x} = \bm{x}^{\top} E^{(g')} \bm{x}$. In other words, from (\ref{eq:sr_def}), block structures $g$ and $g'$ yield the same squared residue $\sigma^2$ for any data matrix $A$. 

Let us consider a data matrix $A$ that has a block structure $g$, and all of the elements in the $(k, h)$th block are $(k - 1) H^{(g)} + h$, where $k = 1, \dots, K^{(g)}$ and $h = 1, \dots,  H^{(g)}$, as shown in Figure \ref{fig:nonzeroE}. The squared residue of such matrix $A$ and block structure $g$ is zero, and thus $\bm{x}^{\top} E^{(g)} \bm{x} = 0$ holds. However, in block structure $g'$ satisfying $g' \neq g$, there exists at least one block of matrix $A$ that contains two or more mutually different values, unless $g'$ is a refinement of $g$, which results in $\bm{x}^{\top} E^{(g')} \bm{x} > 0$. In case that $g'$ is a refinement of $g$, by considering an observed matrix $A$ with block structure $g'$ instead of $g$, we obtain $\bm{x}^{\top} E^{(g')} \bm{x} = 0$ and $\bm{x}^{\top} E^{(g)} \bm{x} > 0$ from the similar discussion. This contradicts the assumption that $\bm{x}^{\top} E^{(g)} \bm{x} = \bm{x}^{\top} E^{(g')} \bm{x}$ for all $\bm{x} \in \mathbb{R}^{np}$. 
\begin{figure}[t]
  \centering
  \includegraphics[width=0.5\hsize]{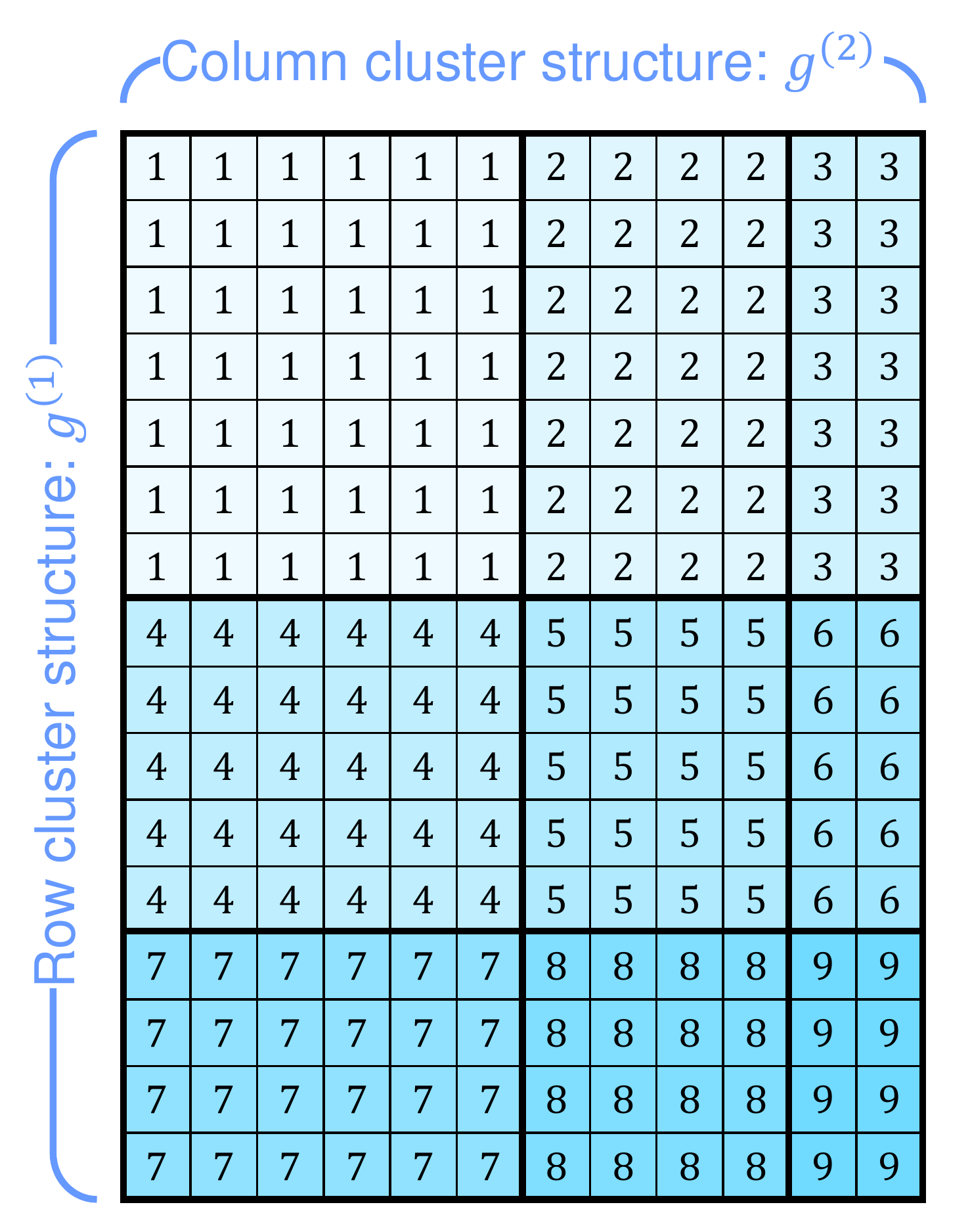}
  \caption{A data matrix $A$ whose squared residue $\sigma^2$ is zero with block structure $g$.}\vspace{5mm}
  \label{fig:nonzeroE}
\end{figure}

\end{proof}


\section{Proof that $\mathrm{rank} (E^{(\hat{g})}) = np-KH$}
\label{sec:ap_rankE}

\begin{proof}
For any cluster memberships $\hat{g}$, by simultaneously switching rows and columns with the same indices, matrix $E^{(\hat{g})}$ can be transformed into matrix $\tilde{E}^{(\hat{g})}$, which is given by
\begin{eqnarray}
\tilde{E}^{(\hat{g})} = \begin{bmatrix}
X^{(1)} & O & \cdots & \cdots & O \\
O & \ddots & \ddots & \cdots & \vdots \\
\vdots & \ddots & X^{[H (k - 1) + h]} & \ddots & \vdots \\
\vdots & \cdots & \ddots & \ddots & O \\
O & \cdots & \cdots & O & X^{(K H)} \\
\end{bmatrix}, 
\end{eqnarray}
where 
\begin{align}
\label{eq:tildeEij}
&X^{[H (k - 1) + h]} \equiv \left( X^{[H (k - 1) + h]}_{ij} \right)_{1 \leq i \leq |I_{k}| |J_{h}|, 1 \leq j \leq |I_{k}| |J_{h}|}, \nonumber \\
X^{[H (k - 1) + h]}_{ij} = \begin{cases}
1 - \frac{1}{|I_k| |J_h|} & \mathrm{if}\ i = j, \\
-\frac{1}{|I_k| |J_h|} & \mathrm{otherwise}, 
\end{cases} \nonumber \\
&i = 1, \cdots, |I_k| |J_h|, \ \ \ 
j = 1, \cdots, |I_k| |J_h|, 
\end{align}
for all $(k, h)$. 

Let $\tilde{\bm{e}}^{[H (k - 1) + h]}_i$ be the $i$th column of the $[H (k - 1) + h]$th row block of matrix $\tilde{E}^{(\hat{g})}$. For $(k, h) \neq (k', h')$, vectors $\tilde{\bm{e}}^{[H (k - 1) + h]}_i$ and $\tilde{\bm{e}}^{[H (k' - 1) + h']}_j$ are linearly independent for an arbitrary set of $(i, j)$. 

From here, we show that within the same $[H (k - 1) + h]$th block, the maximum number of linearly independent columns is $|I_k| |J_h| - 1$. First, from (\ref{eq:tildeEij}), we have 
\begin{eqnarray}
\sum_{i=1}^{|I_k| |J_h|} \tilde{\bm{e}}^{[H (k - 1) + h]}_i = \bm{0}. \ \left( \because 1 - \frac{1}{|I_k| |J_h|} + (|I_k| |J_h| - 1) \left( -\frac{1}{|I_k| |J_h|} \right) = 0 \right)
\end{eqnarray}
Therefore, the maximum number of linearly independent columns is smaller than $|I_k| |J_h|$. Next, the columns of the indices of $i = 1, \cdots, |I_k| |J_h| - 1$ are linearly independent, since 
\begin{align}
&\sum_{i=1}^{|I_k| |J_h| - 1} c_i \tilde{\bm{e}}^{[H (k - 1) + h]}_i = \bm{0}. \nonumber \\
\iff& c_1 \left( 1 - \frac{1}{|I_k| |J_h|} \right) + \sum_{i \neq 1} c_i \left( -\frac{1}{|I_k| |J_h|} \right) = 0,\ \cdots, \nonumber \\
&c_{|I_k| |J_h| - 1} \left( 1 - \frac{1}{|I_k| |J_h|} \right) + \sum_{i \neq |I_k| |J_h| - 1} c_i \left( -\frac{1}{|I_k| |J_h|} \right) = 0. \nonumber \\
\iff& c_1 + \left( -\frac{1}{|I_k| |J_h|} \right) \sum_i c_i = 0,\ \cdots,\ c_{|I_k| |J_h| - 1} + \left( -\frac{1}{|I_k| |J_h|} \right) \sum_i c_i = 0. \nonumber \\
\iff& c_1 = c_2 = \cdots = c_{|I_k| |J_h| - 1}, \nonumber \\
&c_i + \left( -\frac{1}{|I_k| |J_h|} \right) \left( |I_k| |J_h| - 1 \right) c_i = 0, \mathrm{for\ all}\ i. \nonumber \\
\iff& c_1 = c_2 = \cdots = c_{|I_k| |J_h| - 1}, \ \ \ \frac{1}{|I_k| |J_h|} c_i = 0, \mathrm{for\ all}\ i. \nonumber \\
\iff& c_1 = c_2 = \cdots = c_{|I_k| |J_h| - 1} = 0. 
\end{align}

By combining the above results, the maximum number of linearly independent columns of matrix $\tilde{E}^{(\hat{g})}$ is $\sum_{k, h} (|I_k| |J_h| - 1) = np - KH$. Since the rank of matrix $E^{(\hat{g})}$ is equal to that of matrix $\tilde{E}^{(\hat{g})}$, we finally have $\mathrm{rank} (E^{(\hat{g})}) = np-KH$. 
\end{proof}


\section{Proof that the number of mutually different patterns of block structures with \textbf{exactly} $K \times H$ blocks is lower bounded by $K^{n - K} H^{p - H}$}
\label{sec:ap_patterns}

\begin{proof}
To derive a lower bound for the number of mutually different patterns of block structures, let us define a \textbf{subset} $\mathcal{G}^{(1)}_0$ of all the possible patterns of row cluster indexing as a set of all the row cluster membership vectors satisfying the following two conditions. 
\begin{itemize}
\item $n$ rows are clustered into \textbf{exactly} $K$ clusters. 
\item It can be equivalently represented in the unique form of Figure \ref{fig:combination} for some $\tilde{n} \in \{ K, \dots, n \}$. In other words, its first $(\tilde{n} - 1)$ elements contain $1, \dots, (K-1)$ in ascending order, where $\tilde{n}$ is the minimum row index of the $K$th cluster. 
\end{itemize}
For a fixed $\tilde{n}$, there are $\frac{(\tilde{n} - 2)!}{(\tilde{n} - K)! (K - 2)!}$ possible patterns of the first $(\tilde{n} - 1)$ elements of a cluster membership vector in $\mathcal{G}^{(1)}_0$. The last $(n - \tilde{n})$ elements are arbitrary (i.e., different indexing of these elements yields mutually \textbf{not} equivalent set of row cluster memberships), which have $K^{n - \tilde{n}}$ patterns. 
Therefore, there are $\sum_{\tilde{n} = K}^n \frac{(\tilde{n} - 2)!}{(\tilde{n} - K)! (K - 2)!} K^{n - \tilde{n}}$ patterns of mutually different sets of row cluster memberships. From the same discussion for column cluster memberships, we obtain a lower bound for the total number $\kappa$ of the patterns of mutually different block structures: 
\begin{align}
\kappa \geq \left[ \sum_{\tilde{n} = K}^n \frac{(\tilde{n} - 2)!}{(\tilde{n} - K)! (K - 2)!} K^{n - \tilde{n}} \right] \left[ \sum_{\tilde{p} = H}^p \frac{(\tilde{p} - 2)!}{(\tilde{p} - H)! (H - 2)!} H^{p - \tilde{p}} \right] \geq K^{n - K} H^{p - H}, 
\end{align}
which is in the exponential order of $n$ and $p$ for a fixed number of blocks $(K, H)$. 
\begin{figure}[t]
  \centering
  \includegraphics[width=0.85\hsize]{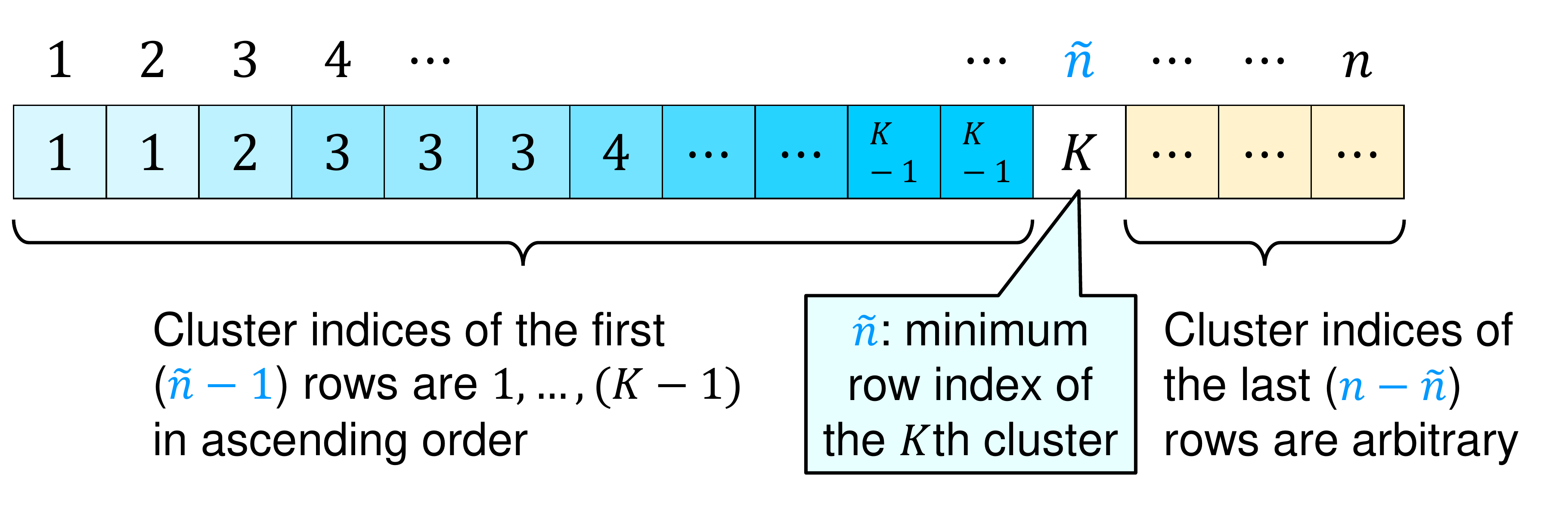}
  \caption{Unique representation of row cluster indexing where $n$ rows are clustered into \textbf{exactly} $K$ clusters. It must be noted that the set of cluster membership vectors $g^{(1)}$ that can be represented in this form is a \textbf{subset} of all the possible cluster membership vectors.}
  \label{fig:combination}
\end{figure}
\end{proof}


\section{Proof that $T_E$ and $(\bm{u}_E, \bm{z}_E)$ are mutually independent}
\label{sec:ap_indTuz}

\begin{proof}
We have assumed that $\bm{x} \sim N(\bm{\mu}_0, \sigma_0^2 I)$ and have defined that $\bm{r}_E \equiv E \bm{x}$, $T_E = \frac{\| \bm{r}_E \|_2}{\sigma_0}$, $\bm{u}_E \equiv \frac{1}{\| \bm{r}_E \|_2} \bm{r}_E$, $\bm{z}_E \equiv \bm{x} - \bm{r}_E$. Note that the following equations hold: 
\begin{align}
\label{eq:umu_zero}
\bm{u}_{E}^{\top} \bm{\mu}_0 &= \frac{1}{\| \bm{r}_E \|_2} \bm{r}_E^{\top} \bm{\mu}_0 = \frac{1}{\| \bm{r}_E \|_2} \bm{x}^{\top} E^{\top} \bm{\mu}_0 = 0. \\
\label{eq:uu_one}
\bm{u}_{E}^{\top} \bm{u}_{E} &= \frac{1}{\| \bm{r}_E \|_2^2} \bm{r}_E^{\top} \bm{r}_E = 1. 
\end{align}
To obtain the last equation, we used the assumption that $E \bm{\mu}_0 = \bm{0}$. 

Therefore, we have
\begin{align}
\label{eq:pxind}
&p(\bm{x}) = \frac{1}{\sqrt{(2\pi \sigma_0^2)^{np}}} \exp \left[ -\frac{1}{2 \sigma_0^2} \| \bm{x} - \bm{\mu}_0 \|_2^2 \right] \nonumber \\
=& \frac{1}{\sqrt{(2\pi \sigma_0^2)^{np}}} \exp \left[ -\frac{1}{2 \sigma_0^2} \| \bm{u}_{E} T_{E} \sigma_0 + \bm{z}_{E} - \bm{\mu}_0 \|_2^2 \right] \nonumber \\
=& \frac{1}{\sqrt{(2\pi \sigma_0^2)^{np}}} \exp \left\{ -\frac{1}{2 \sigma_0^2} \left[ \sigma_0^2 T_{E}^2 \bm{u}_{E}^{\top} \bm{u}_{E} + 2 \sigma_0 T_{E} \bm{u}_{E}^{\top} (\bm{z}_{E} - \bm{\mu}_0) + \| \bm{z}_{E} - \bm{\mu}_0 \|_2^2 \right] \right\} \nonumber \\
=& \frac{1}{\sqrt{(2\pi \sigma_0^2)^{np}}} \exp \left[ -\frac{1}{2} T_{E}^2 -\frac{1}{2 \sigma_0^2} \| \bm{z}_{E} - \bm{\mu}_0 \|_2^2 \right] (\because (\ref{eq:uz_zero}), (\ref{eq:umu_zero}), (\ref{eq:uu_one})) \nonumber \\
=& \frac{1}{\sqrt{(2\pi \sigma_0^2)^{np}}} \exp \left( -\frac{1}{2} T_{E}^2 -\frac{1}{2 \sigma_0^2} \| \bm{z}_E \|_2^2 \right) \exp \left[ -\frac{1}{2 \sigma_0^2} (-2 \bm{z}_E^{\top} \bm{\mu}_0 + \| \bm{\mu}_0 \|_2^2) \right]. 
\end{align}

Next, we use the following Proposition 2.1 in \cite{Shao2003}: let $p_{\theta}$ be a probability density function of an exponential family distribution with parameter $\theta$, which is given by $p_{\theta} (\bm{x}) = h(\bm{x}) \exp \{ [\bm{\eta} (\theta)]^{\top} T(\bm{x}) - \xi (\theta) \}$. Then, $T$ is complete and sufficient for $\bm{\eta}$. From this proposition and (\ref{eq:pxind}), $\bm{z}_E$ is complete and sufficient for $\bm{\mu}_0$. 

We also show that $(T_E, \bm{u}_E)$ are ancillary for $\bm{\mu}_0$. To prove this, we first show that $T_E$ and $\bm{u}_E$ are mutually independent. Let $\bm{y} \equiv \tilde{D} V \bm{x} \in \mathbb{R}^{np-KH}$, where $V$ and $\tilde{D}$ are the matrices defined in (\ref{eq:EandD}) and (\ref{eq:TE}), respectively. From (\ref{eq:DDVnorm}) and the fact that $\tilde{D} V \bm{\mu}_0 = \bm{0}$ ($\because \| \tilde{D} V \bm{\mu}_0 \|_2^2 = \bm{\mu}_0^{\top} E \bm{\mu}_0 = 0$), we have $\bm{y} \sim N(\bm{0}, \sigma_0^2 I_{np-KH})$. Therefore, we have
\begin{align}
p(\bm{y}) = \frac{1}{\sqrt{(2\pi \sigma_0^2)^{np-KH}}} \exp \left( -\frac{1}{2 \sigma_0^2} \| \bm{y} \|_2^2 \right). 
\end{align}
From Proposition 2.1 in \cite{Shao2003} and the fact that $T_E^2 = \| \bm{y} \|_2^2 / \sigma_0^2$, $T_E^2$ is complete and sufficient for $\bm{\mu}_0$. Since there is a one-to-one correspondence between $T_E$ and $T_E^2$, $T_E$ is also complete and sufficient for $\bm{\mu}_0$. Let $\tilde{\bm{u}}_E \equiv \bm{y} / \| \bm{y} \|_2$. Since $\tilde{\bm{u}}_E$ follows a uniform distribution on the surface of unit sphere and $\bm{u}_E = V^{\top} \tilde{D}^{\top} \tilde{\bm{u}}_E$, $\bm{u}_E$ is ancillary for $\bm{\mu}_0$. By combining these results, $T_E$ and $\bm{u}_E$ are mutually independent from Basu's theorem \cite{Basu1955}. Therefore, we have $p(T_E, \bm{u}_E) = p(T_E) p(\bm{u}_E)$, where $p(\cdot)$ denotes a probability density function. From the above discussion about $p(\bm{u}_E)$ and the fact that $T_E \sim \chi_{(np-KH)}$ ($\because$ (\ref{eq:T_E_chi})), $(T_E, \bm{u}_E)$ are ancillary for $\bm{\mu}_0$. 

Based on the above results, $\bm{z}_E$ and $(T_E, \bm{u}_E)$ are independent from Basu's theorem \cite{Basu1955}. Therefore, we have
\begin{align}
\label{eq:p_T_E_given_u_E_z_E}
p(T_E, \bm{u}_E, \bm{z}_E) = p(\bm{z}_E | T_E, \bm{u}_E) p(T_E, \bm{u}_E) = p(\bm{z}_E) p(T_E) p(\bm{u}_E). 
\end{align}
From (\ref{eq:p_T_E_given_u_E_z_E}) and the fact that the ranges of $\bm{z}_E$, $T_E$, and $\bm{u}_E$ do not depend on each other, we also have $p(\bm{u}_E, \bm{z}_E) = p(\bm{u}_E) p(\bm{z}_E)$ and thus
\begin{align}
\label{eq:p_T_E_given_u_z}
p(T_E | \bm{u}_E, \bm{z}_E) &= \frac{p(T_E, \bm{u}_E, \bm{z}_E)}{p(\bm{u}_E, \bm{z}_E)} 
= \frac{p(\bm{z}_E) p(T_E) p(\bm{u}_E)}{p(\bm{u}_E, \bm{z}_E)} 
= \frac{p(\bm{z}_E) p(T_E) p(\bm{u}_E)}{p(\bm{u}_E) p(\bm{z}_E)} \nonumber \\
&= p(T_E), 
\end{align}
which concludes the proof. 
\end{proof}


\section{Sensitivity analysis with respect to the cooling schedule of simulated annealing}
\label{sec:sensitivity_cooling}

We conducted sensitivity analysis of the approximated version of the proposed test with respect to the cooling schedule of SA in the realizable case (i.e., $(K, H) = (K^{\mathrm{(N)}}, H^{\mathrm{(N)}})$). Aside from the settings of the mean vector $\bm{\mu}_0$ and the cooling schedule of SA, we employed the same settings as in Section \ref{sec: exp_approx}. We tried the following five cooling schedules: $T_t = 10 \times r^t$, for all $t \geq 0$, where $r = 0.99, 0.97, 0.95, 0.93, 0.91$. As for the mean vector, we used the following setting: 
\begin{eqnarray}
\bm{\mu}_0 = 0.6 \left[ \mathrm{vec} \left( \begin{bmatrix}
0.7 & 0.55 \\
0.5 & 0.6 \\
\end{bmatrix} \right) - 0.5 \right] + 0.5.
\end{eqnarray}

Figures \ref{fig:pvalues_p_approx_cooling} and \ref{fig:pvalues_n_approx_cooling}, respectively, show the histograms of the $p$-values of the proposed and naive approximated tests for different matrix sizes and cooling schedules $r$. We also plotted (i) the test statistics $D\sqrt{r}$ of the Kolmogorov-Smirnov test \cite{Conover1999}, for the $p$-values of the proposed and naive tests, and (ii) the accuracy of the approximated clustering algorithm in Figures \ref{fig:ks_test_approx_cooling} and \ref{fig:accuracy_approx_cooling}, respectively. Figure \ref{fig:ratios_realizable_approx_cooling} shows the FPR and TPR. From Figure \ref{fig:accuracy_approx_cooling}, we see that the accuracy of the SA algorithm got lower with the smaller value of $r$. As shown in Figure \ref{fig:ratios_realizable_approx_cooling}, the FPR was low in all the settings, while the TPR of both the proposed and naive tests got lower with the larger value of $r$. A possible reason for this result is that with small $r$, the SA algorithm tends to output ``bad'' solutions (i.e., solutions that yield large squared residues) and thus both the proposed and naive tests can easily reject the null hypothesis.

\begin{figure}[p]
  \centering
  \includegraphics[width=0.9\hsize]{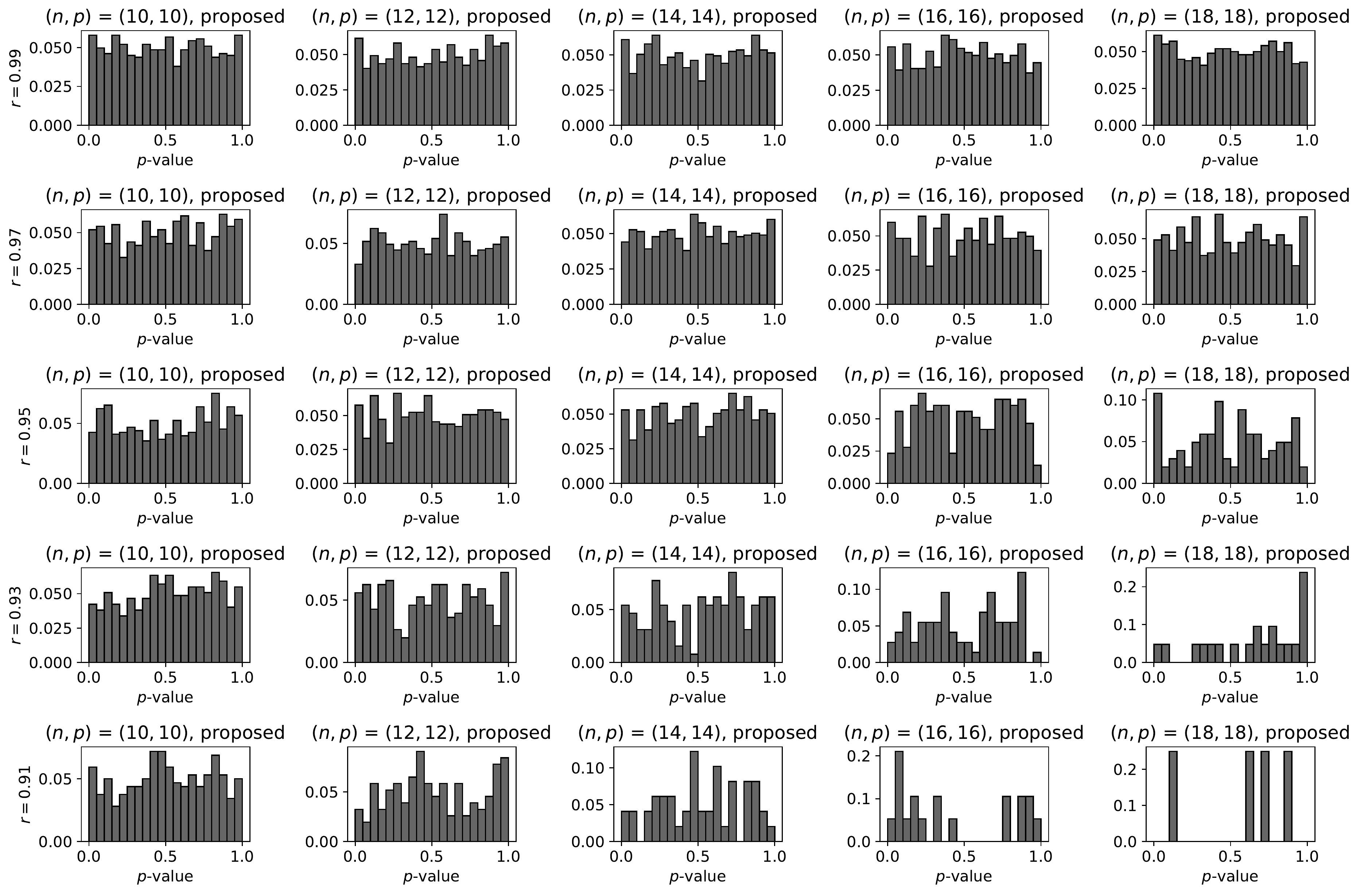}
  \caption{Histograms of $p$-values in the null case (i.e., $\hat{g} = g^{\mathrm{(N)}}$) for different matrix sizes and cooling schedules $r$, which was computed by the \textbf{approximated} version of the \textbf{proposed} test.}\vspace{3mm}
  \label{fig:pvalues_p_approx_cooling}
  \includegraphics[width=0.9\hsize]{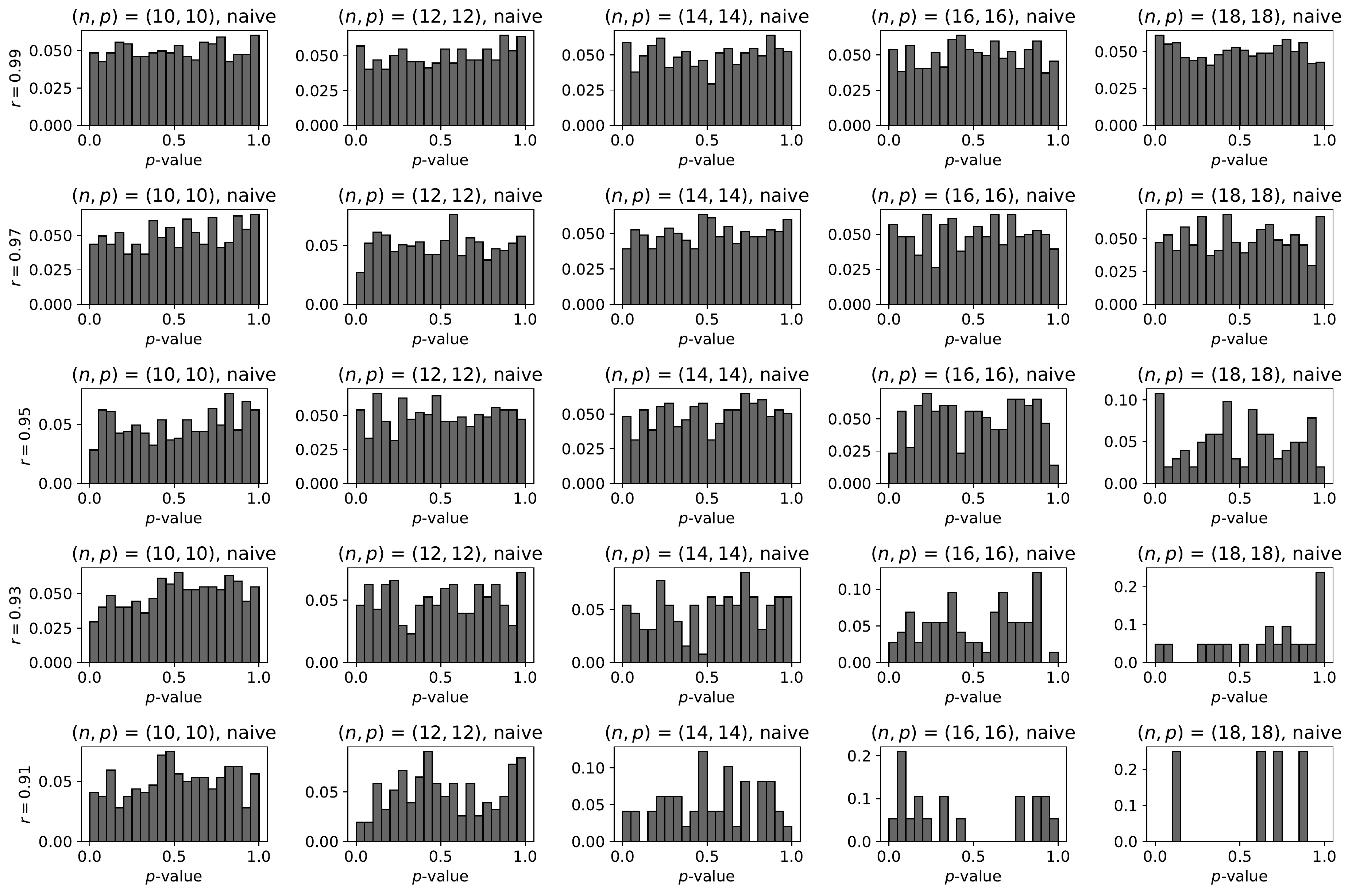}
  \caption{Histograms of $p$-values in the null case (i.e., $\hat{g} = g^{\mathrm{(N)}}$) for different matrix sizes and cooling schedules $r$, which was computed by the \textbf{approximated} version of the \textbf{naive} test (\ref{eq:pval_naive}).}
  \label{fig:pvalues_n_approx_cooling}
\end{figure}
\begin{figure}[t]
  \centering
  \includegraphics[width=0.95\hsize]{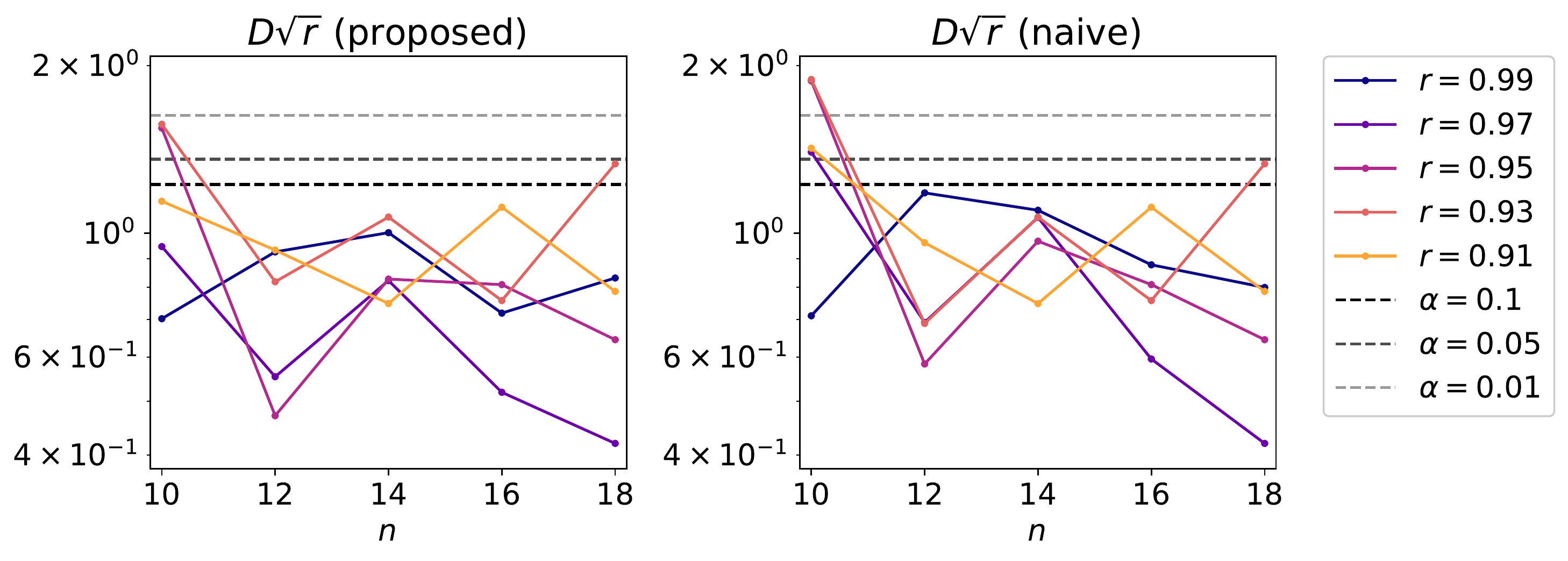}\vspace{-3mm}
  \caption{Test statistics $D\sqrt{r}$ of the Kolmogorov-Smirnov test \cite{Conover1999} for the $p$-values of the proposed (left) and naive (right) \textbf{approximated} tests under the different cooling schedule settings $r$.}\vspace{3mm}
  \label{fig:ks_test_approx_cooling}
  \includegraphics[width=0.7\hsize]{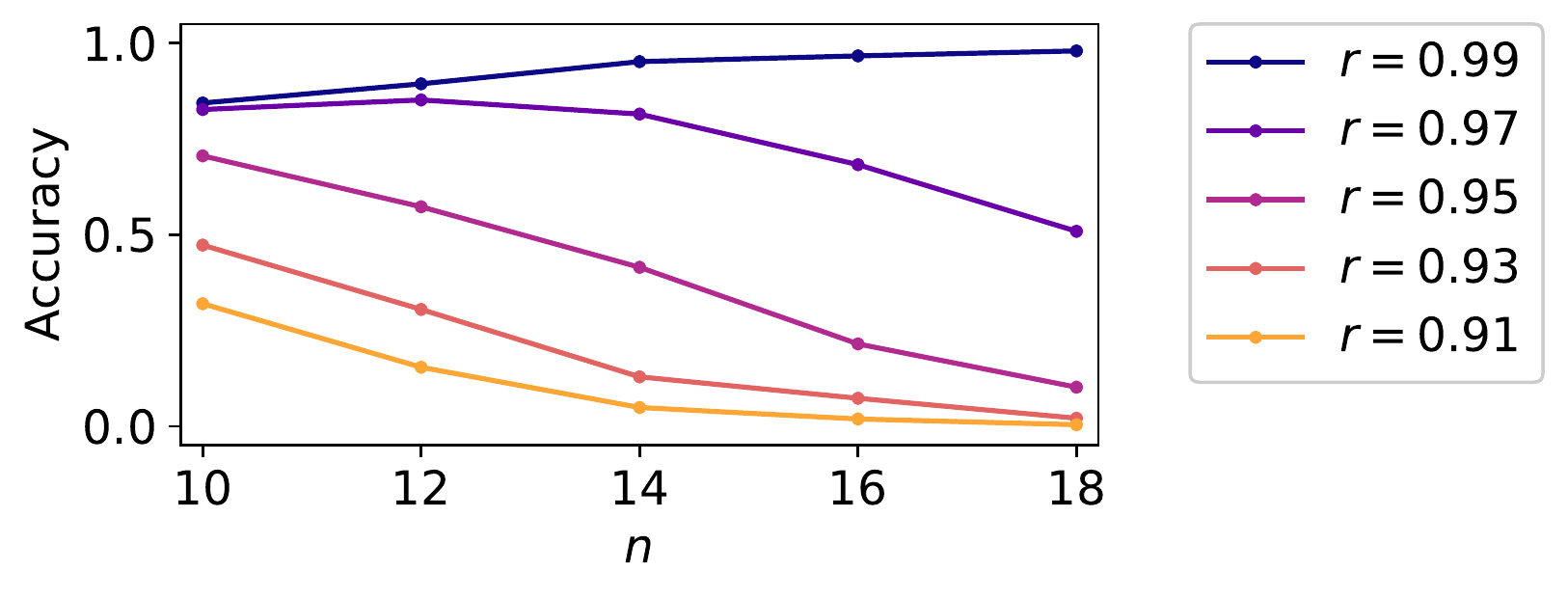}\vspace{-3mm}
  \caption{The ratio of the number of the null cases (i.e., $\hat{g} = g^{\mathrm{(N)}}$) for each setting of matrix size $(n, p)$ and cooling schedule $r$, where $\hat{g}$ is output by the \textbf{approximated} clustering algorithm in Section \ref{sec:test_approx}. For the experiment, we used the setting of $n = p$.}
  \label{fig:accuracy_approx_cooling}
\end{figure}
\begin{figure}[t]
  \centering
  \includegraphics[width=0.99\hsize]{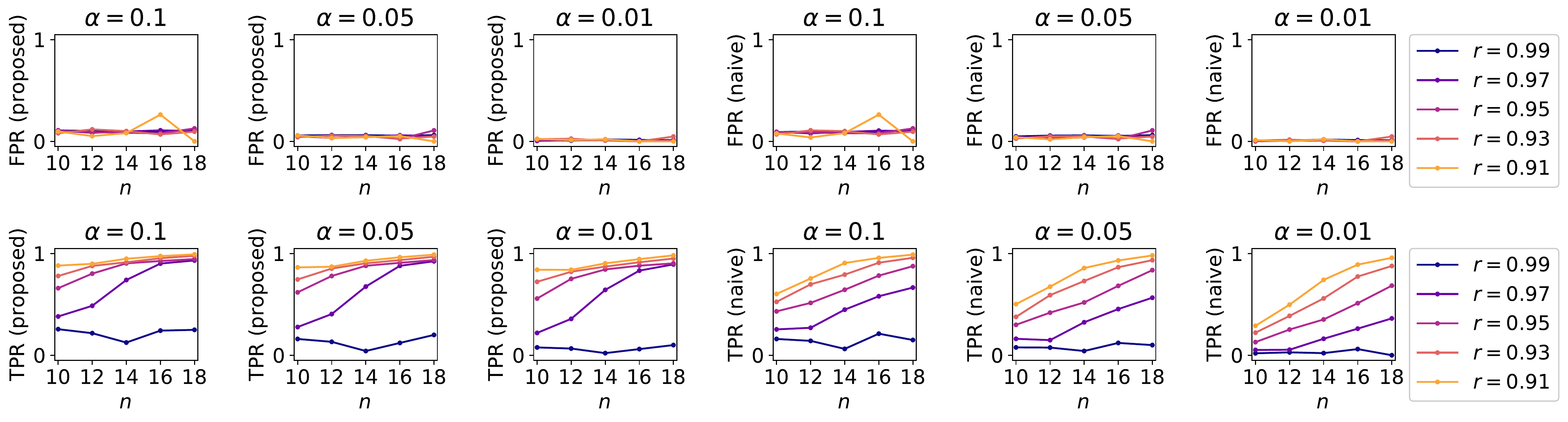}
  \caption{FPR and TPR in the realizable case with different significance rates (e.g., $\alpha = 0.1, 0.05$, and $0.01$), for the \textbf{approximated} version of the proposed (left) and naive (right) statistical tests under the different cooling schedule settings $r$.}\vspace{3mm}
  \label{fig:ratios_realizable_approx_cooling}
\end{figure}


\section{Application of computationally efficient biclustering algorithm for estimating the cluster memberships}
\label{sec:kmeans_biclustering}

The proposed approximated test based on the SA algorithm is guaranteed to converge in probability to the global minimum solution in terms of the squared residue under the conditions given in Section \ref{sec:test_approx}. However, this SA algorithm requires much computation time before convergence. As another option, we can use some computationally efficient biclustering algorithm for estimating the cluster memberships $\hat{g}$. 

There have been proposed various fast biclustering algorithms \cite{Chi2017, Lee2010, Tan2014}. Among these algorithms, we applied the biclustering algorithm that has been proposed by Tan and Witten \cite{Tan2014}, which is aim to minimize the loss function $\mathcal{L} (g, B; \bm{x})$ in (\ref{eq:log_lh}). In this algorithm, to find the local optimal solution, we iteratively estimate the block-wise mean $B$ and row and column cluster memberships, $g^{(1)}$ and $g^{(2)}$, respectively. The specific algorithm of this method is given in Algorithm \ref{algo:tan2014}. There is no theoretical guarantee that this algorithm converges to the global optimal solution in terms of the squared residue in any way, however, under the assumption that it yields a good approximation of the global optimal solution, we can use this algorithm instead of the proposed SA algorithm in Section \ref{sec:test_approx} for estimating $\hat{g}$. 

\begin{algorithm}[t]
\caption{Computationally efficient biclustering algorithm that has been proposed by Tan and Witten \cite{Tan2014}. }         
\label{algo:tan2014}
\begin{algorithmic}[1]
\REQUIRE A mean-centered observed matrix $\bar{A} = (\bar{A}_{ij})_{1 \leq i \leq n, 1 \leq j \leq p}$, $\bar{A}_{ij} = A_{ij} - \frac{1}{np} \sum_{i = 1}^n \sum_{j = 1}^p A_{ij}$. 
\ENSURE Approximated optimal set of cluster memberships $\hat{g} = (\hat{g}^{(1)}, \hat{g}^{(2)})$. 
\STATE Define that $\hat{I}_k \equiv \{ i: \hat{g}^{(1)}_i = k \}$ and $\hat{J}_h \equiv \{ j: \hat{g}^{(2)}_j = h \}$. 
\STATE Define initial row cluster memberships $\hat{g}^{(1)}$ by applying one-way k-means clustering to the rows of matrix $\bar{A}$. 
\STATE Define initial column cluster memberships $\hat{g}^{(2)}$ by applying one-way k-means clustering to the columns of matrix $\bar{A}$. 
\WHILE{\textbf{true}}
  \STATE $\hat{g}^{(1)}_0 \gets \hat{g}^{(1)}$, $\hat{g}^{(2)}_0 \gets \hat{g}^{(2)}$. 
  \STATE $\hat{B}_{kh} \gets \frac{1}{|\hat{I}_k||\hat{J}_h|} \sum_{i \in \hat{I}_k} \sum_{j \in \hat{J}_h} \bar{A}_{ij}$. 
  \FOR {$i = 1, \dots, n$}
    \STATE $\hat{g}^{(1)}_i \gets \argmin_{k \in \{1, \dots, K\}} \sum_{h = 1}^H \sum_{j \in \hat{J}_h} (\bar{A}_{ij} - \hat{B}_{kh})^2$. 
  \ENDFOR
  \STATE $\hat{B}_{kh} \gets \frac{1}{|\hat{I}_k||\hat{J}_h|} \sum_{i \in \hat{I}_k} \sum_{j \in \hat{J}_h} \bar{A}_{ij}$. 
  \FOR {$j = 1, \dots, p$}
    \STATE $\hat{g}^{(2)}_j \gets \argmin_{h \in \{1, \dots, H\}} \sum_{k = 1}^K \sum_{i \in \hat{I}_k} (\bar{A}_{ij} - \hat{B}_{kh})^2$. 
  \ENDFOR
  \IF{$\hat{g}^{(1)}_0 = \hat{g}^{(1)}$ and $\hat{g}^{(2)}_0 = \hat{g}^{(2)}$}
  \STATE \textbf{break}
  \ENDIF
\ENDWHILE
\end{algorithmic}
\end{algorithm}

We checked the behavior of the approximated test in a realizable case when using Algorithm \ref{algo:tan2014} for estimating the optimal cluster memberships $\hat{g}$. For finding the solution $\tilde{g}$ of the truncation interval, we used Algorithm \ref{algo:interval} as in the experiment in Section \ref{sec: exp_approx}. As in Section \ref{sec: exp_approx}, we generated data matrices and applied the approximated test. Aside from the method for estimating the cluster memberships, we used the same settings as in Section \ref{sec: exp_approx}. This experiment was conducted on an Intel Xeon E5-2680 v3 ($12$ cores @ $2.50$ GHz) server with $1,007$ GB of RAM. 

Figures \ref{fig:pvalues_p_approx_tan14} and \ref{fig:pvalues_n_approx_tan14}, respectively, show the histograms of the $p$-values of the proposed and naive approximated tests. We also plotted (i) the test statistics $D\sqrt{r}$ of the Kolmogorov-Smirnov test \cite{Conover1999}, for the $p$-values of the proposed and naive tests, and (ii) the accuracy of the biclustering algorithm in \cite{Tan2014} in Figures \ref{fig:ks_test_approx_tan14} and \ref{fig:accuracy_approx_tan14}, respectively. Figure \ref{fig:ratios_realizable_approx_tan14} shows the FPR and TPR. Finally, we plotted the computation time for each setting of mean vector $\bm{\mu}_0$ and matrix size $n$ in Figure \ref{fig:time}. From these figures, we see that the biclustering algorithm in \cite{Tan2014} was able to achieve accuracy comparable to or better than the proposed SA-based algorithm in less computation time.

\begin{figure}[p]
  \centering
  \includegraphics[width=0.9\hsize]{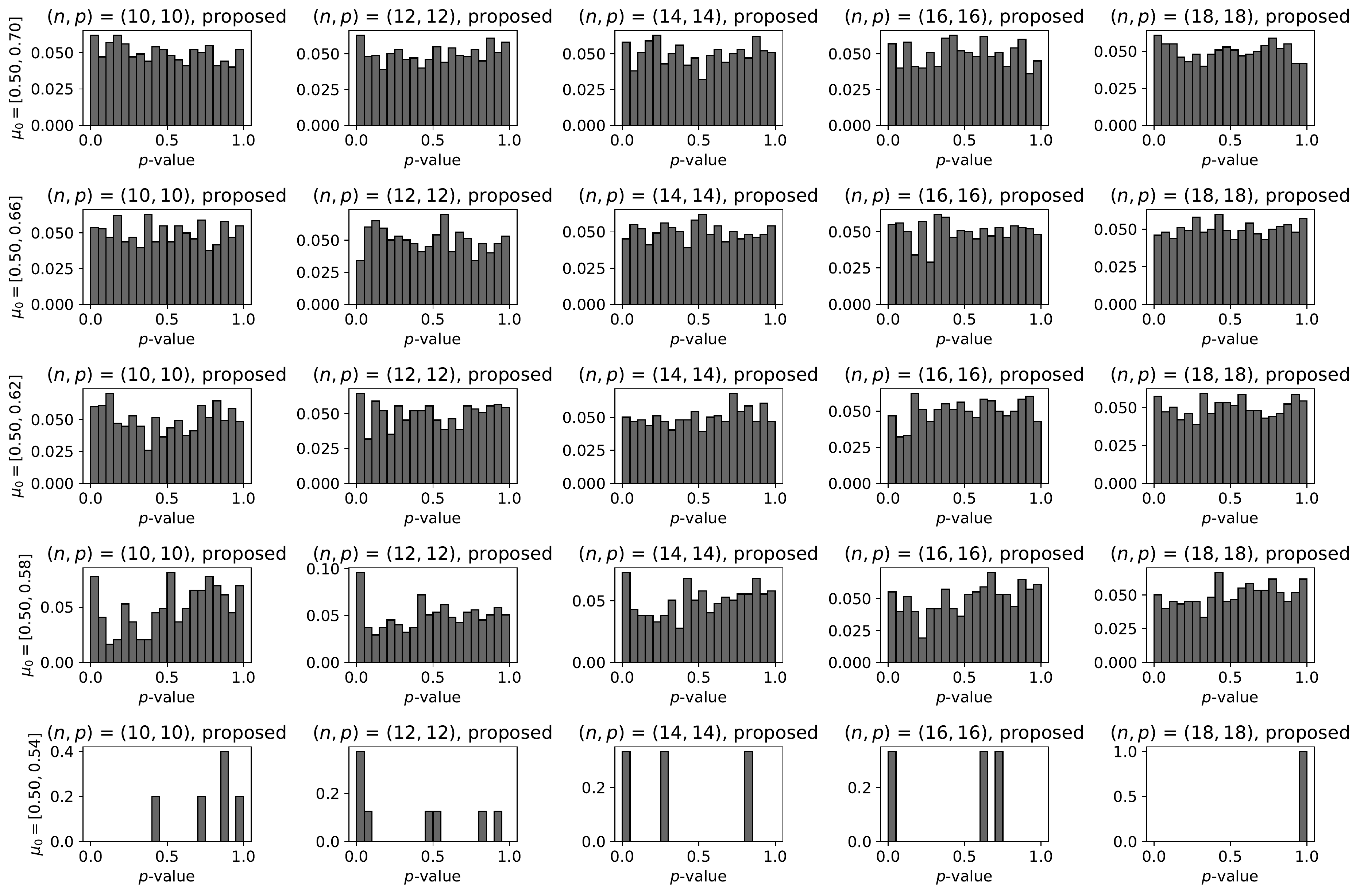}
  \caption{Histograms of $p$-values in the null case (i.e., $\hat{g} = g^{\mathrm{(N)}}$) for different matrix sizes, which was computed by the \textbf{approximated} version of the \textbf{proposed} test based on the biclustering algorithm in \cite{Tan2014}.}\vspace{3mm}
  \label{fig:pvalues_p_approx_tan14}
  \includegraphics[width=0.9\hsize]{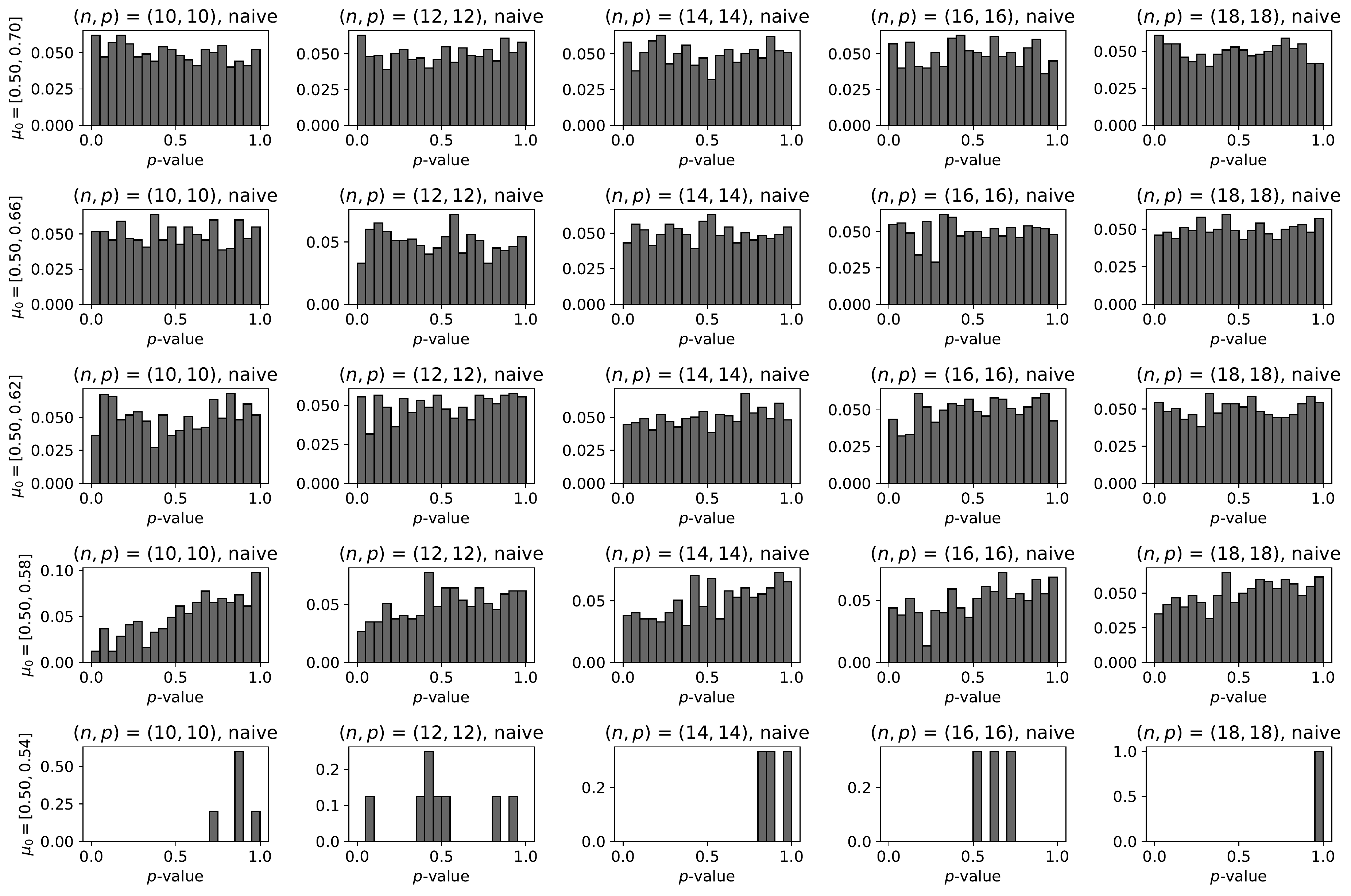}
  \caption{Histograms of $p$-values in the null case (i.e., $\hat{g} = g^{\mathrm{(N)}}$) for different matrix sizes, which was computed by the \textbf{approximated} version of the \textbf{naive} test (\ref{eq:pval_naive}) based on the biclustering algorithm in \cite{Tan2014}.}
  \label{fig:pvalues_n_approx_tan14}
\end{figure}
\begin{figure}[t]
  \centering
  \includegraphics[width=0.95\hsize]{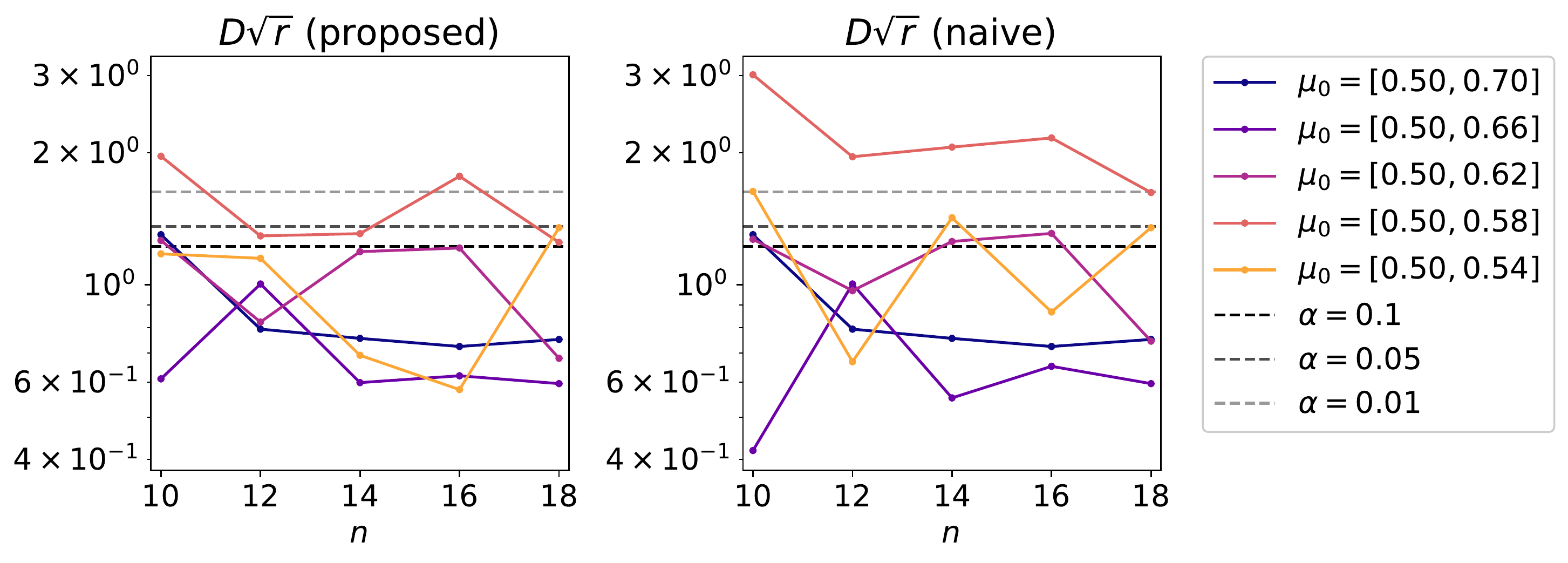}\vspace{-3mm}
  \caption{Test statistics $D\sqrt{r}$ of the Kolmogorov-Smirnov test \cite{Conover1999} for the $p$-values of the proposed (left) and naive (right) \textbf{approximated} tests based on the biclustering algorithm in \cite{Tan2014}.}\vspace{3mm}
  \label{fig:ks_test_approx_tan14}
  \includegraphics[width=0.7\hsize]{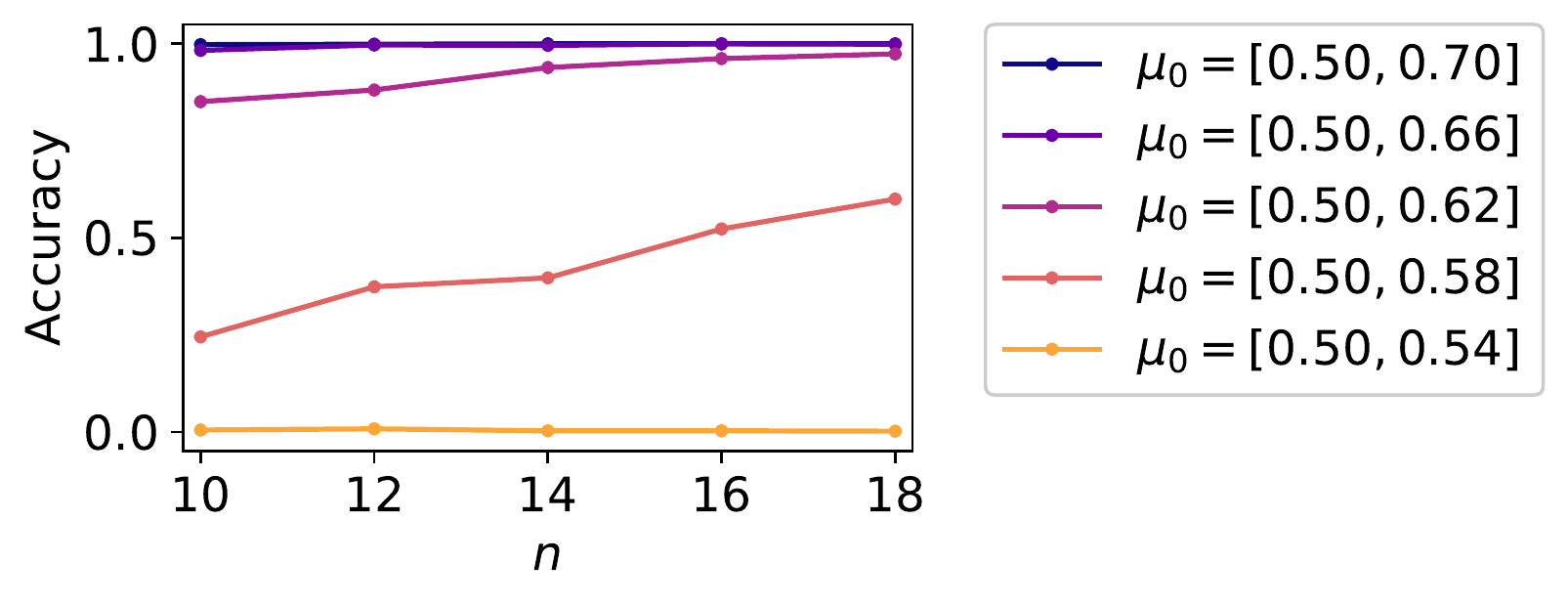}\vspace{-3mm}
  \caption{The ratio of the number of the null cases (i.e., $\hat{g} = g^{\mathrm{(N)}}$) for each setting of matrix size $(n, p)$ and mean vector $\bm{\mu}_0$, where $\hat{g}$ is output by the biclustering algorithm in \cite{Tan2014}. For the experiment, we used the setting of $n = p$.}
  \label{fig:accuracy_approx_tan14}
\end{figure}
\begin{figure}[t]
  \centering
  \includegraphics[width=0.99\hsize]{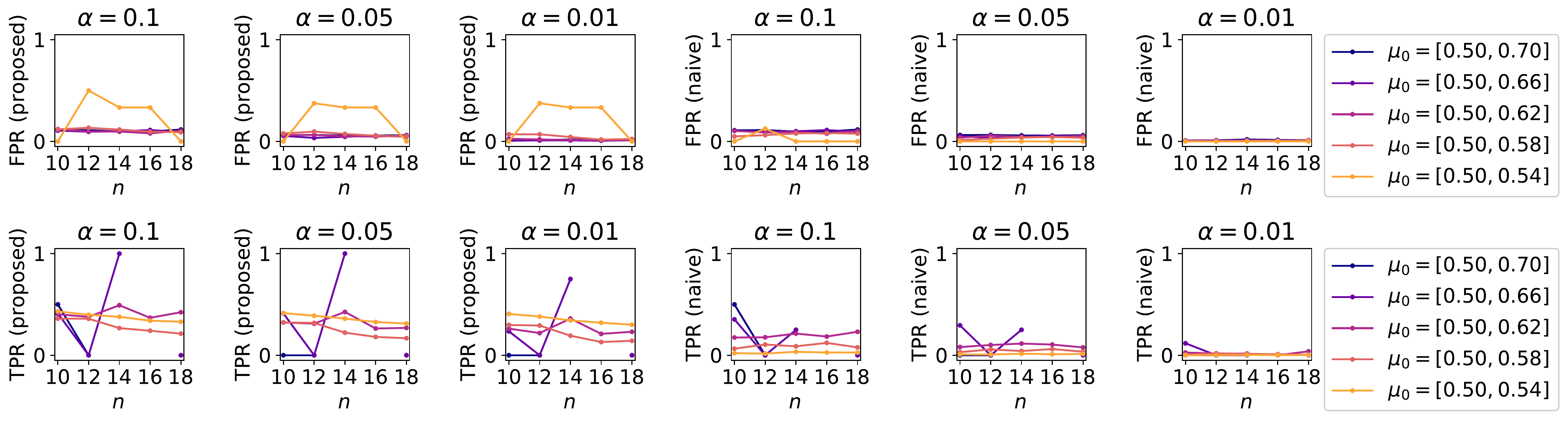}
  \caption{FPR and TPR in the realizable case with different significance rates (e.g., $\alpha = 0.1, 0.05$, and $0.01$), for the \textbf{approximated} version of the proposed (left) and naive (right) statistical tests based on the biclustering algorithm in \cite{Tan2014}.}\vspace{3mm}
  \label{fig:ratios_realizable_approx_tan14}
  \includegraphics[width=0.99\hsize]{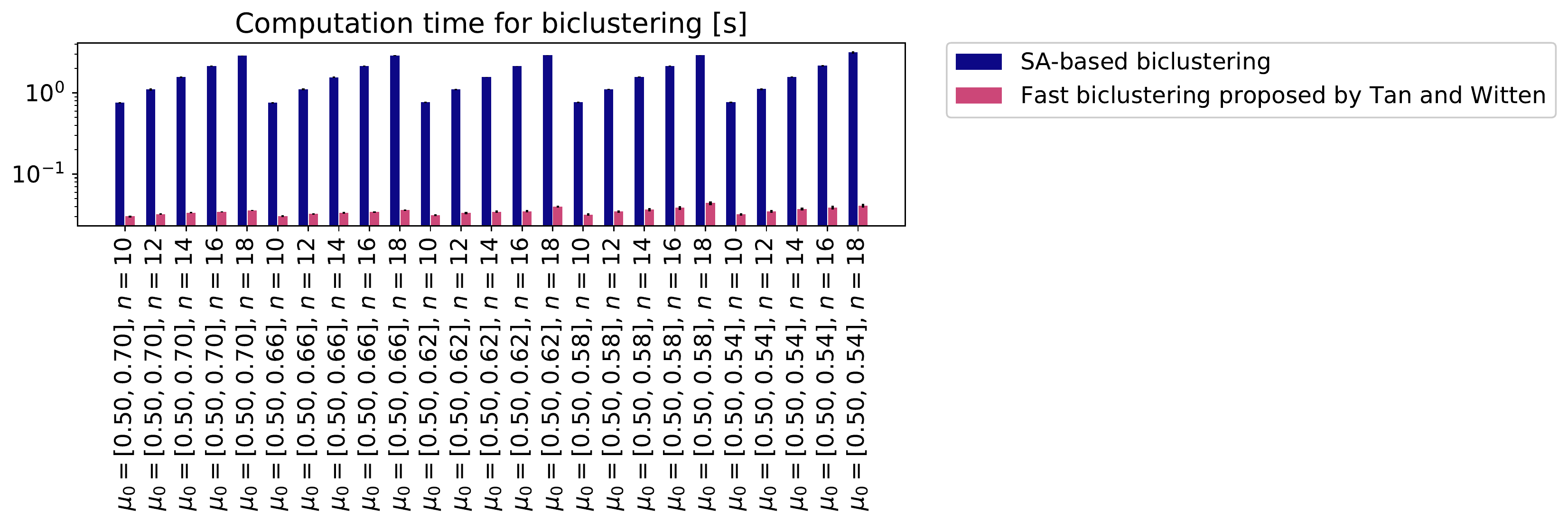}
  \caption{Mean computation time for estimating the cluster memberships $\hat{g}$ based on the proposed SA algorithm and fast biclustering algorithm in \cite{Tan2014}. The error bars indicate the sample standard deviation of the results for $1000$ trials.}\vspace{3mm}
  \label{fig:time}
\end{figure}


\section{Null distribution of test statistic with unknown variance $\sigma_0^2$}
\label{sec:unknown_sigma0}

We derive the null distribution of a new test statistic in case that variance $\sigma_0^2$ is unknown based on a general framework that has been proposed in \cite{Loftus2015}.

\begin{theorem}
\label{thm:TFdist}
Under the null hypothesis, we have
\begin{eqnarray}
\label{eq:defT_sgm_un}
T \equiv \frac{1}{c} \frac{(\| \bm{r} \|_2^2 - \| \bm{r}_1 \|_2^2)}{\| \bm{r}_1 \|_2^2} = \frac{1}{c} \frac{\| \bm{r}_2 \|_2^2}{\| \bm{r}_1 \|_2^2}, \ \ \ 
T | \{ \hat{g}, \bm{u}_1, \bm{u}_2, \bm{z}, \| \bm{r} \|_2 \} \sim F_{d_1, d_2 | \hat{M}^{(\hat{g})}}, 
\end{eqnarray}
where $\| \cdot \|_2$ and $F_{d_1, d_2 | M}$, respectively, denote the Euclid norm and the truncated F distribution with parameters $d_1$ and $d_2$ and with truncation interval of $M$ and
\begin{align}
\label{eq:defruz_sgm_un}
&d_1 \equiv |I_1| |J_1| - 1, \ \ \ d_2 \equiv np - KH - |I_1| |J_1| + 1, \ \ \ c \equiv \frac{d_1}{d_2}, \nonumber \\
&\mathcal{I}^{(k, h)} \equiv \{ n(j-1)+i: i \in I_k, j \in J_h \}, \nonumber \\
&Q^{(\hat{g})} \equiv (Q^{(\hat{g})}_{ij})_{1 \leq i \leq np, 1 \leq j \leq np}, \ \ \ Q^{(\hat{g})}_{ij} = \begin{cases}
E^{(\hat{g})}_{ij} & \mathrm{if}\ [i \in \mathcal{I}^{(1, 1)}] \cap [j \in \mathcal{I}^{(1, 1)}], \\
0 & \mathrm{otherwise}, \\
\end{cases} \nonumber \\
&\bar{Q}^{(\hat{g})} \equiv E^{(\hat{g})} - Q^{(\hat{g})}, \nonumber \\
&\bm{r}_1 \equiv Q^{(\hat{g})} \bm{x}, \ \ \ \bm{r}_2 \equiv \bar{Q}^{(\hat{g})} \bm{x}, \ \ \ \bm{r} \equiv E^{(\hat{g})} \bm{x}, \nonumber \\
&\bm{u}_1 \equiv \frac{1}{\| \bm{r}_1 \|_2} \bm{r}_1, \ \ \ \bm{u}_2 \equiv \frac{1}{\| \bm{r}_2 \|_2} \bm{r}_2, \nonumber \\
&\bm{u} \equiv \frac{1}{\| \bm{r} \|_2} \bm{r} = \frac{1}{\sqrt{cT + 1}} \bm{u}_1 + \sqrt{\frac{cT}{cT + 1}} \bm{u}_2, \nonumber \\
&\bm{z} \equiv \bm{x} - \bm{r}, \nonumber \\
&\hat{M}^{(\hat{g})} \equiv \left\{ t \geq 0: \hat{g} \in \hat{\mathcal{M}} \left[ \| \bm{r} \|_2 \left( \frac{1}{\sqrt{ct + 1}} \bm{u}_1 + \sqrt{\frac{ct}{ct + 1}} \bm{u}_2 \right) + \bm{z} \right] \right\}. 
\end{align}
\end{theorem}
\begin{proof}
Let $Q$ and $\bar{Q}$ be fixed $np \times np$ projection matrices with the ranks of $d_1$ and $d_2$, respectively, satisfying the following conditions: 
\begin{itemize}
\item $Q \bm{\mu}_0 = \bar{Q} \bm{\mu}_0 = \bm{0}$. 
\item There exists a set of row and column indices $I \subseteq \{ 1, \dots, np \}$ such that $Q_{ij} = 0$ if $i \notin I$ or $j \notin I$ holds and $\bar{Q}_{ij} = 0$ if $i \notin \{ 1, \dots, np \} \setminus I$ or $j \notin \{ 1, \dots, np \} \setminus I$ holds. 
\end{itemize}
It must be noted that $Q \bm{x}$ and $\bar{Q} \bm{x}$ are mutually independent from the second condition. 

Based on matrices $Q$ and $\bar{Q}$, we use the following notations: 
\begin{align}
\label{eq:def_T_u_Z_sgm_un}
&E \equiv Q + \bar{Q}, \nonumber \\
&\bm{r}_Q \equiv Q \bm{x}, \ \ \ \bm{r}_{\bar{Q}} \equiv \bar{Q} \bm{x}, \ \ \ \bm{r}_E \equiv E \bm{x}, \nonumber \\
&T_E \equiv \frac{1}{c} \frac{(\| \bm{r}_E \|_2^2 - \| \bm{r}_Q \|_2^2)}{\| \bm{r}_Q \|_2^2} = \frac{1}{c} \frac{\| \bm{r}_{\bar{Q}} \|_2^2}{\| \bm{r}_Q \|_2^2}, \nonumber \\
&\bm{u}_Q \equiv \frac{1}{\| \bm{r}_Q \|_2} \bm{r}_Q, \ \ \ \bm{u}_{\bar{Q}} \equiv \frac{1}{\| \bm{r}_{\bar{Q}} \|_2} \bm{r}_{\bar{Q}}, \nonumber \\
&\bm{u}_E \equiv \frac{1}{\| \bm{r}_E \|_2} \bm{r}_E = \frac{1}{\sqrt{cT_E + 1}} \bm{u}_Q + \sqrt{\frac{cT_E}{cT_E + 1}} \bm{u}_{\bar{Q}}, \nonumber \\
&\bm{z}_E \equiv \bm{x} - \bm{r}_E. 
\end{align}

From the similar discussion as in the proof of Theorem 3.1, we have
\begin{align}
\label{eq:chi_r_Q_barQ}
\frac{\| \bm{r}_Q \|_2}{\sigma_0} \sim \chi_{d_1}, \ \ \ 
\frac{\| \bm{r}_{\bar{Q}} \|_2}{\sigma_0} \sim \chi_{d_2}.
\end{align}
Since $\bm{r}_Q$ and $\bm{r}_{\bar{Q}}$ are mutually independent, so do $\| \bm{r}_Q \|_2$ and $\| \bm{r}_{\bar{Q}} \|_2$. By combining this fact, (\ref{eq:def_T_u_Z_sgm_un}), and (\ref{eq:chi_r_Q_barQ}), we have
\begin{eqnarray}
\label{eq:T_E_F}
T_E \sim F_{d_1, d_2}. 
\end{eqnarray}

Here, we show that $T_E$ and $(\bm{u}_Q, \bm{u}_{\bar{Q}}, \bm{z}_E, \| \bm{r}_E \|_2)$ are mutually independent. From the assumption, we have
\begin{align}
\label{eq:pxind_sgm_un}
&p(\bm{x}) = \frac{1}{\sqrt{(2\pi \sigma_0^2)^{np}}} \exp \left[ -\frac{1}{2 \sigma_0^2} (\bm{x} - \bm{\mu}_0)^{\top} (\bm{x} - \bm{\mu}_0) \right]. \nonumber \\
=& \frac{1}{\sqrt{(2\pi \sigma_0^2)^{np}}} \exp \left[ -\frac{1}{2 \sigma_0^2} (\| \bm{r}_E \|_2 \bm{u}_E + \bm{z}_E - \bm{\mu}_0)^{\top} (\| \bm{r}_E \|_2 \bm{u}_E + \bm{z}_E - \bm{\mu}_0) \right] \nonumber \\
=& \frac{1}{\sqrt{(2\pi \sigma_0^2)^{np}}} \exp \left[ -\frac{1}{2 \sigma_0^2} (\| \bm{r}_E \|_2^2 \| \bm{u}_E \|_2^2 + \| \bm{z}_E - \bm{\mu}_0 \|_2^2) \right]\ \ \ (\because \bm{u}_E^{\top} \bm{z}_E = \bm{u}_E^{\top} \bm{\mu}_0 = 0) \nonumber \\
=& \frac{1}{\sqrt{(2\pi \sigma_0^2)^{np}}} \exp \left[ -\frac{1}{2 \sigma_0^2} (\| \bm{r}_E \|_2^2 + \| \bm{z}_E - \bm{\mu}_0 \|_2^2) \right]\ \ \ (\because \| \bm{u}_E \|_2 = 1) \nonumber \\
=& \exp \left[ -\frac{1}{2 \sigma_0^2} (\| \bm{r}_E \|_2^2 + \| \bm{z}_E \|_2^2 -2 \bm{z}_E^{\top} \bm{\mu}_0) -\frac{\| \bm{\mu}_0 \|_2^2}{2 \sigma_0^2} - \frac{np}{2} \log \left( 2\pi \sigma_0^2 \right) \right]. 
\end{align}
By using the notation of $\bm{\eta} \equiv \begin{bmatrix} -\frac{1}{2 \sigma_0^2} & \frac{1}{\sigma_0^2} \bm{\mu}_0^{\top} \end{bmatrix}^{\top}$, we have
\begin{align}
-\frac{1}{2 \sigma_0^2} (\| \bm{r}_E \|_2^2 + \| \bm{z}_E \|_2^2 -2 \bm{z}_E^{\top} \bm{\mu}_0) = \bm{\eta}^{\top} \begin{bmatrix}
\| \bm{r}_E \|_2^2 + \| \bm{z}_E \|_2^2 & \bm{z}_E^{\top}
\end{bmatrix}^{\top}, 
\end{align}
and thus $(\| \bm{r}_E \|_2^2 + \| \bm{z}_E \|_2^2, \bm{z}_E^{\top})$ are complete and sufficient for $\bm{\eta}$ from Proposition 2.1 in \cite{Shao2003}. Since there is a one-to-one correspondence between $(\bm{z}_E, \| \bm{r}_E \|_2)$ and $(\| \bm{r}_E \|_2^2 + \| \bm{z}_E \|_2^2, \bm{z}_E^{\top})$, $(\bm{z}_E, \| \bm{r}_E \|_2)$ are also complete and sufficient for $\bm{\eta}$. 

Next, we show that $(T_E, \bm{u}_Q, \bm{u}_{\bar{Q}})$ are ancillary for $\bm{\eta}$. To prove this, we first show that $T_E$ and $(\bm{u}_Q, \bm{u}_{\bar{Q}})$ are mutually independent. We use the following notations: 
\begin{align}
\bm{y}_Q \equiv \tilde{D}_Q V_Q \bm{x} \in \mathbb{R}^{d_1}, \ \ \ 
\bm{y}_{\bar{Q}} \equiv \tilde{D}_{\bar{Q}} V_{\bar{Q}} \bm{x} \in \mathbb{R}^{d_2}, \ \ \ 
\bm{y}_E \equiv \tilde{D}_E V_E \bm{x} \in \mathbb{R}^{KH}, 
\end{align}
where $Q = V_Q^{\top} D_Q V_Q$, $\bar{Q} = V_{\bar{Q}}^{\top} D_{\bar{Q}} V_{\bar{Q}}$, and $I-E = V_E^{\top} D_E V_E$ are singular value decompositions of matrices $Q$, $\bar{Q}$, and $I-E$, respectively, and
\begin{align}
&\tilde{D}_Q \equiv \begin{bmatrix} I_{d_1} & O_{(d_1, np - d_1)} \end{bmatrix} \in \mathbb{R}^{d_1 \times np}, \nonumber \\
&\tilde{D}_{\bar{Q}} \equiv \begin{bmatrix} I_{d_2} & O_{(d_2, np - d_2)} \end{bmatrix} \in \mathbb{R}^{d_2 \times np}, \nonumber \\
&\tilde{D}_E \equiv \begin{bmatrix} I_{KH} & O_{(KH, np - KH)} \end{bmatrix} \in \mathbb{R}^{KH \times np}. 
\end{align}
It must be noted that we have
\begin{align}
\bm{y}_Q \sim N(\bm{0}, \sigma_0^2 I_{d_1}), \ \ \ 
\bm{y}_{\bar{Q}} \sim N(\bm{0}, \sigma_0^2 I_{d_2}), \ \ \ 
\bm{y}_E \sim N(\tilde{D}_E V_E \bm{\mu}_0, \sigma_0^2 I_{KH}). 
\end{align}
From (\ref{eq:pxind_sgm_un}), we have
\begin{align}
p(\bm{x}) &= p(\bm{y}_Q, \bm{y}_{\bar{Q}}, \bm{y}_E) \nonumber \\
&= \frac{1}{\sqrt{(2\pi \sigma_0^2)^{np}}} \exp \left[ -\frac{1}{2 \sigma_0^2} (\| \bm{y}_Q \|_2^2 + \| \bm{y}_{\bar{Q}} \|_2^2 + \| \bm{y}_E - \tilde{D}_E V_E \bm{\mu}_0 \|_2^2) \right] \nonumber \\
&= p(\bm{y}_Q) p(\bm{y}_{\bar{Q}}) p(\bm{y}_E), 
\end{align}
which results in that $\bm{y}_Q$, $\bm{y}_{\bar{Q}}$, and $\bm{y}_E$ are independent. Since $\bm{r}_Q = V_Q^{\top} \tilde{D}_Q^{\top} \bm{y}_Q$, $\bm{r}_{\bar{Q}} = V_{\bar{Q}}^{\top} \tilde{D}_{\bar{Q}}^{\top} \bm{y}_{\bar{Q}}$, and $\bm{z}_E = V_E^{\top} \tilde{D}_E^{\top} \bm{y}_E$ hold, $\bm{r}_Q$, $\bm{r}_{\bar{Q}}$, and $\bm{z}_E$ are also independent. Based on a similar discussion as in Appendix \ref{sec:ap_indTuz}, $\| \bm{r}_Q \|_2$ and $\bm{u}_Q$ are mutually independent, and so are $\| \bm{r}_{\bar{Q}} \|_2$ and $\bm{u}_{\bar{Q}}$. Therefore, we have
\begin{align}
&p(\| \bm{r}_Q \|_2, \bm{u}_Q, \| \bm{r}_{\bar{Q}} \|_2, \bm{u}_{\bar{Q}}, \bm{z}_E) = p(\| \bm{r}_Q \|_2, \bm{u}_Q) p(\| \bm{r}_{\bar{Q}} \|_2, \bm{u}_{\bar{Q}}) p(\bm{z}_E) \nonumber \\
=& p(\| \bm{r}_Q \|_2) p(\bm{u}_Q) p(\| \bm{r}_{\bar{Q}} \|_2) p(\bm{u}_{\bar{Q}}) p(\bm{z}_E) \nonumber \\
=& p(\| \bm{r}_Q \|_2, \| \bm{r}_{\bar{Q}} \|_2) p(\bm{u}_Q, \bm{u}_{\bar{Q}}) p(\bm{z}_E), 
\end{align}
which results in that $(\| \bm{r}_Q \|_2, \| \bm{r}_{\bar{Q}} \|_2)$ and $(\bm{u}_Q, \bm{u}_{\bar{Q}})$ are mutually independent. Based on this fact, $T_E = \frac{1}{c} \frac{\| \bm{r}_{\bar{Q}} \|_2^2}{\| \bm{r}_Q \|_2^2}$ and $(\bm{u}_Q, \bm{u}_{\bar{Q}})$ are mutually independent. By using this result and the fact that $\bm{u}_Q$ and $\bm{u}_{\bar{Q}}$ are also mutually independent, we have $p(T_E, \bm{u}_Q, \bm{u}_{\bar{Q}}) = p(T_E) p(\bm{u}_Q) p(\bm{u}_{\bar{Q}})$. Based on a similar discussion as in Appendix \ref{sec:ap_indTuz} about $p(\bm{u}_Q)$ and $p(\bm{u}_{\bar{Q}})$ and the fact that $T_E \sim F_{d_1, d_2}$ from (\ref{eq:T_E_F}), $(T_E, \bm{u}_Q, \bm{u}_{\bar{Q}})$ are ancillary for $\bm{\eta}$. Therefore, $(T_E, \bm{u}_Q, \bm{u}_{\bar{Q}})$ and $(\bm{z}_E, \| \bm{r}_E \|_2)$ are mutually independent from Basu's theorem \cite{Basu1955}.

By combining the above results, we have
\begin{align}
\label{eq:T_E_ind_uzr}
&p(T_E | \bm{u}_Q, \bm{u}_{\bar{Q}}, \bm{z}_E, \| \bm{r}_E \|_2) = \frac{p(T_E, \bm{u}_Q, \bm{u}_{\bar{Q}}, \bm{z}_E, \| \bm{r}_E \|_2)}{p(\bm{u}_Q, \bm{u}_{\bar{Q}}, \bm{z}_E, \| \bm{r}_E \|_2)} \nonumber \\
=& \frac{p(T_E, \bm{u}_Q, \bm{u}_{\bar{Q}}) p(\bm{z}_E, \| \bm{r}_E \|_2)}{p(\bm{u}_Q, \bm{u}_{\bar{Q}}, \bm{z}_E, \| \bm{r}_E \|_2)} 
= \frac{p(T_E, \bm{u}_Q, \bm{u}_{\bar{Q}}) p(\bm{z}_E, \| \bm{r}_E \|_2)}{p(\bm{u}_Q, \bm{u}_{\bar{Q}}) p(\bm{z}_E, \| \bm{r}_E \|_2)} \nonumber \\
&= \frac{p(T_E, \bm{u}_Q, \bm{u}_{\bar{Q}})}{p(\bm{u}_Q, \bm{u}_{\bar{Q}})} 
= p(T_E). 
\end{align}
To derive the third equation, we used the fact that $(\bm{u}_Q, \bm{u}_{\bar{Q}})$ and $(\bm{z}_E, \| \bm{r}_E \|_2)$ are mutually independent based on a similar discussion as above. From (\ref{eq:T_E_ind_uzr}), $T_E$ and $(\bm{u}_Q, \bm{u}_{\bar{Q}}, \bm{z}_E, \| \bm{r}_E \|_2)$ are mutually independent.

By combining the above fact and (\ref{eq:T_E_F}), we have
\begin{eqnarray}
\label{eq:tildeT_uz_sgm_un}
T_E | \bm{u}_Q, \bm{u}_{\bar{Q}}, \bm{z}_E, \| \bm{r}_E \|_2 \sim F_{d_1, d_2}. 
\end{eqnarray}

Next, we consider adding a condition of selection event of $\hat{g}$ to the distribution of $T_E | \bm{u}_Q, \bm{u}_{\bar{Q}}, \bm{z}_E, \| \bm{r}_E \|_2$ in (\ref{eq:tildeT_uz_sgm_un}). Given $(\bm{u}_Q, \bm{u}_{\bar{Q}}, \bm{z}_E, \| \bm{r}_E \|_2)$, the result of selection depends solely on the value of $T_E$, since $\bm{x} = \| \bm{r}_E \|_2 \left( \frac{1}{\sqrt{cT_E + 1}} \bm{u}_Q + \sqrt{\frac{cT_E}{cT_E + 1}} \bm{u}_{\bar{Q}} \right) + \bm{z}_E$ holds. Therefore, adding the selection condition to (\ref{eq:tildeT_uz_sgm_un}) corresponds to truncation of $T_E$ to the region where $\hat{\mathcal{M}} \left[ \| \bm{r}_E \|_2 \left( \frac{1}{\sqrt{cT_E + 1}} \bm{u}_Q + \sqrt{\frac{cT_E}{cT_E + 1}} \bm{u}_{\bar{Q}} \right) + \bm{z}_E \right] = \hat{g}$ holds: 
\begin{eqnarray}
\label{eq:tildeT_uzg_sgm_un}
T_E | \bm{u}_Q, \bm{u}_{\bar{Q}}, \bm{z}_E, \| \bm{r}_E \|_2, \hat{g} \sim F_{d_1, d_2 | \hat{M}^{(\hat{g})} (E)}. 
\end{eqnarray}

Third, we consider replacing $Q$ and $\bar{Q}$ in (\ref{eq:tildeT_uzg_sgm_un}) with $Q^{(\hat{g})}$ and $\bar{Q}^{(\hat{g})}$, which is the output by clustering algorithm $\mathcal{A}$ based on the data vector $\bm{x}$. Based on a similar discussion as in Appendix \ref{sec:ap_rankE}, the matrices $Q^{(\hat{g})}$ and $\bar{Q}^{(\hat{g})}$ are also projection matrices with the ranks of $d_1$ and $d_2$, respectively, and they satisfy the following conditions: 
\begin{itemize}
\item $Q^{(\hat{g})} \bm{\mu}_0 = \bar{Q}^{(\hat{g})} \bm{\mu}_0 = \bm{0}$. 
\item There exists a set of row and column indices $I \subseteq \{ 1, \dots, np \}$ such that $Q^{(\hat{g})}_{ij} = 0$ if $i \notin I$ or $j \notin I$ holds and $\bar{Q}^{(\hat{g})}_{ij} = 0$ if $i \notin \{ 1, \dots, np \} \setminus I$ or $j \notin \{ 1, \dots, np \} \setminus I$ holds. 
\end{itemize}

Since matrices $Q^{(\hat{g})}$ and $\bar{Q}^{(\hat{g})}$ depend on the data vector $\bm{x}$ only through the choice of $\hat{g}$, under the condition that the selection result $\hat{g}$ is given, (\ref{eq:tildeT_uzg_sgm_un}) still holds with matrices $Q^{(\hat{g})}$ and $\bar{Q}^{(\hat{g})}$, which concludes the proof. 
\end{proof}

\end{appendices}


\clearpage
\bibliographystyle{abbrv}
\bibliography{paper}

\begin{thebibliography}{10}

\bibitem{Basu1955}
D.~Basu.
\newblock On statistics independent of a complete sufficient statistic.
\newblock {\em Sankhy\={a}: The Indian Journal of Statistics}, 15(4):377--380,
  1955.

\bibitem{Berk2013}
R.~Berk, L.~Brown, A.~Buja, K.~Zhang, and L.~Zhao.
\newblock Valid post-selection inference.
\newblock {\em The Annals of Statistics}, 41(2):802--837, 2013.

\bibitem{Bickel2016}
P.~J. Bickel and P.~Sarkar.
\newblock Hypothesis testing for automated community detection in networks.
\newblock {\em Journal of the Royal Statistical Society: Series B (Statistical
  Methodology)}, 78(1):253--273, 2016.

\bibitem{Chi2017}
E.~C. Chi, G.~I. Allen, and R.~G. Baraniuk.
\newblock Convex biclustering.
\newblock {\em Biometrics}, 73(1):10--19, 2017.

\bibitem{Cho2004}
H.~Cho, I.~S. Dhillon, Y.~Guan, and S.~Sra.
\newblock Minimum sum-squared residue co-clustering of gene expression data.
\newblock In {\em Proceedings of the 2004 SIAM International Conference on Data
  Mining}, pages 114--125, 2004.

\bibitem{Conover1999}
W.~J. Conover.
\newblock {\em Practical Nonparametric Statistics}.
\newblock John Wiley \& Sons, New York, 1999.

\bibitem{Fithian2014}
W.~Fithian, D.~Sun, and J.~Taylor.
\newblock Optimal inference after model selection.
\newblock arXiv:1410.2597, 2014.

\bibitem{Gangrade2019}
A.~Gangrade, P.~Venkatesh, B.~Nazer, and V.~Saligrama.
\newblock Efficient near-optimal testing of community changes in balanced
  stochastic block models.
\newblock In {\em Advances in Neural Information Processing Systems 32}, pages
  10364--10375, 2019.

\bibitem{Govaert2003}
G.~Govaert and M.~Nadif.
\newblock Clustering with block mixture models.
\newblock {\em Pattern Recognition}, 36:463--473, 2003.

\bibitem{Hajek1988}
B.~Hajek.
\newblock Cooling schedules for optimal annealing.
\newblock {\em Mathematics of Operations Research}, 13(2):311--329, 1988.

\bibitem{Harper2015}
F.~M. Harper and J.~A. Konstan.
\newblock The movielens datasets: History and context.
\newblock {\em ACM Transactions on Interactive Intelligent Systems},
  5(4):1--19, 2015.

\bibitem{Hartigan1972}
J.~A. Hartigan.
\newblock Direct clustering of a data matrix.
\newblock {\em Journal of the American Statistical Association},
  67(337):123--129, 1972.

\bibitem{Henriques2018}
R.~Henriques and S.~C. Madeira.
\newblock Bsig: evaluating the statistical significance of biclustering
  solutions.
\newblock {\em Data Mining and Knowledge Discovery}, 32:124--161, 2018.

\bibitem{Hu2020}
J.~Hu, J.~Zhang, H.~Qin, T.~Yan, and J.~Zhu.
\newblock Using maximum entry-wise deviation to test the goodness-of-fit for
  stochastic block models.
\newblock {\em Journal of the American Statistical Association}, 0:1--30, 2020.

\bibitem{Inoue2017}
S.~Inoue, Y.~Umezu, S.~Tsubota, and I.~Takeuchi.
\newblock Post clustering inference for heterogeneous data.
\newblock In {\em Information-Based Induction Science Workshop}, pages 69--76,
  2017.

\bibitem{Karwa2016}
V.~Karwa, D.~Pati, S.~Petrovi\'c, L.~Solus, N.~Alexeev, M.~Rai\v{c},
  D.~Wilburne, R.~Williams, and B.~Yan.
\newblock Exact tests for stochastic block models.
\newblock arXiv:1612.06040, 2016.

\bibitem{Kirkpatrick1983}
S.~Kirkpatrick, C.~D. Gelatt, and M.~P. Vecchi.
\newblock Optimization by simulated annealing.
\newblock {\em Science}, 220:671--680, 1983.

\bibitem{Lee2016}
J.~D. Lee, D.~L. Sun, Y.~Sun, and J.~E. Taylor.
\newblock Exact post-selection inference, with application to the lasso.
\newblock {\em The Annals of Statistics}, 44(3):907--927, 2016.

\bibitem{Lee2015}
J.~D. Lee, Y.~Sun, and J.~E. Taylor.
\newblock Evaluating the statistical significance of biclusters.
\newblock In {\em Advances in Neural Information Processing Systems 28}, pages
  1324--1332, 2015.

\bibitem{Lee2014}
J.~D. Lee and J.~E. Taylor.
\newblock Exact post model selection inference for marginal screening.
\newblock In {\em Advances in Neural Information Processing Systems 27}, pages
  136--144, 2014.

\bibitem{Lee2010}
M.~Lee, H.~Shen, J.~Z. Huang, and J.~S. Marron.
\newblock Biclustering via sparse singular value decomposition.
\newblock {\em Biometrics}, 66(4):1087--1095, 2010.

\bibitem{Lei2016}
J.~Lei.
\newblock A goodness-of-fit test for stochastic block models.
\newblock {\em The Annals of Statistics}, 44(1):401--424, 2016.

\bibitem{Loftus2015}
J.~R. Loftus and J.~E. Taylor.
\newblock Selective inference in regression models with groups of variables.
\newblock arXiv:1511.01478, 2015.

\bibitem{Lomet2012}
A.~Lomet, G.~Govaert, and Y.~Grandvalet.
\newblock Model selection in block clustering by the integrated classification
  likelihood.
\newblock In {\em Proceedings of the 20th International Conference on
  Computational Statistics}, pages 519--530, 2012.

\bibitem{Nadif2010}
M.~Nadif and G.~Govaert.
\newblock Model-based co-clustering for continuous data.
\newblock In {\em Proceedings of the 9th International Conference on Machine
  Learning and Applications}, pages 175--180, 2010.

\bibitem{Perrone2017}
V.~Perrone, P.~A. Jenkins, D.~Span\`{o}, and Y.~W. Teh.
\newblock Poisson random fields for dynamic feature models.
\newblock {\em Journal of Machine Learning Research}, 18(127):1--45, 2017.

\bibitem{Saber2011}
H.~B. Saber, M.~Elloumi, and M.~Nadif.
\newblock Block mixture model for the biclustering of microarray data.
\newblock In {\em Proceedings of the 22nd International Workshop on Database
  and Expert Systems Applications}, pages 423--427, 2011.

\bibitem{Shan2008}
H.~Shan and A.~Banerjee.
\newblock Bayesian co-clustering.
\newblock In {\em Proceedings of the 8th IEEE International Conference on Data
  Mining}, pages 530--539, 2008.

\bibitem{Shao2003}
J.~Shao.
\newblock {\em Mathematical Statistics}.
\newblock Springer-Verlag New York, 2003.

\bibitem{Tan2014}
K.~M. Tan and D.~M. Witten.
\newblock Sparse biclustering of transposable data.
\newblock {\em Journal of computational and graphical statistics},
  23:985--1008, 2014.

\bibitem{Terada2017}
Y.~Terada and H.~Shimodaira.
\newblock Selective inference for the problem of regions via multiscale
  bootstrap.
\newblock arXiv:1711.00949, 2017.

\bibitem{Tian2018}
X.~Tian and J.~Taylor.
\newblock Selective inference with a randomized response.
\newblock {\em The Annals of Statistics}, 46(2):679--710, 2018.

\bibitem{Tibshirani2018}
R.~J. Tibshirani, A.~Rinaldo, R.~Tibshirani, and L.~Wasserman.
\newblock Uniform asymptotic inference and the bootstrap after model selection.
\newblock {\em The Annals of Statistics}, 46(3):1255--1287, 2018.

\bibitem{Cerny1985}
V.~\v{C}ern\'{y}.
\newblock Thermodynamical approach to the traveling salesman problem: An
  efficient simulation algorithm.
\newblock {\em Journal of Optimization Theory and Applications}, 45:41--51,
  1985.

\bibitem{Watanabe2021}
C.~Watanabe and T.~Suzuki.
\newblock Goodness-of-fit test for latent block models.
\newblock {\em Computational Statistics \& Data Analysis}, 154:107090, 2021.

\bibitem{Wyse2012}
J.~Wyse and N.~Friel.
\newblock Block clustering with collapsed latent block models.
\newblock {\em Statistics and Computing}, 22:415--428, 2012.

\bibitem{Yuan2018}
M.~Yuan, Y.~Feng, and Z.~Shang.
\newblock A likelihood-ratio type test for stochastic block models with bounded
  degrees.
\newblock arXiv:1807.04426, 2018.

\end{thebibliography}

\end{document}